\documentclass[twoside,11pt]{article}
\usepackage{jair, theapa, rawfonts}

\ShortHeadings{Pattern-Based Approach to the Workflow Satisfiability Problem}
{Karapetyan, Parkes, Gutin, \& Gagarin}
\firstpageno{1}

\usepackage{times}

\usepackage{amsmath,amssymb,latexsym,amsthm}
\usepackage{paralist,array,url}
\usepackage[ruled,lined,linesnumbered,boxed,vlined]{algorithm2e}
\usepackage{afterpage}
\usepackage{pgfplots}
\usepackage{booktabs}
\usepackage{subcaption}
\usepackage{multicol}
\usepackage{mathtools}
\usepackage{mathrsfs}

\usetikzlibrary{patterns}
\usepgfplotslibrary{fillbetween}
\usepgfplotslibrary{groupplots}

\newcommand{\set}[1]{\left\{#1\right\}}

\long\def\ignore#1{}

\newcommand{\skipfigure}[1]{\textcolor{green}{The figure is skipped to speed up compilation.}}
\renewcommand{\skipfigure}[1]{#1}

\SetKwInOut{Input}{input}
\SetKwInOut{Output}{output}
           
\usepackage{color}

\ifdefined\excludeappendix
    \newcommand{\appendixA}{Appendix~A in~\cite{KarapetyanPGG16}}
    \newcommand{\appendixB}{Appendix~B in~\cite{KarapetyanPGG16}}
    \newcommand{\appendixC}{Appendix~C in~\cite{KarapetyanPGG16}}
    \newcommand{\appendixD}{Appendix~D in~\cite{KarapetyanPGG16}}
    \newcommand{\appendicesBandC}{Appendices~B and~C in~\cite{KarapetyanPGG16}}
\else
    \newcommand{\appendixA}{Appendix~\ref{ap:encoding-constraints}}
    \newcommand{\appendixB}{Appendix~\ref{ap:estimated-PT}}
    \newcommand{\appendixC}{Appendix~\ref{ap:PT-forced}}
    \newcommand{\appendixD}{Appendix~\ref{ap:100k}}
    \newcommand{\appendicesBandC}{Appendices~\ref{ap:estimated-PT} and~\ref{ap:PT-forced}}
\fi

\newtheorem{definition}{Definition}[section]
\newtheorem{theorem}{Theorem}[section]
\newtheorem{proposition}[theorem]{Proposition}

\newcommand{\pattern}{\ensuremath{\mathcal{P}}}
\allowdisplaybreaks

\pgfplotsset{compat=1.10} % AJP: because overleaf said so

\begin{document}

\title{Pattern-Based Approach\\
to the Workflow Satisfiability Problem\\
with User-Independent Constraints\thanks{A preliminary version of some portions of this paper was published in the proceedings of FAW 2015~\cite{KaGaGu}.}}

\author{\name Daniel Karapetyan \email daniel.karapetyan@gmail.com \\
		\addr Institute for Analytics and Data Science, University of Essex, Colchester CO4 3SQ, UK\\
		\addr School of Computer Science, University of Nottingham, Nottingham NG8 1BB, UK\\
		\AND
		\name Andrew J. Parkes \email andrew.parkes@nottingham.ac.uk \\
		\addr School of Computer Science, University of Nottingham, Nottingham NG8 1BB, UK\\
		\AND
		\name Gregory Gutin \email g.gutin@rhul.ac.uk \\
		\addr Department of Computer Science, Royal Holloway, University of London, Egham TW20 0EX, UK\\
		\AND
		\name Andrei Gagarin \email gagarina@cardiff.ac.uk \\
		\addr School of Mathematics, Cardiff University, Cardiff CF10 3AT, UK\\
		\addr Royal Holloway, University of London, Egham TW20 0EX, UK}

\maketitle

\begin{abstract}
     The fixed parameter tractable (FPT) approach is a powerful tool in tackling computationally hard problems.
     In this paper, we link FPT results to classic artificial intelligence (AI) search techniques to show how they complement each other.
     Specifically, we consider the workflow satisfiability problem (WSP) which asks whether there exists an assignment of authorised users to the steps in a workflow specification, subject to certain constraints on the assignment.
     It was shown by Cohen et al.\ (JAIR 2014) that WSP restricted to the class of user-independent (UI) constraints, 
     covering many practical cases, admits FPT algorithms, i.e.\ can be solved in time exponential only in the number of steps $k$ and polynomial in the number of users $n$.
     Since usually $k \ll n$ in WSP, such FPT algorithms are of great practical interest.
     
     We present a new interpretation of the FPT nature of the WSP with UI constraints giving a decomposition of the problem into two levels. 
     Exploiting this two-level split, we develop a new FPT algorithm that is by many orders of magnitude faster than the previous state-of-the-art WSP algorithm and also has only polynomial-space complexity.
     We also introduce new pseudo-Boolean (PB) and Constraint Satisfaction (CSP) formulations of the WSP with UI constraints which efficiently exploit this new decomposition of the problem and raise the novel issue of how to use general-purpose solvers to tackle FPT problems in a fashion that meets FPT efficiency expectations.
     In our computational study, the phase transition (PT) properties of the WSP are investigated for the first time,   
      under a model for generation of random instances.
     We show how PT studies can be extended, in a novel fashion, to support empirical evaluation of scaling of FPT algorithms.

\bigskip

\end{abstract}

%%%%%%%%%%%%%%%%%%%%%%%%%%%%%%%%%%%%%%%%%%%%%%%%%%%%%%%%%%%%%%%%%%%%%%%%%%%%%%%%%%%%%%%%%%%%%
\section{Introduction} 
\label{sec:intro}

%============= WSP ==============

%  Combinatorial explosion in algorithms' running times has been an ongoing computational challenge within Artificial Intelligence (AI)\@. 
 The combinatorial explosion, in the running times of algorithms for many important problems, has been an ongoing computational challenge within Artificial Intelligence (AI)\@. 
 AI has developed, and continues to develop, many powerful techniques to address this challenge.
 One of them, originating from theoretical computer science but recently also used in AI (see, e.g.,~\shortcite{DeHaan:2015:IJCAI:FPT-planning,OrdyniakSzeider2013:JAIR:FPT-Bayesian}), is the theory of fixed parameter tractable (FPT) algorithms, which is concerned with the parametrisation of hard problems that reveals that these problems are tractable under certain conditions.
 In this paper, we link together AI and FPT results and apply them to an important access control problem arising in many organisations.

 Specifically, many organisations often have to solve `workflow' problems in which multiple sets of tasks, or \emph{steps}, need to be assigned to workers, or \emph{users}, subject to constraints that are designed to ensure effective, safe and secure processing of the tasks. 
 For example, security might require that some sets of tasks are performed by a small group of workers or maybe just one worker. 
 Alternatively, some sets might need to be performed by at least two users, for example, so as to ensure independent processing or cross-checking of work, etc.~\shortcite{BaBuKa14,CrGuYe13,Roy2015,WaLi10}.
 Furthermore, different users have different capabilities and security permissions, and will generally not be authorised to process all of the steps.  
 In the \emph{Workflow Satisfiability Problem} (WSP), the aim is to assign authorised users to the steps in a workflow specification, subject to constraints arising from business rules and practices. %
%\AG{remove "(Note the term ``workflow'' originally arose from the flow of the steps between users, however, ", then continue "In this security context,..."} \DK{Is this just to shorten the sentence?  I don't see any problems with this sentence.} 
 (Note the term ``workflow'' originally arose from the flow of the steps between users, however, in this context, the time ordering is not relevant -- the challenge is to make a feasible assignment for all the steps.)
 The WSP has important applications and has been extensively studied in the security research community~\shortcite{BaBuKa14,BeFeAt99,Bertolissi2018,Cr05,ValuedWSP-SACMAT,CramptonGKW17,CramptonGW17,SantosRCP17,WaLi10}.

%============= Complexity, UI and FPT ==============

 The WSP is NP-complete, and it has been difficult to solve, even for some moderately-sized instances~\shortcite{JOCO2014,Roy2015}.  
 Work in WSP has attempted to render solving of the WSP practical by finding a subclass of problems that admit \emph{fixed parameter tractable} (FPT) algorithms; informally speaking, this means that there is a small parameter $k$ such that the problem is exponential in $k$ but polynomial in the size of the problem.
 In the case of the WSP, the parameter $k$ is naturally the number of steps -- in real-life instances this number is usually much smaller than the number $n$ of users~\shortcite{WaLi10}. 
  
 It has been shown~\shortcite{CoCrGaGuJo14} that the WSP is FPT if it includes only authorisations and \emph{user-independent} (UI) constraints, i.e.\ constraints whose satisfaction does not depend on specific user identities.
 Then existing methods~\shortcite{JOCO2014} achieve a runtime that is polynomial in $n$, and exponential only in $k$. 
 (Constraints described in the WSP literature are relatively simple, so it was assumed by~\citeA{CoCrGaGuJo14,JOCO2014} that it is possible to test whether a solution satisfies any single constraint in polynomial time. 
 We will use the same assumption in this paper.
 We will also assume that an authorisation test -- whether a user is authorised to a step -- takes constant time, as we store authorisation lists in the form of bitmaps.)

%============= Main directions of our study ==============

 The major contributions of this paper are:
\begin{enumerate}
	\item 
    A new understanding of the FPT nature of the WSP with UI constraints that decomposes the problem into two levels: `upper' level corresponding to UI constraints, and `lower' level corresponding to user assignment and authorisations.

    \item
    An effective backtracking search method, `Pattern Backtracking' (PBT), exploiting the two level decomposition and supported by heuristics and pruning.
    (The PBT implementation presented in this paper is a significant improvement over the algorithm introduced in the preliminary report~\shortcite{KaGaGu} and even more so over the FPT algorithm of~\shortcite{JOCO2014} which unlike PBT has an exponential space complexity.)
	
    \item
    A declarative method that we use explicitly in a pseudo-Boolean~\shortcite{BorosHammer2002:PBO,BePa10} and implicitly in a CSP formulations, which also exploits the two level decomposition of the problem. 
	
    \item 
    Experimental studies of the algorithm performances, focussing on average case complexity, and supported by a new methodology that carefully exploits work in AI on phase transition (PT) phenomena~\shortcite{Bollobas:book,Huberman87:phase,cheeseman91where,mitchell92hard} but in a fashion extended and adapted for the needs of an FPT study.

%    supported by phase transition (PT) phenomena \cite{Bollobas:book,Huberman87:phase,cheeseman91where,mitchell92hard} in a fashion extended and adapted for the needs of an FPT study.  
\end{enumerate}

%============= Patterns ==============

To fully exploit the two level decomposition of the problem, we use special structures called \emph{patterns} as introduced by~\citeA{CoCrGaGuJo14}.
Patterns capture the decisions concerned with UI constraints but generally do not fix user assignments.
In particular, they specify which steps are to be performed by the same user and which steps are to be performed by different users.

%=============== PBT algorithm ================

 The notion of patterns is a convenient tool for handling the decomposition of WSP into two levels: an upper level corresponding to UI constraints, where the decisions can be encoded with a pattern, and a lower level corresponding to user assignments (note that authorisations are intrinsically user dependent and hence cannot be handled with patterns).
 This is used in our two-level algorithm which we call Pattern Backtracking (PBT)\@.
 Its upper level implements a tree search in the space of patterns (thus not fixing user assignments), and the lower level searches for a user assignment restricted by a pattern.
 The space of upper level solutions has size exponential in $k$ (and not depending on $n$), and the lower level can be reduced to a bipartite matching problem, i.e.\ admits a polynomial-time algorithm. 
 Thus, PBT has running time exponential in $k$ only.
  
 We will show that PBT is not only FPT but also has polynomial space usage. 
 Note that, the complexity class FPT does not in itself directly restrict the space usage; the previous algorithms had a space usage that was exponential in $k$, and this restricted their application to (roughly) $k \le 20$.
 Moreover, due to the two-level structure of PBT, together with careful design of pruning methods and branching heuristics, the resulting implementation is many orders of magnitude faster than previous methods, with much improved scaling behaviour. 
 The reachable values of $k$, i.e.\ the number of steps in the workflow, jump from about $k \le 20$ to about $k \le 50$; the reachable number of users is also now extended to being in millions.
 This is a significant improvement for an NP-complete problem, and also can be expected to be sufficient for practical-sized WSP instances.
 
%================= Colouring / PBPB view ===================

 We also provide a new pseudo-Boolean (PB) formulation of the problem: `Pattern Based PB' (PBPB)\@.
 Existing PB or integer programming encodings of the WSP with UI constraints~\cite{JOCO2014,WaLi10} are based on the binary decision variables $x_{su}$, indicating whether step $s$ is assigned to user $u$ (we refer to these encodings as `UDPB' for ``user-dependent PB'' encodings)\@.
 To exploit the user-independent nature of our problem, we also use a set of $M$-variables which are not directly linked to specific users:  
 \begin{equation}
	M_{ij} = 1 \mbox{ iff steps $i$ and $j$ are assigned to the same user, $0$ otherwise.}
 \end{equation}
 %Setting $M_{ij} = 1$ or $M_{ij} = 0$ roughly corresponds to how the patterns are used in PBT\@. 
 The $M$-variables in the PBPB formulation are used to allow an `off-the-shelf' PB solver to better exploit the two-level decomposition of the problem as branching on $M$-variables captures the UI property of the constraints and corresponds to the upper level of the search (in the space of patterns).
 Notice that such variables have also been used extensively in powerful semi-definite programming approaches to graph colouring \shortcite{Lovasz1979:shannon}, and have been referred to as the `colouring matrix', e.g. see~\shortcite{DukanovicRendl2008:SDP-heuristic-GCP}.
 By using a generic PB solver, SAT4J~\shortcite{BePa10}, we show that the performance of the new PBPB formulation is several orders of magnitude better than the previous `UDPB' formulation. 
 Moreover, PBPB can even compete with PBT in terms of scaling behaviour, in some circumstances, though not in terms of constant factors.

We also give a Constraint Satisfaction Problem (CSP) formulation and solve it with the state-of-the-art OR-Tools\footnote{ \url{https://developers.google.com/optimization/} (accessed Jan.~2019)} (in the latest MiniZinc constraint solvers competition, OR-Tools was the clear winner~\shortcite{MiniZinc}).
Even though the formulation and solver do not explicitly use our knowledge of the FPT nature of the problem, the performance of the CSP-based solver is similar to that of PBPB\@.  
 We discuss a possible reason for this similarity.

%========== PT =============

 Regarding the computational complexity, FPT problems and algorithms have been mainly studied from the perspective of worst case analysis, well-known to be over-pessimistic in many cases.
 In terms of the study of the potential practical usages of proposed algorithms, it is vital to study the performance averaged over specific instances. 
 Ideally, it would be useful to have a benchmark suite of real-world instances.  
 However, in the case of WSP, there is not yet any such suite publicly available, and even if it were, then it would probably not be suitable for scaling studies due to the diverse nature of instances in such suites.
 Hence, as commonly accepted in computational studies of WSP~\shortcite{Bertolissi2018,JOCO2014,WaLi10}, we use a generator of artificial instances. 
 
 When studying average case performance, it is vital to have a systematic way to decide on the parameters used in the generation of test instances, and in a fashion that gives the best chance of obtaining a meaningful and reliable insight into the behaviour of algorithms.
 With this motivation, in the latter parts of the paper we study the WSP from the perspective of phase transition (PT) or threshold phenomena.
 It has long been known that complex systems can exhibit threshold phenomena, see e.g.~\shortcite{Bollobas:book,Huberman87:phase}.
 %, FanShen2011:AIJ-PT-in-random-CSP
 They are characterised by a sharp change, or PT, in the properties of problem instances when a parameter in the instance generation is changed.
 An important discovery,  e.g.~\shortcite{cheeseman91where,mitchell92hard,SelmanKirkpatrick1996:AIJ-critical-behavior,mezard2005clustering}, was that such thresholds are also associated with decision problems that are the most challenging for search algorithms: the PT region is the source of the hardest instances of the associated decision problem. 
 In the context of NP-complete problems, this generally means the empirical average time complexity is exponential in $n$ in the PT region, i.e.\ it has a form that matches the worst case expectations, though usually with a reduced coefficient in the exponent.
 Outside of the PT region, the average complexity can drop significantly, resulting in what is usually informally called an ``easy-hard-easy'' transition.

 Testing instances without a study of the associated PT properties has the danger of accidentally picking an easy region and, as a result, obtaining overly optimistic results.
 Since the WSP is a decision problem, a fair and effective testing of the scaling of the algorithms is best done by focussing on the PT region. 
 (We also argue  later that real-world instances are likely to be close to the PT region.)

 The empirical average case analysis of FPT algorithms is more challenging than a typical PT study, since FPT problems have not just one but several size parameters.
 We give a novel empirical-average-case study that has been systematically organised for the case of FPT studies.
 We will show how the instance generator we use leads to hard problems (in the FPT sense), and also exhibits other behaviours expected from a PT.

%%%%%%%%%%%%%%%%%%%%%%%%%%%%%%%%%%%%%%%%%%%%%%%%%%%%%%%%%%%%%%
\subsection{Structure of the paper}

 Section~\ref{sec:background} gives the needed background on the theory of FPT, and the WSP\@.
 Section~\ref{sec:WSP-FPT} introduces the notion of a pattern, central for the algorithms we discuss. 
 Section~\ref{sec:algorithm} describes the new PBT algorithm in detail and demonstrates that it is FPT\@.
Section~\ref{sec:PB-form} provides the new PB formulation PBPB of the WSP, (associated with this, \appendixA{}  gives details of encoding of various UI constraints.), along with the CSP encoding.
 Section~\ref{sec:analysis-of-approaches} provides a descriptive comparison of the workings of the PBT algorithm, the existing algorithm~\shortcite{JOCO2014} which we call here Pattern User-Iterative (PUI), and the PBPB encoding. 

 Section~\ref{sec:testbed} introduces the instance generator that we use for computational experiments. 
 Some more technical aspects related to the PT are given in \appendicesBandC{}.
 Section~\ref{sec:experiments} provides the results of empirical comparisons of the algorithms and their scaling behaviours with reference to the PT, with some more results reported in \appendixD.
 Finally, in Section~\ref{sec:conclusion} we discuss overall conclusions and potential for future work.

%%%%%%%%%%%%%%%%%%%%%%%%%%%%%%%%%%%%%%%%  "PART ONE" %%%%%%%%%%%%%
\section{Background} 
\label{sec:background}

 To make the paper reasonably self-contained, this section provides background and discusses related work on parameterized algorithmics and the WSP. 

%%%%%%%%%%%%%%%%%%%%%%%%%%%%%%%%%%%%%%%%%%%%%%%%%%%%%%%%%%%%%%%%%%%%%%%%
\subsection{Parameterized Algorithms and Complexity} 
\label{subsec:FPT}

A parameterized problem $\Pi$ can be considered as a set of pairs
$(I,k)$ where $I$ is the \emph{problem instance} and $k$ (usually a nonnegative
integer) is the \emph{parameter}.  $\Pi$ is called
\emph{fixed-parameter tractable (FPT)} if membership of $(I,k)$ in
$\Pi$ can be decided by an algorithm of runtime $O(f(k)|I|^c)$, where $|I|$ is the size
of $I$, $f(k)$ is a computable function of the
parameter $k$ only, and $c$ is a constant
independent from $k$ and $I$. Such an algorithm is called an {\em FPT} algorithm.
 The parameterised algorithmics literature also uses the notation ``O$^*$'' to denote a version of big-Oh that suppresses polynomial factors in both $k$ and $|I|$ (in the same fashion that the $\tilde{O}$ notation suppresses logarithmic factors), and so we can write $f(k) |I|^{O(1)}$ as $O^*( f(k) )$. 
 There is a general expectation that a problem admitting an FPT algorithm is ``easy'' as high-order polynomials are not so likely to occur. 
 
 When the decision time is replaced by the much more powerful $|I|^{O(f(k))},$
we obtain the class XP, where each problem is polynomial-time solvable
for any fixed value of $k.$ 
There is a hierarchy of parameterized complexity
classes between FPT and XP (for each integer $t\ge 1$, there is a class W[$t$]):
$$FPT % = W[0] % AJP removed as does not match t > 1
\subseteq W[1] \subseteq W[2] \subseteq \dots \subseteq XP.$$
 For the definition of classes W[$t$], see, e.g.,~\cite{DoFe13}. 
 Due to a number of results obtained on the topic, it is widely believed that FPT$\neq$W[1], i.e.\ no W[1]-hard problem admits an FPT algorithm~\cite{DoFe13}. 

Note that, in general, the WSP parameterized by the number of steps $k$ is W[1]-hard \cite{WaLi10}, but  the WSP with only user-independent constraints is FPT \shortcite{CoCrGaGuJo14}. 
For more information on parameterized algorithms and complexity, see, e.g., \cite{DoFe13}.

%%%%%%%%%%%%%%%%%%%%%%%%%%%%%%%%%%%%%%%
\subsection{The WSP}  
\label{subsec:WSP}

 In the WSP, we are given a set $U$ of $n$ \emph{users}, a set $S$ of $k$ \emph{steps}, a set $\mathcal{A} = \set{A(u) \subseteq S : u \in U}$ of \emph{authorisation lists}, and a set $C$ of \emph{(workflow) constraints}.
 In general, a \emph{constraint} $c \in C$ can be described as a pair $c = (T_c, \Theta_c)$, where $T_c \subseteq S$ is the \emph{scope} of the constraint and $\Theta_c$ is a set of functions from $T_c$ to $U$ which specifies those assignments of steps in $T_c$ to users in $U$ that satisfy the constraint (authorisations are disregarded).

 If $W = (S, U, \mathcal{A}, C)$ is the \emph{workflow} and $T \subseteq S$ is a set of steps, then we say that a function $\pi : T \rightarrow U$ is a \emph{plan}. 
 A plan is called
%\begin{itemize}
%	\item 
    \emph{authorised} if $\pi^{-1}(u) \subseteq A(u)$ for all $u \in U$ (each user is authorised to the steps they are assigned), and
%	\item 
	\emph{eligible} if for all $c \in C$ such that $T_c \subseteq T$, $\pi|_{T_c} \in \Theta_c$ (every constraint with scope contained in $T$ is satisfied).
%\end{itemize}
 A plan that is both authorised and eligible is called a \emph{valid plan}.
 If $T = S$, then the plan is \emph{complete}.
 A workflow $W$ is \emph{satisfiable} if and only if there exists a complete valid plan.

\begin{figure}[htb]
\newcommand{\auth}{+}
\newcommand{\unau}{--}

\centering

\begin{tikzpicture}[
	%>=stealth', 
	%shorten >=1pt, 
	scale=2.5,
	thick,
	vertex/.style={circle, draw},
	edges/.style={ultra thick}]
		
	\node[vertex] (1) at (-1, 0) {$s_1$};
	\node[vertex] (2) at (-1, 1) {$s_2$};
	\node[vertex] (3) at (0, 1) {$s_3$};
	\node[vertex] (4) at (0, 0) {$s_4$};
	\node[vertex] (5) at (1, 1) {$s_5$};
		
	\path[edges, blue] (1) edge node[fill=white]{$=$} (2);
	\path[edges, red] (2) edge node[fill=white]{$\neq$} (3);
	\path[edges, red] (1) edge node[fill=white]{$\neq$} (3);
	\path[edges, red] (3) edge node[fill=white]{$\neq$} (4);
	\path[edges, red] (3) edge node[fill=white]{$\neq$} (5);
	
	\node[anchor=west] at (2, 0.5)
	{
		\begin{tabular}{@{} c @{\quad\quad} ccccc @{}}
        \toprule
		user	& $s_1$ & $s_2$ & $s_3$ & $s_4$ & $s_5$ \\
		\midrule
		$u_1$	& \unau & \unau & \auth & \unau & \unau\\
		$u_2$	& \unau & \auth & \unau & \unau & \unau\\
		$u_3$	& \auth & \auth & \unau & \auth & \auth\\
		$u_4$	& \auth & \unau & \unau & \unau & \unau\\
		$u_5$	& \unau & \unau & \auth & \auth & \auth\\
		$u_6$	& \auth & \unau & \auth & \unau & \unau\\
        \bottomrule
		\end{tabular}
	};
\end{tikzpicture}
\caption{
	An example of a WSP instance with $k = 5$ and $n = 6$.
    On the left, each step is represented with a node, and each constraint with an edge (in general, a constraint is represented with a hyper-edge but in this example all the constraints have scopes of size two).
    Red edges (marked with ``$\neq$'') represent not-equals constraints, and blue edges (marked with ``$=$'') represent equals constraints.
    The table on the right gives authorisations $A(u)$, where `+' means authorised and `$-$' means unauthorised.
    }
\label{fig:wsp-example}
\end{figure}

 Consider an example of a WSP instance with UI constraints in Figure~\ref{fig:wsp-example}.
 This particular instance includes just two types of constraints: ``equals'', defined for a scope of two steps, that requires those two steps to be assigned the same user; and ``not-equals'', also defined for a scope of two steps, that requires those two steps to be assigned different users.
 Then the following are three examples of complete plans for this instance:
\begin{align}
	& \pi(s_1) = u_1,\ 
      \pi(s_2) = u_1,\ 
      \pi(s_3) = u_2,\ 
      \pi(s_4) = u_1,\
      \pi(s_5) = u_1: \text{ eligible, non-authorised;}\\
	& \pi(s_1) = u_3,\ 
      \pi(s_2) = u_2,\ 
      \pi(s_3) = u_1,\ 
      \pi(s_4) = u_3,\
      \pi(s_5) = u_5: \text{ ineligible, authorised;}\\
\label{eq:valid-plat-example}
	& \pi(s_1) = u_3,\ 
      \pi(s_2) = u_3,\ 
      \pi(s_3) = u_6,\ 
      \pi(s_4) = u_3,\
      \pi(s_5) = u_5: \text{ authorised, eligible, i.e.\ valid.}
\end{align}
 The existence of a valid complete plan (\ref{eq:valid-plat-example}) implies that this instance is satisfiable.
 
 Clearly, not every workflow is satisfiable, and hence it is important to be able to determine whether a workflow is satisfiable or not and, if it is satisfiable, to find a valid complete plan.
 Unfortunately, the WSP is NP-hard~\cite{WaLi10}.
 However, since the number $k$ of steps is usually relatively small in practice (usually $k \ll n = |U|$ and we assume, in what follows, that $k < n$), Wang and Li~\citeyear{WaLi10} introduced its parameterisation by $k$ (we use terminology of the recent monograph~\cite{DoFe13} on parameterised algorithms and complexity).
 Algorithms for this parameterised problem were also studied in~\shortcite{CoCrGaGuJo14,JOCO2014,CrGuYe13}.
 While in general the WSP is W[1]-hard~\cite{WaLi10},
 %(this means the WSP, in general, is very unlikely to be FPT as it is widely believed that FPT$\neq$W[1]~\cite{DowneyFellows99}), 
 the WSP restricted to some practically important families of constraints, but still allowing arbitrary authorisations, is FPT~\shortcite{CoCrGaGuJo14,CrGuYe13,Roy2015,WaLi10}.
 
 Many business rules are not concerned with the identities of the users that perform a set of steps.
 Accordingly, we say a constraint $c = (T,\Theta)$ is user-independent (UI) if, whenever $\theta \in \Theta$ and $\phi: U \rightarrow U$ is a permutation, then $\phi \circ \theta \in \Theta$. 
 In other words, given a complete plan $\pi$ that satisfies $c$ and any permutation $\phi: U \rightarrow U$, the plan $\pi' : S \rightarrow U$, where $\pi'(s) = \phi(\pi(s))$, also satisfies $c$.
 The class of UI constraints is general enough in many practical cases; for example, all the constraints defined in the ANSI RBAC standard \shortcite{ansi-rbac04} are UI\@. 
 Most of the constraints studied in \shortcite{JOCO2014,CrGuYe13,WaLi10} and other papers are also UI\@. 
 Classical examples of UI constraints are the requirements that two steps are performed by either two different users (\emph{separation-of-duty}), or the same user (\emph{binding-of-duty}).
 More complex constraints such as at least-$r$ and at-most-$r$ can state that at least/at most $r$ users are required to complete some sensitive set of steps.

\subsection{Connection with Graph Colouring and Constraint Satisfaction} 
\label{sec:WSP-and-colouring}

%================= CSP view ===================

The WSP (with arbitrary constraints) can be seen as a constraint satisfaction problem (CSP), where the unary constraints are called authorisations. 
(In this paper, while WSP constraints are UI, authorisations are arbitrary.)
However, WSP is not a typical CSP: in many applications of CSP, the number of variables is much larger than the number of values (size of the domains), whereas in the WSP viewed as a CSP, usually the number of steps (number of variables) is much smaller than the number of users (size of the domains).
 
 To the best of our knowledge, WSP is the first real-world application of such CSPs, and there are only a few studies that consider relevant cases.
 One of the most closely related ones is~\shortcite{Fellows2011} which discusses the ``all different'' constraints, requiring that all the variables in the scope are assigned different values.  
 From the WSP point of view, it is a special kind of UI constraint.
 Among other results, it is shown in~\shortcite{Fellows2011} that CSP with ``all different'' constraints parametrised by the number of variables admits FPT algorithms.
 However, our study considers a much more general class of constraints. Moreover, it is concerned with practical considerations such as practically efficient algorithms and average case empirical analysis.

 We note here that the WSP with UI constraints can be considered as an extension of the hypergraph list colouring problem, where steps correspond to the hypergraph vertices, users to the colours, and user-step authorisations define colours lists.
 Each constraint $(T_c, \Theta_c)$ then defines a hyperedge connecting vertices $T_c$, but the logic of colouring a hyperedge can be arbitrarily sophisticated as long as it is colour-symmetric. 
 This colour symmetry is exactly the requirement implied by UI constraints, while the general WSP does not restrict the colouring logic at all.

%%%%%%%%%%%%%%%%%%%%%%%%%%%%%%%%%%%%%%%%%%%%%%%%%%%%%
\section{Patterns}  
\label{sec:WSP-FPT} 

 In this section, we discuss the concept of patterns as these capture equivalence classes under the permutation of users and form a vital part of the PBT algorithm presented in the next section.

%%%%%%%%%%%%%%%%%%%%%%%%%%%
\subsection{Equivalence Classes and Patterns}  
\label{sec:patterns-and-ui} 

 We define an \emph{equivalence relation} on the set of all plans.
 We say that two plans $\pi : T \rightarrow U$ and $\pi' : T' \rightarrow U$ are \emph{equivalent}, denoted by $\pi \approx \pi'$, if and only if $T = T'$, and $\pi(s) = \pi(t)$ if and only if $\pi'(s) = \pi'(t)$ for every $s, t \in T$.  
 (This is a special case of an equivalence relation defined in~\shortcite{CoCrGaGuJo14}.)
 To handle equivalence classes of plans, we introduce a notion of pattern.
 Patterns were first introduced by Cohen et al.~\citeyear{CoCrGaGuJo14} but we follow the definition of patterns from~\cite{ValuedWSP-SACMAT}.
 Let a \emph{pattern} $\pattern$ (on $T$) be a partition of $T$ into non-empty sets called \emph{blocks}.
A pattern prescribes groups (blocks) of steps to be assigned to the same user, and also requires that steps from different blocks are assigned to different users.
 In other words, a pattern $\pattern$ requires that $\pi(s) = \pi(t)$ if and only if $s, t \in B$ for some $B \in \pattern$.
 Directly from the definition, if
$B_1 \neq B_2 \in \pattern$, then $\pi(B_1) \neq \pi(B_2)$, where $\pi(B)$ is the user assigned to every step within block $B$.
 
 Patterns also provide a convenient way to test equivalence of plans.
 Let $\pattern(\pi)$ be the pattern describing the equivalence class of the plan $\pi$; it can be computed as $\pattern(\pi) = \{\pi^{-1}(u):\ u \in U,\ \pi^{-1}(u) \neq \emptyset\}$.
 Then $\pi \approx \pi'$ if and only if $\pattern(\pi) = \pattern(\pi')$ \shortcite{CoCrGaGuJo14}, see Figure~\ref{fig:equivalent-plans-example} for an example.

\begin{figure}[htb]
 	\centering
 	\begin{tabular}{ccccc}
    	  $\pi_1$ 
        && 
          $\pi_2$ 
        \\
        \\
          \begin{tabular}{cc}
              \toprule
              Step & User \\
              \midrule
              $s_1$ & $u_3$ \\
              $s_2$ & $u_3$ \\
              $s_3$ & $u_1$ \\
              $s_4$ & $u_5$ \\
              $s_5$ & $u_5$ \\
              \bottomrule
          \end{tabular}
		&
			$\approx$
		&
          \begin{tabular}{cc}
              \toprule
              Step & User \\
              \midrule
              $s_1$ & $u_3$ \\
              $s_2$ & $u_3$ \\
              $s_3$ & $u_6$ \\
              $s_4$ & $u_5$ \\
              $s_5$ & $u_5$ \\
              \bottomrule
          \end{tabular}
        \\
        \\
          $\pattern(\pi_1) = \set{ \set{s_1, s_2},\ \set{s_3}, \set{s_4, s_5} }$
        &&
          $\pattern(\pi_2) = \set{ \set{s_1, s_2},\ \set{s_3}, \set{s_4, s_5} }$
        \\
	\end{tabular}
    \caption{
    	An example of two equivalent plans $\pi_1$ and $\pi_2$, eligible for the instance defined in Figure~\ref{fig:wsp-example}.
    	The patterns $\pattern(\pi_1)$ and $\pattern(\pi_2)$ are equal.
    	Note that $\pi_1$ is authorised while $\pi_2$ is not.
    }
    \label{fig:equivalent-plans-example}
\end{figure}

 In the language of graph list-colouring, a plan is a partial colouring (since a plan is not necessarily complete), a block is a set of vertices which are required to have the same colour, and a pattern is a set of blocks with a requirement that all the blocks are assigned different colours ignoring the concrete assignment of colours.  
 The idea of patterns is somewhat related to the Zykov's method~\shortcite{Dutton1981} (for the graph colouring problem) as they both are designed to effectively exploit the colour symmetry of the problems.
 Although in WSP this ``colour symmetry'' is broken by the authorisation lists, we use a similar approach to satisfy constraints $C$ which are symmetric in WSP with UI constraints.
 
 We say that a pattern $\pattern$ is \emph{eligible} if there exists an eligible plan $\pi$ such that $\pattern(\pi) = \pattern$.
 Similarly, we say that a pattern $\pattern$ is \emph{authorised} if there exists an authorised plan $\pi$ such that $\pattern(\pi) = \pattern$.
 An eligible authorised pattern is called \emph{valid}.
 If a pattern is defined on $T = S$, then it is called \emph{complete}.

%%%%%%%%%%%%%%%%%%%%%%
\subsection{Finding an authorised plan within an equivalence class} 
\label{sec:pattern-validity}

 In this section we will describe an algorithm for finding an authorised plan $\pi$ such that $\pattern(\pi) = \pattern$ for a given pattern $\pattern$ or detecting that such a plan does not exist.

 \begin{definition}
 For a given pattern $\pattern$, an \emph{assignment graph} $G(\pattern)$ is a bipartite graph $(V_1 \cup V_2, E)$, where $V_1 = \pattern$ (i.e.\ each vertex in $V_1$ represents a block in the partition $\pattern$), $V_2 = U$ and $(B, u) \in E$ if and only if $B \in \pattern$, $u \in U$ and $B \subseteq A(u)$.
 \end{definition}
 
 We can now formulate a necessary and sufficient condition for authorisation of a pattern.

%%%%%%%%%%%%%%%%%%%%%%
\begin{proposition}
\label{prop:matching}
A pattern $\pattern$ is authorised if and only if $G(\pattern)$ has a matching covering every vertex in $V_1$.
\end{proposition}
Proof of Proposition~\ref{prop:matching} is straightforward.

Proposition~\ref{prop:matching} implies that, to determine whether an eligible pattern $\pattern$ is valid, it is enough to construct the assignment graph $G(\pattern)$ and find a maximum size matching in $G(\pattern)$.  
 It also provides an algorithm for converting a matching $M$ of size $|\pattern|$ in $G(\pattern)$ into a valid plan $\pi$ such that $\pattern(\pi) = \pattern$.
 An example of how Proposition~\ref{prop:matching} can be applied to obtain a plan from a pattern is shown in Figure~\ref{fig:assignment-graph}.

 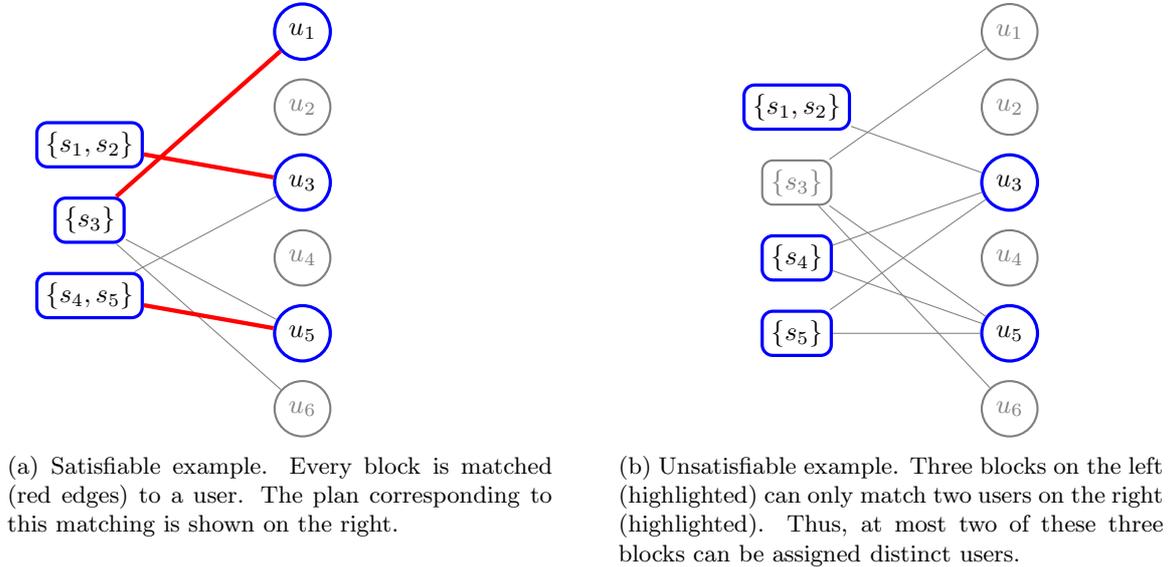
\begin{figure}[tbh]
 	\centering
    \begin{subfigure}[t]{0.45\textwidth}
    \centering
	\begin{tikzpicture}[
    	xscale=0.7,
        yscale=1,
		%>=stealth', 
		%shorten >=1pt, 
		thick,
		main node/.style={rounded corners, draw, very thick, blue, text=black},
		user node/.style={circle, draw, very thick, blue, text=black},
        secondary/.style={thick, gray}]

		\node[main node] (12)  at (0, 4.5) {$\set{s_1, s_2}$};
		\node[main node] (3) at (0, 3.5) {$\set{s_3}$};
		\node[main node] (45) at (0, 2.5) {$\set{s_4, s_5}$};

		\node[user node] (u1) at (4, 6) {$u_1$};
		\node[user node, secondary] (u2) at (4, 5) {$u_2$};
		\node[user node] (u3) at (4, 4) {$u_3$};
		\node[user node, secondary] (u4) at (4, 3) {$u_4$};
		\node[user node] (u5) at (4, 2) {$u_5$};
		\node[user node, secondary] (u6) at (4, 1) {$u_6$};

		\path[gray, thin]
    		(3) edge (u5)
      		(3) edge (u6)
    		(45) edge (u3);
    		
		\path[ultra thick, red]
    		(12) edge (u3)
    		(3) edge (u1)
    		(45) edge (u5);

        %		edge [bend right] node[left] {0.3} (2)
		%		edge [loop above] node {0.1} (1);
        
         \node at (8, 3.5) {
    %     	\begin{tabular}{cc}
    % 		\toprule
    % 			Step & User \\
    % 		\midrule
    % 			$s_1$ & $u_3$ \\
    % 			$s_2$ & $u_3$ \\
    % 			$s_3$ & $u_1$ \\
  	 %			$s_4$ & $u_5$ \\
    % 			$s_5$ & $u_5$ \\  
  		% 	\bottomrule
	   % 	\end{tabular}
         };
	\end{tikzpicture}

    \caption{
    	Satisfiable example.
    	Every block is matched (red edges) to a user.  The plan corresponding to this matching is shown on the right.}
    \label{fig:assignment-graph-sat}
    \end{subfigure}
    \hspace{0.04\textwidth}
    \begin{subfigure}[t]{0.45\textwidth}
    \centering
	\begin{tikzpicture}[
    	xscale=0.7,
        yscale=1,
		%>=stealth', 
		%shorten >=1pt, 
		thick,
		main node/.style={rounded corners, draw, very thick, blue, text=black},
		user node/.style={circle, draw, very thick, blue, text=black},
        secondary/.style={thick, gray}]

		\node[main node] (12)  at (0, 4.5) {$\set{s_1, s_2}$};
		\node[main node, secondary] (3) at (0, 3.5) {$\set{s_3}$};
		\node[main node] (4) at (0, 2.5) {$\set{s_4}$};
		\node[main node] (5) at (0, 1.5) {$\set{s_5}$};

		\node[user node, secondary] (u1) at (4, 5.5) {$u_1$};
		\node[user node, secondary] (u2) at (4, 4.5) {$u_2$};
		\node[user node] (u3) at (4, 3.5) {$u_3$};
		\node[user node, secondary] (u4) at (4, 2.5) {$u_4$};
		\node[user node] (u5) at (4, 1.5) {$u_5$};
		\node[user node, secondary] (u6) at (4, 0.5) {$u_6$};

		\path[gray, thin]
    		(3) edge (u1)
    		(3) edge (u5)
    		(3) edge (u6)
    		
		%\path[ultra thick, blue] %% changed red to blue as the meaning is different
    		(12) edge (u3)
    		(4) edge (u3)
    		(5) edge (u3)
    		(4) edge (u5)
    		(5) edge (u5);

        %		edge [bend right] node[left] {0.3} (2)
		%		edge [loop above] node {0.1} (1);
	\end{tikzpicture}
    \caption{
    	Unsatisfiable example.
        Three blocks on the left (highlighted) can only match two users on the right (highlighted).
        Thus, at most two of these three blocks can be assigned distinct users.}
    \label{fig:assignment-graph-unsat}
\end{subfigure}

    \caption{
    	Two examples of assignment graphs for eligible patterns $\pattern_1 = \set{ \set{1, 2},\ \set{3},\ \set{4, 5}}$ (Figure (a)), and $\pattern_2 = \set{ \set{1, 2},\ \set{3},\ \set{4},\ \set{5}}$ (Figure (b)).
        (The WSP instance is defined in Figure~\ref{fig:wsp-example}.)
        The matching in Figure (a) is of the maximum possible size $|\pattern_1| = 3$, and hence $\pattern_1$ is authorised.
        There exists no matching of size $|\pattern_2| = 4$, and hence $\pattern_2$ is unauthorised, i.e.\ there exists no valid plan within the corresponding equivalence class.
    }
    \label{fig:assignment-graph}
 \end{figure}

 %Depending on the algorithm employed, it takes $O(n^{2.5})$ or $O(n^3)$ time to solve the maximum matching problem, where $n$ is the number of vertices in the graph.
 %However, the assignment graph is usually highly unbalanced since $|\pattern| \le k$ and $|U| \gg k$.
 %As the maximum size of $M$ is at most $|\pattern|$ and the maximum length of an augmenting path\footnote{For a matching $M$ in a graph $G$, a path $P$ is called $M$-\emph{augmenting} if the edges of $P$ alternate between edges in $M$ and edges not in $M$ and, moreover, the first and last edges of $P$ are not in $M$.~\cite{Ahuja1993}.} in $G$ is $O(k)$, the time complexity of the Hungarian and Hopcroft-Karp methods are $O(k^3)$ and $O(k^{2.5})$, respectively~\cite{Ahuja1993}.
 %In other words, finding an authorised plan within an equivalence class is not only polynomial in $n$ but it is actually polynomial in a usually smaller $k$.

%%%%%%%%%%%%%%%%%%%%%%%%%%%%%%%%%%%%%%%%%%%%%%%%%%%%%%%%%%%%%%%%%%%%%%%%%%
\section{The New PBT Algorithm}  
\label{sec:algorithm} 

In this section we first give the core PBT algorithm, and then some improvements, including branching heuristics that significantly improved the performance. 
Cohen et al.~\citeyear{CoCrGaGuJo14} were the first to introduce a general theoretical approach based on patterns to solve WSP\@. 
In their next paper, Cohen et al.~\citeyear{JOCO2014} provided the first refined FPT algorithm and its implementation in the case of WSP with UI constraints; for a short discussion of the algorithm of~\cite{JOCO2014}, see Section~\ref{sec:branching-strategies}.

%%%%%%%%%%%%%%%%%%%%%%%%%%%%%%%%%%%%%%%%%%%%%%%%%%%%%%%%%%%%%%%%%
\subsection{Pattern-Backtracking Algorithm (PBT)} 
\label{sec:pbt}

 We call our new method \emph{Pattern-Backtracking} (PBT) as it uses the backtracking approach to systematically explore the search space of patterns.
 The key idea behind the PBT algorithm is that it is not necessary to search the space of plans; it is sufficient to search the space of patterns checking, for each eligible complete pattern, if it is authorised. 
 (This idea was used in the preliminary version \cite{KaGaGu} of this paper; a similar idea was also exploited by dos Santos et al.~\citeyear{SantosRCP15,SantosRCP17} who however use different algorithmic techniques such as Petri nets and Datalog.)
 
 The algorithm is FPT for parameter $k$, as the size of the space of patterns is a function of $k$ only, and finding a valid plan within an equivalence class (or detecting that the equivalence class does not contain any valid plans) takes time polynomial in $n$. 
 This separation helps to focus on the most important decisions first (as the search for eligible complete patterns is the hardest part of the problem).

\begin{figure}[htb]
\centering
\begin{tikzpicture}[
	tree node/.style={rounded corners, draw, anchor=west, very thick, blue, text=black}
	]
	
	\node[tree node] (root) at (0, 0) {$\big\{ \big\}$};

	\node[tree node] (1n1) at (1.5, 0) {$\big\{ \{ s_1 \} \big\}$};

	\node[tree node] (2n1) at (4, 0.7) {$\big\{ \{ s_1, s_2 \} \big\}$};
	\node[tree node] (2n2) at (4, -0.7) {$\big\{ \{ s_1 \},\ \{ s_2 \} \big\}$};

	\node[tree node] (3n1) at (7.5, 2.2) {$\big\{ \{ s_1, s_2, s_3 \} \big\}$};
	\node[tree node] (3n2) at (7.5, 1.2) {$\big\{ \{ s_1, s_2 \},\ \{ s_3 \} \big\}$};
	\node[tree node] (3n3) at (7.5, 0) {$\big\{ \{ s_1, s_3 \},\ \{ s_2 \} \big\}$};
	\node[tree node] (3n4) at (7.5, -1) {$\big\{ \{ s_1 \},\ \{ s_2, s_3 \} \big\}$};
	\node[tree node] (3n5) at (7.5, -2) {$\big\{ \{ s_1 \},\ \{ s_2 \},\ \{ s_3 \} \big\}$};
	
	\path[red, very thick, ->]
		(root) edge (1n1)
		
		(1n1) edge (2n1)
		(1n1) edge (2n2)
		
		(2n1) edge (3n1)
		(2n1) edge (3n2)
		(2n2) edge (3n3)
		(2n2) edge (3n4)
		(2n2) edge (3n5);
\end{tikzpicture} 
\caption{Illustration of the backtracking mechanism within PBT.}
\label{fig:backtracking}
\end{figure}
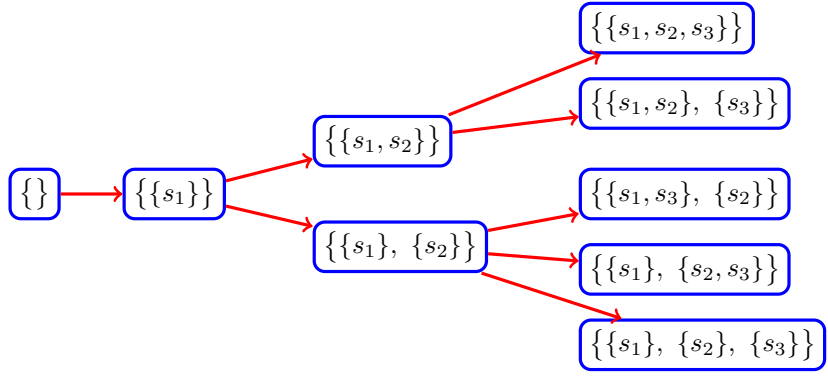
 
 The basic version of PBT is a DPLL-style backtracking algorithm traversing the space of patterns, see Figure~\ref{fig:backtracking}.
 The search starts with an empty pattern $\pattern$, and then, at each forward iteration, the algorithm adds one step to $\pattern$.
 Let $S(\pattern)$ denote the steps included in $\pattern$.
 The branching captures the simple notion that any step $s$ not in $S(\pattern)$ must either be placed in some block of $\pattern$, or else form a new block.  
 These choices are disjoint and so it follows that the search tree is indeed a tree; the same pattern cannot be produced in different ways. 
 Hence the tree can be searched using depth-first search (DFS)\@. 
 This is the key difference between PBT and the previous pattern-based WSP algorithm, PUI\@.
 PUI explores the space of patterns in a breadth-first way, in the worst case building and storing an exponential-in-$k$ number of valid incomplete patterns.
 This leads to the space complexity of PUI also being exponential in $k$.  
 For details see Section~\ref{sec:branching-strategies}.
 
 Pruning in the basic version of PBT is based on constraint violations only.
 Hence, any leaf node at level $k$ in the basic version corresponds to an eligible complete pattern.
 Using Proposition~\ref{prop:matching}, the algorithm verifies the authorisation of the pattern.
 Once an eligible authorised (i.e.\ valid) complete pattern is found, the algorithm terminates returning the corresponding valid complete plan.

 In practice (the source code of our implementation of PBT is publicly available, see~\shortcite{SourceCodes}), we interleave the backtracking with authorisation tests.
 The real PBT uses an improved branch pruning that takes into account both constraints and authorisation.
 In particular, for every node of the search tree the algorithm verifies if the pattern is authorised, and if not, then that branch of the search tree is pruned.

 The calling procedure for the PBT algorithm is shown in Algorithm~\ref{alg:start}, which in turn calls the recursive search function in Algorithm~\ref{alg:recursion}. 
 The recursive function tries all possible extensions $\pattern'$ of the current pattern $\pattern$ with step $s \notin S(\pattern)$.
 The step $s$ is selected heuristically (line~\ref{line:select-s}) using the empirically-tuned function $\rho(s, \pattern)$, which indicates the importance of step $s$ in narrowing down the search space.
 The implementation of $\rho(s, \pattern)$ depends on the specific types of constraints involved in the instance and should reflect the intuition regarding the structure of the problem.
 See function (\ref{eq:rho}) in Section~\ref{sec:branching-heuristic} for a particular implementation of $\rho(s, \pattern)$ for the types of constraints we used in our computational study.

\SetKwInOut{Input}{input}
\SetKwInOut{Output}{output}

%%%%%%%%%%%%%%%%%%
\begin{algorithm}[htb]
	\caption{Backtracking search initialisation (entry procedure of PBT)
    \label{alg:start}}
    
	\Input {WSP instance $W = (S, U, \mathcal{A}, C)$}
	\Output {Valid plan $\pi$ or UNSAT}

	Initialise $\pattern, G, M \gets \emptyset$, 
	$\pi \gets \text{Recursion}(\pattern, G, M)$\;

	\Return{$\pi$} ($\pi$ may be UNSAT here)\;
\end{algorithm}
 
\begin{algorithm}[htb]
	\caption{Recursion$(\pattern, G, M)$ (recursive function for backtracking search)
    \label{alg:recursion}}
 
	\Input {Pattern $\pattern$, authorisation graph $G = G(\pattern)$ and a matching $M$ in $G$ of size $|\pattern|$}

	\Output {Valid plan or UNSAT if no valid plan exists in this branch of the search}
	\If {$S(\pattern) = S$}
	{
		\Return {plan $\pi$ defined by matching $M$}\;
	}
	\Else
	{
		Select an unassigned step $s \in S \setminus S(\pattern)$ that maximises $\rho(s, \pattern)$ (for details see Section~\ref{sec:branching-heuristic})\; \label{line:select-s}
		Compute all the eligible patterns $X$ extending $\pattern$ with step $s$ (for details see Section~\ref{sec:constraints-pruning})\;  \label{line:extend-pattern}
		\ForEach {$\pattern' \in X$}
		{
			Produce an assignment graph $G' = G(\pattern')$ (for details see Section~\ref{sec:authorisations-pruning})\; \label{line:assignment-graph}
			\If {there exists a matching $M'$ of size $|\pattern'|$ in $G'$ \label{line:matching}} 
			{
				$\pi \gets \text{Recursion}(\pattern', G', M')$\;
				\If {$\pi \neq \text{UNSAT}$}
				{
					\Return{$\pi$}\;
				}
			} 
		}
	}

	\Return{UNSAT (for a particular branch of recursion; does not mean that the whole instance is unsat)}\;
\end{algorithm}

%%%%%%%%%%%%%%%%%%%%

 Authorisation tests are implemented in lines~\ref{line:assignment-graph} and~\ref{line:matching} and discussed in more detail in Section~\ref{sec:authorisations-pruning}.
 All the propagation and eligibility tests are implemented in line~\ref{line:extend-pattern} and discussed in Section~\ref{sec:constraints-pruning}.

\subsection{Authorisation-Based Pruning}
\label{sec:authorisations-pruning}

 Although it would be sufficient to check authorisations of complete patterns only, testing authorisations at each node allows us to prune branches if they contain no authorised plans.
 In doing so, it is not necessary to generate the assignment graph from scratch in every node:
 \begin{itemize}
	\item 
    If $\pattern'$ is obtained from $\pattern$ by extending some block $B \in \pattern$, then $G'$ can be obtained by removing all existing edges $(B, u)$, $u \in U$, and adding edges $(B', u)$ such that $B' \subseteq A(u)$, where $B'$ is the extended block.
    One may note that the edge set $(B', u)$, $u \in U$, is a subset of the edge set $(B, u)$, $u \in U$; however, as we will explain below, we do not store the entire set of edges $(B, u)$, $u \in U$, and hence cannot exploit this property.
	
    \item 
    If $\pattern'$ is obtained from $\pattern$ by adding a new block $\set{s}$, then $G'$ can be obtained by adding a new vertex $B = \set{s}$ to $G$ and adding the edge $(B, u)$, for each $u \in U$ such that $s \in A(u)$.
 \end{itemize}
 Similarly, it is possible to recover $G$ from $G'$; hence, we can reuse the same data structure updating it in every node, with updates taking only $O(kn)$ time.
 It is also not necessary to compute the maximum matching $M'$ in $G'$ from scratch.
 We can obtain matching $M'$ in $G'$ from matching $M$ in $G$ in $O(kn)$ time. 
 Indeed, if a new block is added, then a maximum matching of $G'$ can have at most one edge more than $M$. 
 If an existing block in $G$ is extended, we can remove the old edge from $M$ adjacent to the extended block, and then again a maximum matching can have at most one edge more than the updated $M$.
 By Berge's Augmenting Path Theorem, $G'$ has a matching of size $|M|+1$ if and only if $G'$ has an $M$-augmenting path~\cite{West}. 
 Thus, we only need to try to find an $M$-augmenting path in $G'$.
 The augmenting path algorithm requires $O(|V(G')|+|E(G')|)=O(kn)$ time.
 
 In fact, we will never need more than $k \ge |\pattern|$ edges from a vertex $B \in \pattern$ of the assignment graph.
 Hence, when calculating the edge set for a $B \in \pattern$, we can terminate when reaching $k$ edges.
 As a result, we only need $O(k^2)$ time to update/extend the matching $M$, but we still need $O(kn)$ time to update or extend the assignment graph. 

 Note that incremental maintenance of the matching $M$ in every node of the search tree actually improves the worst-case time complexity compared to computation of $M$ in the leaf nodes only, despite incremental maintenance being unable to take advantage of the Hopcroft-Karp method.
 Indeed, in the worst case (when the search tree is of its maximum size), each internal node  of the tree has at least two children.
 Hence, the total number of nodes is at most twice the number of leaf nodes.
 Then the total time spent by PBT on authorisation validations is $O((kn + k^2) p)$, where $p$ is the total number of complete patterns.
 Observe that validation of authorisations of complete patterns only, as in the basic version of PBT, would take $O((kn + k^{2.5}) p)$ time.

 When updating the assignment graph, we have to compute the edge set for a block $B \in \pattern$, and this takes $O(kn)$ time, i.e.\ relatively expensive due to the potentially large $n$.
 Since, during the search, we are likely to compute the edge set for many of the blocks $B$ multiple times, we cache these edge sets for each block $B$.
 The number of blocks generated during the search (bounded by $2^k$) can be prohibitive for caching all the edge sets, and, thus, we erase the cache every time it reaches 16\,384 records (this constant was obtained by parameter tuning). 
 
%%%%%%%%%%%%%%%%%%%%%%%%%%%%%%%%%%%%%%%%%%%%%%%%%%%%%%%%%%%%%%%%%%%%%%%
\subsection{Eligibility-Based Pruning}
\label{sec:constraints-pruning}

 While much of the implementation of PBT is generic enough to handle any UI constraints, some heuristics make use of our knowledge about the types of constraints present in our test instances.
 In particular, we assume that the instances include only not-equals, at-most and at-least constraints, as in \shortcite{Basin2012,JOCO2014,Roy2015,WaLi10}.
 Note that all of the methods discussed in this paper make use of this information -- which is a common practice when developing decision support systems.
 Hence, this makes our experiments more realistic and also helps to fairly compare all the approaches, as the old ones were already specialised.
 
 Since all our constraints are only restricting the number of distinct users to be assigned to the scope, we implemented incremental maintenance of corresponding counters for each at-most and at-least constraint.
 The implementation of line~\ref{line:extend-pattern} scans all the constraints with scopes that include step $s$ and verifies which of the blocks in $\pattern$ can be extended with $s$ without violating the constraint.
 Similarly, the algorithm considers creation of a new block $\{ s \}$.
 
 Note that pruning based on some constraint $c$ does not always require that $T_c \subseteq S(\pattern')$.
 E.g., an at-most-3 constraint with scope $\set{ s_1, s_2, \ldots, s_5 }$ can
%  be used to 
 prune pattern $\set{ \set{s_1}, \set{s_2}, \set{s_3}, \set{s_4} }$.
 Our implementation of PBT produces only the extensions of the pattern that will not immediately break any constraint (disregarding authorisations).

%%%%%%%%%%%%%%%%%%%%%%%%%%%%%%%%%%%%%%%%%%%%%%%%%%%%%%%%%%%%%
\subsection{Branching Heuristic} 
\label{sec:branching-heuristic}

 This section essentially describes line~\ref{line:select-s} of Algorithm~\ref{alg:recursion}.
 As standard in tree-based search, the selection of the branch variable, or step $s$ in PBT, can make a big difference to the size of the search tree. 
 The selection is performed by taking the step with the highest value of a score $\rho(s, \pattern)$.
 The score is designed with the intention to encourage early pruning of the search tree.

 In order to define components of the score function $\rho(s, \pattern)$, we split the constraints $C$ into the not-equals, $C_{\neq}$, the at-most $C_\le$, and at-least $C_\ge$.
 Different constraints have different strengths in terms of the pruning of the search, and so are permitted to be counted with different weights. 
 Specifically:

\begin{equation}
\label{eq:rho}
\rho(s, \pattern) = \alpha_{\neq} \rho_{\neq}(s) + \alpha_{\neq, \le} \rho_{\neq, \le}(s) + \alpha_{\ge, \le} \rho_{\ge, \le}(s) + \alpha^0_\leq \rho^0_\leq(s, \pattern) + \alpha^1_\leq \rho^1_\leq(s, \pattern) + \alpha^2_\leq \rho^2_\leq(s, \pattern),
\end{equation}
where the $\alpha$'s are numeric weights, and 
\begin{itemize}
	\item 
    $\rho_{\neq}(s) = | \set{ c :\; c \in C_{\neq},\ s \in T_c } |$ is the number of not-equals constraints $c$ that cover step $s$.
    The more constraints cover a step, the more important it is in general.
        
	\item 
    $\rho_{\neq, \le}(s) = | \set{ c :\; c \in C_{\neq},\ s \in T_c,\ \exists c' \in C_\le,\ T_c \subseteq T_{c'} } |$ is the count of the not-equals constraints $c$ covering $s$ and also covered by some at-most constraint $c'$.
    Not-equals and at-most constraints are in conflict, and their interaction is likely to further restrict the search.
    
	\item 
    $\rho_{\ge, \le}(s) = | \set{ c \in C_{\le},\ c' \in C_{\ge} :\; |T_{c} \cap T_{c'}| \ge 3,\ s \in T_c \cap T_{c'} } |$ is the number of pairs of constraints (at-least, at-most) covering $s$ with intersections of at least 3 steps.
    Large intersections of at-most and at-least constraints are rare but do significantly reduce the search space.
	    
	\item 
    $\rho^i_\leq(s, \pattern) = | \set{c :\; c \in C_{\le},\ s \in T_c,\ |T_c \cap S(\pattern)| = r - i} |$, where $r$ is the parameter of the at-most-$r$ constraint, is the number of at-most-$r$ constraints such that they can cover at most $i$ new blocks of the pattern.
	For example, $i = 0$ means that $r$ distinct users are already assigned to the scope $T_c$ and, hence, the choice of users for $s$ is limited to those $r$ users.    
\end{itemize}

 We emphasise that these terms might seem `quirky', but they are the result of extensive experimentation with many different ideas, and they represent the best choices found.  
 Space precludes discussion of other possibilities that were tried but not found to be effective in our instances.

 Note that $\rho_{\neq}(s)$, $\rho_{\neq, \le}(s)$, $\rho{\ge, \le}(s)$ do not depend on the state of the search and can be pre-calculated.
 Also $\rho^i_{\le}(s, \pattern)$ can make use of the counters that we maintain to speed-up eligibility tests (see above).

 The values of parameters $\alpha_{\neq}$, $\alpha_{\neq, \le}$, $\alpha_{\ge, \le}$, $\alpha^0_\leq$, $\alpha^1_\leq$ and $\alpha^2_\leq$ were selected empirically using a bespoke automated parameter tuning method.
 We found out that the algorithm is not very sensitive to the values of these parameters, and we settled at $\alpha_{\neq} = 3$, $\alpha_{\neq, \le} = 4$, $\alpha_{\ge, \le} = 2$, 
 $\alpha^0_\leq = 40$, $\alpha^1_\leq = 4$ and $\alpha^2_\leq = 0$.
 
 Note that the function does not account for at-least constraints except for rare cases of large intersection in at-least and at-most constraints.  
 This reflects our empirical observation (also confirmed analytically in \appendixB) that the at-least constraints are usually relatively weak in our instances and rarely help in pruning branches of the search tree. 

 We conducted empirical studies of the branching factor and the depth of the search.
 Our results show significant reduction of both parameters when the branching heuristic was enabled, greatly reducing the size of the search trees.
 For example, at $k = 40$, with the branching heuristic disabled, the overall number of nodes in the search tree is about $5.7 \cdot 10^{10}$, of which around 90\% were at the depths 23--28.
 The average branching factor at the depth 23 reaches its maximum of 5.55.
 With the branching heuristic enabled, the overall number of nodes is about $1.0 \cdot 10^6$, of which 90\% are at the depths 16--27.
 The average branching factor reaches its maximum of 2.56 at the depth of 7.
 
 \subsection{Worst-Case Analysis of PBT} 
\label{sec:worst-case-analysis}

 Recall that, in the worst case, the total number of patterns considered by the PBT algorithm is less than twice the number of complete patterns. 
 Observe that the number of complete patterns equals the number of partitions of a set of size $k$, i.e.\ the $k$'th Bell number $\mathcal{B}_k$ which is $O(2^{k \log_2 k})$. 
 Finally, observe that the PBT algorithm spends time $O(k^2 + kn)$ on each node of the search tree, assuming that checking of all relevant constraints takes $O(k^2)$ time.
 Thus, the time complexity of the PBT algorithm is $O(\mathcal{B}_k \cdot (k^2 + kn))=O^*(2^{k\log_2 k})$.
 In fact, the running time $O^*(2^{k\log_2 k})$ is likely to be optimal. 
 Indeed, \citeA{GW2016} proved that unless the Strong Exponential Time Hypothesis~\shortcite{ImPa01} fails, which is generally considered quite unlikely, there is no $O^*(c^{k\log_2 k})$-time algorithm for WSP with UI constraints with any $c<2$.

 The PBT algorithm follows the depth-first search order and, hence, stores only one pattern at a time.  
 It also maintains a subgraph of the assignment graph with only $O(k^2)$ edges. 
 (See Section~\ref{sec:authorisations-pruning} where the upper bound $k$ on the degrees of blocks in the matching graph is explained.)
 Hence, the space complexity of the algorithm is $O(k^2)$, i.e.\ smaller than the problem itself ($O(kn)$).
 Note that small space complexity is very good for reducing cache misses.

%%%%%%%%%%%%%%%%%%%%%%%%%%%%%%%%%%%%%%%%%%%%%%%%%%%%%%%%%%%%%%%%%%%
%%%%%%%%%%% PBPB formulation %%%%%%%%%%%%%%%%%%%%%%%%%%%%%%%%%%%%%%
%%%%%%%%%%%%%%%%%%%%%%%%%%%%%%%%%%%%%%%%%%%%%%%%%%%%%%%%%%%%%%%%%%%
\section{Pseudo-Boolean and Constraint Satisfaction Formulations}  
\label{sec:PB-form}

 To make the paper self-contained, we provide in Section~\ref{sec:old-formulation} the old pseudo-Boolean formulation~\cite{JOCO2014}, and in Section~\ref{sec:new-formulation} we show how it can be extended to exploit the FPT nature of the problem.
 We also give a CSP formulation of the problem in Section~\ref{sec:CSP-formulation}.
 
%%%%%%%%%%%%%%%%%%%%%%%%%%%%%%%%%%%%%%%%%%%%%%%%%%%%%%%%%%%%%%%%%%%
\subsection{Old Pseudo-Boolean Formulation (UDPB)}   
\label{sec:old-formulation}
 
 The main decision variables in the old formulation are $x_{s, u}$, $s \in S$, $u \in U$, where user $u$ is assigned to step $s$ if and only if $x_{s, u} = 1$.
 Because $x_{s, u}$ directly assigns a step to a particular user, we call this formulation \emph{User Driven PB} (UDPB).
 We also introduce auxiliary variables $y_{c, u}$ and $z_{c, u}$ for at-least and at-most constraints $c$, respectively.
 Variables $y_{c, u}$ and $z_{c, u}$ are used to bound the number of distinct users assigned to the constraints' scopes.% (recall, see Section~\ref{sec:constraints-pruning}, that our test instances include not-equals, at-least and at-most constraints).
 
 The old formulation is presented below in~(\ref{eq:old-one-user})--(\ref{eq:old-define-x}).
 We use notation $A^{-1}(s) = \{ u \in U :\; s \in A(u) \}$ for the set of users authorised for step $s \in S$.

\begin{align}
\label{eq:old-one-user}
& \sum_{u \in U} x_{s, u} = 1 
	&& \forall s \in S, \\
\label{eq:old-authorisations}
& x_{s, u} = 0 
	&& \forall s \in S \text{ and } \forall u \in U \setminus A^{-1}(s), \\
\label{eq:old-not-equals}
& x_{s_1, u} + x_{s_2, u} \le 1 
	&& \forall \text{ not-equals constraint with scope $\set{s_1, s_2}$} \text{ and } \forall u \in U,\\
\label{eq:old-y}
& y_{c, u} \ge x_{s, u} 
	&& \forall \text{ at-most-$r$ constraints $c$ with scope $T_c$},\ \forall s \in T_c \text{ and } \forall u \in U,\\
\label{eq:old-at-most}
& \sum_{u \in U} y_{c, u} \le r 
	&& \forall \text{ at-most-$r$ constraint $c$},\\
\label{eq:old-z}
& z_{c, u} \le \sum_{s \in T_c} x_{s, u} 
	&& \forall \text{ at-least-$r$ constraint $c$ with scope $T_c$}, \text{ and } \forall u \in U,\\
\label{eq:old-at-least}
& \sum_{u \in U} z_{c, u} \ge r 
	&& \forall \text{ at-least-$r$ constraint $c$}\\
& y_{c, u} \in \{ 0, 1 \} 
	&& \forall \text{ at-most-$r$ constraint $c$ and } \forall u \in U,\\
& z_{c, u} \in \{ 0, 1 \} 
	&& \forall \text{ at-least-$r$ constraint $c$ and } \forall u \in U,\\
\label{eq:old-define-x}
& x_{s, u} \in \{ 0, 1 \} 
	&& \forall s \in S \text{ and } \forall u \in U.
\end{align} 

 Constraints~(\ref{eq:old-one-user}) guarantee that exactly one user is assigned to each step.
 Constraints~(\ref{eq:old-authorisations}) implement authorisations.
 Constraints~(\ref{eq:old-not-equals})--(\ref{eq:old-at-least}) define WSP constraints (recall that our test instances include only not-equals, at-least and at-most UI constraints, although the UDPB formulation can obviously encode any computable WSP constraints, whether UI or not).

%%%%%%%%%%%%%%%%%%%%%%%%%%%%%%%%%%%%%%%%%%%%%%%%%%%%%%%%%%%%%%%%%%%
\subsection{New Pseudo-Boolean Formulation (PBPB)}
\label{sec:new-formulation}

 A contribution of this paper is a new pseudo-Boolean formulation (\ref{eq:mx-symmetric})--(\ref{eq:mx-x-define}) exploiting the FPT nature of the problem.
 This formulation, which we call \emph{Pattern Based PB} (PBPB), was inspired by formulations of the graph colouring problem \cite{DukanovicRendl2008:SDP-heuristic-GCP,Dutton1981,Lovasz1979:shannon}. 
 %\footnote{The matrix $M$ can be thought of as a `Merge matrix' as it controls whether or not two steps are effectively merged by being required to have the same user.}
 In particular, steps $s_1$ and $s_2$ are assigned the same user if and only if $M_{s_1, s_2} = 1$ (we assume $M_{s,s} = 1$ for every $s \in S$).
 Such variables are not concerned with the identity of users and, thus, are more effective when handling UI constraints.
 This is the same idea as behind colour matrix in~\cite{DukanovicRendl2008:SDP-heuristic-GCP} which preserves the colour symmetry and encapsulates only the decisions that matter at the upper level of the search.
 However it extends such usage in two ways. Firstly, WSP with UI constraints has a richer set of constraints, defined on ``hyperedges''. 
 Secondly, the matrix $M$ is tightly integrated with the non-UI authorisations.
 Thus, we still use the $x$ variables with the same meaning as in the UDPB formulation but complement them with the new variables $M$.

 The formulation (\ref{eq:mx-symmetric})--(\ref{eq:mx-x-define}) is given for only the specific constraints (namely, at-most-3 and at-least-3 constraints with scope of size 5, and not-equals constraints); for formulations of other constraints, including general UI constraints, see \appendixA{}.

\begin{align}
\label{eq:mx-symmetric}
& M_{s_1, s_2} = M_{s_2, s_1} 
	&& \forall s_1 \neq s_2 \in S, \\
\label{eq:mx-diagonal}
& M_{s, s} = 1
	&& \forall s \in S, \\
\label{eq:mx-link-m1}
& M_{s_1, s_2} \ge M_{s_1, s_3} + M_{s_2, s_3} - 1
	&& \forall s_1 \neq s_2 \neq s_3 \in S, \\
\label{eq:mx-link-m2}
& M_{s_1, s_2} \le M_{s_2, s_3} - M_{s_1, s_3} + 1
	&& \forall s_1 \neq s_2 \neq s_3 \in S, \\
\label{eq:mx-one-user}
& \sum_{u \in U} x_{s, u} = 1 
	&& \forall s \in S, \\
\label{eq:mx-link1}
& x_{s_1, u} - x_{s_2, u} \le 1 - M_{s_1, s_2} 
	&& \forall s_1 \neq s_2 \in S \text{ and } \forall u \in U, \\
\label{eq:mx-link2}
& x_{s_1, u} + x_{s_2, u} \le 1 + M_{s_1, s_2} 
	&& \forall s_1 \neq s_2 \in S \text{ and } \forall u \in U, \\
\label{eq:mx-authorisations}
& x_{s, u} = 0
	&& \forall s \in S \text{ and } \forall u \in U \setminus A^{-1}(s), \\
\label{eq:mx-not-equals}
& M_{s_1, s_2} = 0 
	&& \forall \text{ not-equals constraint with scope $\set{s_1, s_2}$}, \\
\label{eq:mx-at-most}
& \sum_{s_1 < s_2 \in T} M_{s_1, s_2} \ge 2
	&& \forall \text{ at-most-3 constraint with scope $T$}, |T| = 5, \\
\label{eq:mx-at-least}
& \sum_{s_1 < s_2 \in T} M_{s_1, s_2} \le 3 
	&& \forall \text{ at-least-3 constraint with scope $T$, } |T| = 5, \\
& M_{s_1, s_2} \in \{ 0, 1 \} 
	&& \forall s_1, s_2 \in S,\\
\label{eq:mx-x-define}
& x_{s, u} \in \{ 0, 1 \} 
	&& \forall s \in S \text{ and } \forall u \in U.
\end{align}
  
 %As explained above, we still need the $x_{s, u}$ variables to define authorisations, see (\ref{eq:mx-authorisations}).
 The $x_{s, u}$ variables, used to define authorisations in (\ref{eq:mx-authorisations}), need to be linked to the $M_{s_1, s_2}$ variables.
 In particular, if $M_{s_1, s_2} = 1$ then we require that $x_{s_1, u} = x_{s_2, u}$ for every $u \in U$, see (\ref{eq:mx-link1}), and if $M_{s_1, s_2} = 0$ then $x_{s_1, u} + x_{s_2, u} \le 1$, i.e.\ at least one of $x_{s_1, u}$ and $x_{s_2, u}$ has to take value 0, see (\ref{eq:mx-link2}).
 To improve propagation, we formulate optional (transitive closure) constraints (\ref{eq:mx-link-m1}) and (\ref{eq:mx-link-m2}).
 These constraints are entailed by the link between the $M$ and $x$ variables in (\ref{eq:mx-one-user})--(\ref{eq:mx-link2}), but adding them increases the propagation avoiding the cost of extra reasoning involving the $x$ variables.

 Constraints (\ref{eq:mx-not-equals}) encode not-equals, (\ref{eq:mx-at-least}) encode at-least-3 and (\ref{eq:mx-at-most}) encode at-most-3 (these are the constraints present in our instances; for details see Section~\ref{sec:instance-generator}).
 It is useful that (\ref{eq:mx-not-equals})--(\ref{eq:mx-at-least}) involve only the $M$ variables; together with  (\ref{eq:mx-symmetric})--(\ref{eq:mx-link-m2}) they guarantee that a solution corresponds to an eligible pattern. 
 Hence, (\ref{eq:mx-not-equals})--(\ref{eq:mx-at-least}) correspond to the upper level of the search, i.e.\ the search over the space of patterns.

 It is easy to observe that any constraint that is expressed only in terms of the $M$'s is automatically UI, as it does not involve the $x$ variables (users), and so cannot change with permutations of them.
 The following proposition states that the converse also applies.
\begin{proposition} 
\label{prop:all-UI-are-M}
 On solving an instance of the WSP, the decision variables $M$ are sufficient to encode any UI constraint.
\end{proposition}
\begin{proof}
 By definition, any WSP constraint $c = (T_c, \Theta_c)$ can be defined by the set $\Theta_c$ of all the plans $\pi : T_c \rightarrow U$ that are eligible for $C = \set{c}$.
 Moreover, if a constraint $c$ is UI then $\pi \in \Theta_c$ implies that $\pi' \in \Theta_c$ for every $\pi' \approx \pi$.
 Then it follows that a UI constraint can be described by listing equivalence classes of plans or, equivalently, patterns on $T_c$.

 Recall that a pattern can be uniquely described with the $M$ variables; in particular, a pattern $\pattern$ can be described as $M_{s', s''} = 1$ for every $s', s'' \in B$, $B \in \pattern$, and $M_{s', s''} = 0$ for every $s' \in B'$, $s'' \in B''$ and $B' \neq B'' \in \pattern$.
 Then it is easy to exclude a pattern via a linear inequality expressed in variables $M$.
 
 Let $\overline{\mathit{PAT}}$ be the list of all patterns on $T_c$ disobeying a UI constraint $c = (T_c, \Theta_c)$. 
 Then, in general, we can encode a UI constraint $c$ with constraints (\ref{eq:mx-symmetric}), (\ref{eq:mx-diagonal}) and
\begin{equation}
	\sum_{B \in \pattern} \sum_{s' < s'' \in B} (1 - M_{s', s''}) 
    	+ \sum_{B' \neq B'' \in \pattern} \sum_{s' \in B'} \sum_{s'' \in B'', s' < s''} M_{s', s''} 
    \ge 1 
    \qquad 
    \forall \pattern \in \overline{\mathit{PAT}} \,.
    \label{eq:UI-as-PB}
\end{equation}
%\qed
\end{proof}
%The proof of Proposition~\ref{prop:all-UI-are-M} is given in \ref{ap:proofs}.
 To see how Proposition~\ref{prop:all-UI-are-M} works, consider the following example.
 To require that $\pattern \neq \set{ \set{s_1, s_2}, \set{s_3} }$ (which is a UI constraint), or, equally,
 $$
	M \neq
	\left(
		\begin{array}{ccc}
			1 & 1 & 0 \\
			1 & 1 & 0 \\
			0 & 0 & 1 
		\end{array}
	\right)
 $$
 one can write
$$
	(1 - M_{12}) + M_{13} + M_{23} \ge 1 \,.
$$
 Note that to encode a UI constraint in this way, one may need numerous constraints to prohibit multiple patterns.
 %This might be impractical if $|\overline{\mathit{PAT}}|$ is large. \AG{What is $\overline{\mathit{PAT}}$ here? It's not defined anymore. Just simplify-combine this phrase with the previous.}
  In \appendixA{} we give more compact approaches to formulate some standard UI constraints and also discuss why our encodings (\ref{eq:mx-at-most}) and (\ref{eq:mx-at-least}) are correct.
  
% \clearpage

\subsection{CSP Formulation}
\label{sec:CSP-formulation}

 As noted above, the WSP with UI constraints is effectively a Constraint Satisfaction Problem (CSP)\@.
 There is a variable $x_s$ for each step $s \in S$, with the domain $A^{-1}(s)$.
 Note that any CSP constraint that is not sensitive to specific values, but only there equality or not corresponds to a UI constraint, and vice versa.
 
 The not-equals constraint with scope $T = \{ i, j \}$ is encoded as
\begin{equation}
\label{eq:csp-not-equals}
x_i \neq x_j
\end{equation}

 In order to encode an at-least-3 constraint with scope $T$ of size 5, we request that there exists at least one subset of steps in $T$ of cardinality 3 with all different values.
 Let $\mathcal{X}$ be the set of all subsets of $T$ of cardinality 3.
 Let $\mathcal{X}_r \subset T$ be the $r$th element of $\mathcal{X}$.
 Then the constraint can be encoded using auxiliary Boolean variables $v_r$ as follows:
\begin{align}
& \bigvee_{r=1}^{|\mathcal{X}|} v_r && \\
\label{eq:csp-at-least}
& v_r \implies  x_i \neq x_j && r = 1, 2, \ldots, |\mathcal{X}|,\quad i < j \in \mathcal{X}_r \\
& v_r \in \{0, 1\} && r = 1, 2, \ldots, |\mathcal{X}|.
\end{align}

 In order to encode an at-most-3 constraint with scope $T$ of size 5, we request that in every subset $Y \subset T$ of cardinality 4, there exists at least one pair of steps $i \neq j \in Y$ such that $x_i = x_j$.
 (Indeed, an at-least-3 out of 5 constraint with scope $T$ is falsified if and only if any four steps in $T$ are assigned all different values.)
 Let $\mathcal{Y}$ be the set of all subsets of $T$ of cardinality 4.
 Then the constraint can be encoded as follows:
\begin{align}
\label{eq:csp-at-most}
& \bigvee_{i \neq j \in Y} x_i = x_j && \forall Y \in \mathcal{Y} 
\end{align}

%%%%%%%%%%%%%%%%%%%%%%%%%%%%%%%%%%%%%%%%%%%%%%%%%%%%%%%%%%%%%%%%%%%
%%%%%%%%%% Analysis %%%%%%%%%%%%%%%%%%%%%%%%%%%%%%%%%%%%%%%%%%%%%%%
%%%%%%%%%%%%%%%%%%%%%%%%%%%%%%%%%%%%%%%%%%%%%%%%%%%%%%%%%%%%%%%%%%%
\section{Analysis of WSP Solution Approaches} 
\label{sec:analysis-of-approaches} 
 
 In this section we analyse and compare the existing and new WSP solution approaches.
 In Section~\ref{sec:branching-strategies}, we discuss different branching strategies and how they are linked to performance of WSP algorithms.
 Section~\ref{sec:newPB-properties} discusses properties of PBPB with respect to both the upper level search for eligible patterns and of the assignment problems that arise at the lower level. 
 (For asymptotic worst-case analysis of PBT and PUI please see Section~\ref{sec:worst-case-analysis})\@. 

%%%%%%%%%%%%%%%%%%%%%%%%%%%%%%%%%%%%%%%%%%%%%%%%%%%%%%%%%%%%%%%%%%%
\subsection{Branching Strategies} 
\label{sec:branching-strategies}

 In this section, we use the language of patterns and of the PBPB encoding in order to discuss possible branching strategies in a WSP algorithm. 
 The first key observation is that any pattern can be described with $M_{ij}$ variables.
 The matrix of $M$ variables corresponding to a complete pattern is exactly a permutation of block-ones-diagonal matrix, where a block in the matrix corresponds to a block of the pattern.
 A pattern as used within PBT is then a set of blocks (see Figure~\ref{fig:open-pattern}) with the requirement that the steps in different blocks are assigned different users.%
 (It is important to note that in such figures, purely for illustration, we are assuming that the rows/columns are permuted to reveal any block-diagonal structure; there is no implication that the steps are processed in a fixed order.)
 We will say that such an (incomplete) pattern is `open' as the relation between the steps in the pattern and those not in the pattern is left as undetermined; for a step not in the pattern, the values of $M$ are not yet fixed.
 The openness of the pattern corresponds to the open nature of the assignments within PBT\@.
 %PBT considers steps one at a time, and its branching corresponds to picking a block in the pattern which to extend (or creating a new block).
 Figure~\ref{fig:matrix-pbt} illustrates the branching within PBT in terms of the options for extending the value assignments to the $M$ matrix.

\newcommand{\cellfont}[0]{\small\bfseries}

\tikzset{
	left step label/.style={anchor=east, font={\small}, inner sep=2pt},
	top step label/.style={anchor=south, font={\small}, inner sep=2pt},
	cell base/.style={minimum width=1cm, minimum height=1cm, draw=none, transform shape},
	blk1/.style={cell base, fill=green!50!white, label={[font=\cellfont]center:1}},
	blk2/.style={cell base, fill=red!50!white, label={[font=\cellfont]center:1}},
	zero/.style={cell base, fill=gray!30!white, label={[font=\cellfont]center:0}},
	none/.style={cell base, fill=white, label={[font=\cellfont]center:?}},
	diag/.style={cell base, fill=white, label={[font=\cellfont]center:1}},
	matrix grid style/.style={gray, ultra thin, opacity=0.5},
	auth/.style={cell base, fill=white},
	unau/.style={cell base, pattern color=gray!50!white, pattern=north east lines},
}

\newcommand{\decisionmatrix}[1]{
	\xdef\y{0}
	\foreach \row in #1
	{
		\xdef\x{0}
		\foreach \cell in \row
		{
			%\draw[\cell] (\x, \y) rectangle (\x + 1, \y - 1);
			\node[\cell] at (\x + 0.5, \y + 0.5) {};
			\pgfmathparse{\x+1}
			\xdef\x{\pgfmathresult}
		}
		\pgfmathparse{\y+1}
		\xdef\y{\pgfmathresult}
	}

	\draw[xstep=1.0,ystep=-1,matrix grid style] (0,\y) grid (\x,0);
}

\begin{figure}[htb]
\begin{subfigure}{0.45\textwidth}
	\centering
	\begin{tikzpicture}[scale=0.5, y=-1cm]
		\decisionmatrix{{%
			{blk1, blk1, zero, zero, zero, none, none, none},%
			{blk1, blk1, zero, zero, zero, none, none, none},%
			{zero, zero, blk2, blk2, blk2, none, none, none},%
			{zero, zero, blk2, blk2, blk2, none, none, none},%
			{zero, zero, blk2, blk2, blk2, none, none, none},%
			{none, none, none, none, none, diag, none, none},%
			{none, none, none, none, none, none, diag, none},%
			{none, none, none, none, none, none, none, diag}%
		}};
		
%		\begin{scope}[xshift=-14cm]
%		\decisionmatrix{{%
%			{none, none, unau, unau, none, unau, none, unau, none, none, unau, unau},%
%			{none, unau, none, unau, none, none, unau, unau, none, unau, none, unau},%
%			{none, none, none, unau, none, unau, unau, none, unau, none, unau, unau},%
%			{unau, none, unau, none, unau, none, none, none, none, unau, unau, none},%
%			{none, none, unau, none, none, none, unau, none, unau, unau, none, none},%
%			{unau, none, none, unau, none, none, unau, none, unau, none, unau, none},%
%			{none, unau, unau, unau, unau, none, unau, unau, none, unau, none, none},%
%			{unau, none, unau, unau, none, unau, none, none, unau, none, unau, none}%
%		}};		

%		\foreach \i in {1,...,8}
%		{
%			\node[left step label] at (0, \i - 0.5) {$s_{\i}$};
%		}

%		\foreach \i in {1,3,...,12}
%		{
%			\node[top step label] at (\i - 0.5, 0) {$u_{\i}$};
%		}
%		\end{scope}

		\foreach \i in {1,...,8}
		{
			\node[left step label] at (0, \i - 0.5) {$s_{\i}$};
			\node[top step label] at (\i - 0.5, 0) {$s_{\i}$};
		}
	\end{tikzpicture}
	\caption{
    	Open pattern.
        The pattern has two blocks, each of which can be extended with new steps.
        New steps can also be assigned to a new block.
        %None of the $x$ variables is fixed; however, PBT guarantees that there exists at least one feasible assignment of $x$ for any pattern it produces.
    }
    \label{fig:open-pattern}
\end{subfigure}
\hfill
\begin{subfigure}{0.45\textwidth}
	\centering
	\begin{tikzpicture}[scale=0.5, y=-1cm]
		\decisionmatrix{{%
			{blk1, blk1, zero, zero, zero, zero, zero, zero},%
			{blk1, blk1, zero, zero, zero, zero, zero, zero},%
			{zero, zero, blk2, blk2, blk2, zero, zero, zero},%
			{zero, zero, blk2, blk2, blk2, zero, zero, zero},%
			{zero, zero, blk2, blk2, blk2, zero, zero, zero},%
			{zero, zero, zero, zero, zero, diag, none, none},%
			{zero, zero, zero, zero, zero, none, diag, none},%
			{zero, zero, zero, zero, zero, none, none, diag}%
		}};

%		\begin{scope}[xshift=-14cm]
%		\decisionmatrix{{%
%			{blk1, zero, unau, unau, zero, unau, zero, unau, zero, zero, unau, unau},%
%			{blk1, unau, zero, unau, zero, zero, unau, unau, zero, unau, zero, unau},%
%			{zero, blk2, zero, unau, zero, unau, unau, zero, unau, zero, unau, unau},%
%			{unau, blk2, unau, zero, unau, zero, zero, zero, zero, unau, unau, zero},%
%			{zero, blk2, unau, zero, zero, zero, unau, zero, unau, unau, zero, zero},%
%			{unau, zero, none, unau, none, none, unau, none, unau, none, unau, none},%
%			{zero, unau, unau, unau, unau, none, unau, unau, none, unau, none, none},%
%			{unau, zero, unau, unau, none, unau, none, none, unau, none, unau, none}%
%		}};		

%		\foreach \i in {1,...,8}
%		{
%			\node[left step label] at (0, \i - 0.5) {$s_{\i}$};
%		}

%		\foreach \i in {1,3,...,12}
%		{
%			\node[top step label] at (\i - 0.5, 0) {$u_{\i}$};
%		}
%		\end{scope}

		\foreach \i in {1,...,8}
		{
			\node[left step label] at (0, \i - 0.5) {$s_{\i}$};
			\node[top step label] at (\i - 0.5, 0) {$s_{\i}$};
		}
	\end{tikzpicture}
	\caption{
    	Closed pattern.
        Both blocks are closed, i.e.\ the algorithm cannot add any other steps to these two blocks; it can only create new blocks.
        %Two users are already assigned to the two blocks, and this assignment is final.
        %No other steps can be assigned to these users.
    }
    \label{fig:closed-pattern}
\end{subfigure}

\caption{
	Open vs.\ closed pattern in terms of $M$ matrix.
	%In each figure, to the left is the $x$ matrix; to the right is the $M$ matrix.
	Full $M$ matrix is shown for clarity although it is symmetric by definition.
    %Hatch shading in the $x$ matrix shows unauthorised steps.
    Question marks show undecided variables.
    Grey cells (with zeros) are variables fixed at 0; green and red cells (with ones) are variables fixed at 1.
    }
	\label{fig:open-closed-patterns}
\end{figure}

\tikzset{
	left step lable/.style={anchor=east, font={\tiny}},
	top step lable/.style={anchor=south, font={\tiny}},
	cell base/.style={minimum width=1cm, minimum height=1cm, draw=none, transform shape},
	blk1/.style={cell base, fill=green!50!white},
	blk2/.style={cell base, fill=red!50!white},
	blk3/.style={cell base, fill=blue!50!white},
	zero/.style={cell base, fill=gray!30!white},
	none/.style={cell base, fill=white},
	diag/.style={cell base, fill=white},
	matrix grid style/.style={gray, ultra thin, opacity=0.5},
	decision variables/.style={black, solid, draw}
}

\begin{figure}[htb]
\centering
\begin{tikzpicture}[y=-1cm]
	\newcommand{\decisionsframe}{\path[decision variables] (0,5) -- (5,5) -- (5,0) -- (6,0) -- (6,6) -- (0,6) -- cycle;};
	\node (P) at (0, -3.5) {
		\begin{tikzpicture}[scale=0.3]
		\decisionmatrix{{%
			{blk1, blk1, zero, zero, zero, none, none, none},%
			{blk1, blk1, zero, zero, zero, none, none, none},%
			{zero, zero, blk2, blk2, blk2, none, none, none},%
			{zero, zero, blk2, blk2, blk2, none, none, none},%
			{zero, zero, blk2, blk2, blk2, none, none, none},%
			{none, none, none, none, none, diag, none, none},%
			{none, none, none, none, none, none, diag, none},%
			{none, none, none, none, none, none, none, diag}%
		}};		
		\end{tikzpicture}
		};

	\node (C1) at (-3, 0) {
		\begin{tikzpicture}[scale=0.3]
		\decisionmatrix{{%
			{blk1, blk1, zero, zero, zero, blk1, none, none},%
			{blk1, blk1, zero, zero, zero, blk1, none, none},%
			{zero, zero, blk2, blk2, blk2, zero, none, none},%
			{zero, zero, blk2, blk2, blk2, zero, none, none},%
			{zero, zero, blk2, blk2, blk2, zero, none, none},%
			{blk1, blk1, zero, zero, zero, blk1, none, none},%
			{none, none, none, none, none, none, diag, none},%
			{none, none, none, none, none, none, none, diag}%
		}};
		\decisionsframe
		\end{tikzpicture}
		};

	\node (C2) at (0, 0) {
		\begin{tikzpicture}[scale=0.3]
		\decisionmatrix{{%
			{blk1, blk1, zero, zero, zero, zero, none, none},%
			{blk1, blk1, zero, zero, zero, zero, none, none},%
			{zero, zero, blk2, blk2, blk2, blk2, none, none},%
			{zero, zero, blk2, blk2, blk2, blk2, none, none},%
			{zero, zero, blk2, blk2, blk2, blk2, none, none},%
			{zero, zero, blk2, blk2, blk2, blk2, none, none},%
			{none, none, none, none, none, none, diag, none},%
			{none, none, none, none, none, none, none, diag}%
		}};
		\decisionsframe
		\end{tikzpicture}
		};

	\node (C3) at (3, 0) {
		\begin{tikzpicture}[scale=0.3]
		\decisionmatrix{{%
			{blk1, blk1, zero, zero, zero, zero, none, none},%
			{blk1, blk1, zero, zero, zero, zero, none, none},%
			{zero, zero, blk2, blk2, blk2, zero, none, none},%
			{zero, zero, blk2, blk2, blk2, zero, none, none},%
			{zero, zero, blk2, blk2, blk2, zero, none, none},%
			{zero, zero, zero, zero, zero, blk3, none, none},%
			{none, none, none, none, none, none, diag, none},%
			{none, none, none, none, none, none, none, diag}%
		}};
		\decisionsframe
		\end{tikzpicture}
		};
	
	\path[very thick, ->]
		(P) edge (C1)
		(P) edge (C2)
		(P) edge (C3);
\end{tikzpicture}
\caption{
 	PBT branching.
 	The parent pattern contains two blocks (green of size 2 and red of size 3).
 	PBT extends the parent pattern by assigning a new step in three different ways: (left) extend the first block with the new step; (centre) extend the second block with the new step; (right) create a new block consisting of the new step only.
    A black frame encloses the the variables fixed in each of the branches.
    Note that the branching step can be chosen arbitrarily; PBT uses a heuristic to select the branching step.
	}
    \label{fig:matrix-pbt}
\end{figure}
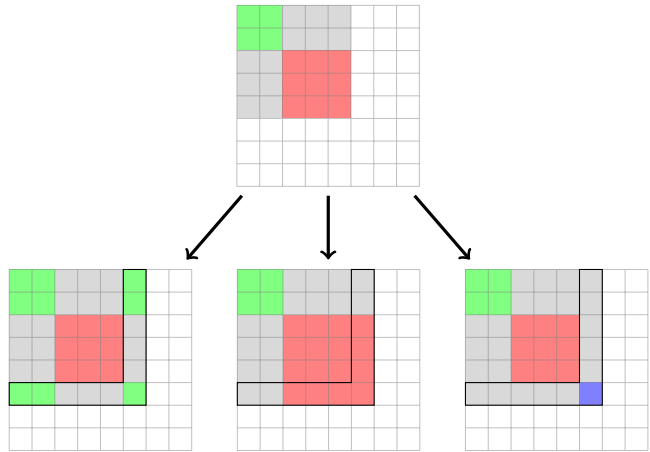 

 PUI, the previously state-of-the-art FPT algorithm for the WSP with UI constraints~\cite{JOCO2014}, implements a different branching strategy.
 It iterates over the set of users $U$ gradually building a set $\mathscr{P}$ of valid patterns.
 At an iteration corresponding to some user $u_i \in U$, PUI attempts to extend each $\pattern \in \mathscr{P}$ with exactly one new block $B$, trying each non-empty $B \subseteq A(u_i) \setminus S(\pattern)$.
 PUI will never attempt to extend an existing block in a pattern, and in this sense it uses `closed' patterns, see Figure~\ref{fig:closed-pattern}.
 While closed patterns support richer propagation, they also reduce the flexibility of search.
 Indeed, a single extension of a closed pattern inevitably fixes many more $M$ variables than that of an open pattern, reducing the ability of the algorithm to focus on the most constrained parts of the problem.
 In general, the smaller the decisions are, the more flexibility the search has in ordering them and, hence, prioritising the most important ones.
 Thus, we suggest that open patterns are preferable to closed patterns.

Storing the set $\mathscr{P}$ of valid patterns also makes the space complexity of PUI exponential in $k$.
Consider an instance with $n = k$ and no constraints.
Let $s \in S$ be some step, and let all the users in this instance be authorised to all the steps except for $s$, i.e.\ $A(u) = S \setminus \{ s \}$ for every $u \in U$.
Clearly, such an instance is unsatisfiable.
To establish this, PUI will generate and store all incomplete valid patterns.
Observe that any partition of $S' \subseteq (S \setminus \{ s \})$ is a valid pattern in this instance.
To count them, consider the set of all partitions of $S$.
There are exactly $\mathcal{B}_k$ of such partitions.
Now, in each of these partitions, remove the block that includes $s$.
This will give us a set of all partial partitions of $S$ that do not cover $s$.
This set exactly corresponds to the set of all the incomplete valid patterns in our problem instance.
Hence, to solve this instance, PUI will populate $\mathscr{P}$ with exactly $\mathcal{B}_k$ patterns.
Hence, the space complexity of PUI is $\Omega^*(B_k)$, where $\Omega^*$ hides polynomial factors similarly to the $O^*$ notation.

 Another aspect in which PBT is different to PUI is that PBT implements delayed assignment of users, branching on user-independent $M$ variables.
 Conversely, PUI branches on the $x$ variables, i.e.\ fixes the assignment of users while branching, achieving the FPT running time by merging equivalent branches.
 We argue that delayed user assignment would usually be more effective.
 As we show in \appendixA, a significant portion of patterns are authorised but usually only a few patterns are eligible; hence branching on the $M$ variables is likely to produce smaller search trees.

 It is possible to implement an algorithm with closed patterns and delayed user assignment.
 Our attempt to implement such a `closed-pattern-based PBT', however, resulted in an algorithm significantly slower than PBT (although faster than PUI).
 
 It is also possible to implement an open-pattern style search with immediate user assignment.
 Indeed, a general purpose PB solver is likely to exploit this strategy when solving UDPB formulation.
 However, assignment of some (but not all) $x_{su}$ for a $u \in U$ would mean that subsequent pruning of the branch does not guarantee that the corresponding open pattern is invalid; indeed, a different assignment of $x_{su}$, $s \in S$, may produce a valid plan.
 Hence, patterns could not be used to merge branches of the search, and the algorithm would not be FPT\@.
 For FPT algorithms with immediate user assignment, it is vital, like in PUI, to try all the authorised combinations of $x_{su}$ for a given $u \in U$ before proceeding to the next $u$; by following this strict sequence, the algorithm guarantees that it finds all the eligible patterns authorisable by processed users.
 For the same reason, PUI could not be implemented as a DFS algorithm.
 
 An important observation is that FPT algorithms with immediate user assignments (i.e.\ PUI) are forced to order the branching variables by user whereas algorithms with delayed user assignments and open patterns (i.e.\ PBT) order the branching variables by steps.
 In a problem with a relatively small number of steps and a large number of users, it is more likely that the users are relatively uniform compared to the steps, and hence the search is more sensitive to the order of steps than to the order of users.
 In other words, branching heuristics in PBT-like algorithms are expected to be more effective compared to branching heuristics in PUI-like algorithms.
 
 Now consider the combination of the PBPB encoding with a DPLL-based~\cite{DavisLL1962} PB solver such as SAT4J\@.
 Internally, a PB solver on PBPB will need to be making branching decisions. 
 This is generally done so as to prefer branching on variables that propagate to entail values for other variables, and given the central nature of the $M$ variables, it seems reasonable they would be favoured as branch variables. 
 As pointed out above, a complete assignment to the variables $M_{s_1, s_2}$, $s_1, s_2 \in S$, and satisfying the constraints~(\ref{eq:mx-symmetric})--(\ref{eq:mx-link-m2}), uniquely defines a complete pattern. 
 A PB solver will be handling partial assignments to the $M$ variables, but it is still reasonable to ask if they are structured like open or closed patterns. 
 To address this, consider the effects of the transitivity constraints (\ref{eq:mx-link-m1}) and (\ref{eq:mx-link-m2}).
 If two $M$ variables sharing a step, e.g.\ $M_{12}$ and $M_{23}$, are set to 1, then (\ref{eq:mx-link-m2}) immediately forces a propagation, $M_{13} = 1$, and similarly for (\ref{eq:mx-link-m1}), so if block-diagonals of 1's happen to overlap then will form into a larger block of 1's.
 Hence there is a tendency to complete the blocks in the $M$ matrix, but there will be no reason to close them.
 This will tend to drive the partial $M$-assignments to have a structure close to open patterns.
 We hence expect that a standard (DPLL-based) PB solver could work on the PBPB formulation in a similar fashion of using open patterns and then extending them. 
 So we expect that the behaviour of the PB solver with PBPB will be more similar to PBT than to PUI; we will see evidence for this in Section~\ref{sec:experiments}.% -- a result that initially surprised us.
 
 We also note here, that, unlike the bespoke PBT and PUI algorithms, PBPB solvers have the flexibility to arbitrarily alternate between branching on $M$'s and $x$'s.
 When user authorisations are tight, this may lead to superior strategies and hence is a potential strength of the general-purpose solver approach.
 
%%%%%%%%%%%%%%%%%%%%%%%%%%%%%%%%%%%%%%%%%%%%%%%%%%%%%%%%%%%%%%%%%%%%%%%%%

\subsection{Properties of the New Pseudo-Boolean Formulation PBPB}   
\label{sec:newPB-properties}

 In this section, we show that the PBPB formulation can also potentially admit FPT running time and polynomial space complexity. 
 The discussion breaks into the upper level search on the $M$-variables for eligible patterns, and the subsequent lower level matching problems arising from the $x$ variables in the context of a pattern.
 
 For the upper level, there are $O(k^2)$ of the $M$ variables, and so a tree search in PB would have the potential to fully instantiate these before handling the user assignments via the $x$ variables (which is not an unreasonable assumption, see Section~\ref{sec:branching-strategies}), using a tree of worst-case size $2^{O(k^2)}$.
 This is FPT, and so whether or not PBPB has a potential to be FPT as a whole depends on the complexity of the user assignments once a complete pattern is reached.

 When all the $M$'s are instantiated, the PBPB formulation (\ref{eq:mx-symmetric})--(\ref{eq:mx-x-define}) reduces to the following:
\begin{align}
\label{eq:sum-x-one}
& \sum_{u \in U} x_{s, u} = 1 
	&& \forall s \in S, \\
& x_{s, u} = 0
	&& \forall s \in S \text{ and } \forall u \in U \setminus A^{-1}(s), \\
\label{eq:substituted-same-user}
& x_{s_1, u} = x_{s_2, u}
	&& \forall s_1 \neq s_2 \in B, B \in \pattern \text{ and } \forall u \in U, \\
\label{eq:substituted-diff-user}
& x_{s_1, u} + x_{s_2, u} \le 1
	&& \forall B_1 \neq B_2 \in \pattern,\ \forall s_1 \in B_1,\ \forall s_2 \in B_2 \text{ and } \forall u \in U, \\
& x_{s, u} \in \{ 0, 1 \} 
	&& \forall s \in S \text{ and } \forall u \in U.
\label{eq:PHP-x-one}
\end{align}

 This is a bipartite matching problem but with blocks of steps being assigned to a user.
 However, because of (\ref{eq:substituted-same-user}), when any one step in a block is assigned (i.e.\ some $x_{s,u} = 1$) then all the other steps in the block are also forced, by propagation, to the same user.

%\DK{Andrew, could you please check the below couple of sentences?}
 While we know algorithms to solve the matching problem in polynomial time, a general-purpose PB solver in a mode that essentially only uses SAT-style resolution cannot solve it efficiently.
 The Pigeon Hole Problem, seeking to assign $n$ entities to $n - 1$ holes and which is a special case of the Bipartite Matching Problem, is known to be exponentially hard for DPLL solvers~\cite{HAKEN1985:intractability} that only use SAT representations; but polynomial size proofs are possible for PB, and so this explains why we encode it using PB.
 The following proposition shows the formulation (\ref{eq:sum-x-one})--(\ref{eq:PHP-x-one}) can be solved in FPT time by a general purpose PB solver, using standard tree-search methods (branching and propagation), but not introducing new variables during the search process. It is important to make this assumption because search or proof methods that are allowed to introduce new variables have the potential to be a lot more powerful, e.g.~\cite{Razborov2002:Complexity-PHP}, but such methods are currently too difficult to control in practice.
% Its proof is given in \ref{ap:proofs}.

\begin{proposition} 
	\label{thm:PB-is-FPT}
	The PB formulation (\ref{eq:sum-x-one})--(\ref{eq:PHP-x-one}) can be solved by tree search and propagation (without the introduction of new variables), in polynomial space, and in time exponential in $|\pattern|$ only.\footnote{The proposition is closely related to known methods in kernelisation \shortcite{Gutin2015}; however, due to the lack of space we do not want to pursue that here.}
\end{proposition}
%Proof of Proposition~\ref{thm:PB-is-FPT}.
\begin{proof}
 The PB formulation (\ref{eq:sum-x-one})--(\ref{eq:PHP-x-one}) corresponds to the standard bipartite matching problem on a graph with the vertices of one partition consisting of the blocks of $\pattern$ and the other partition vertices are the users $U$.
 Observe that if the degree (number of authorised users) of a block $B \in \pattern$ is greater than the number of blocks in $|\pattern|$, then no set of choices for the other blocks can remove all the options for that block. 
 Hence, all vertices $B \in \pattern$ of degree $|\pattern|$ or above can be delayed until last in the search tree: if the search does reach them, then they can be given arbitrary values and so will never lead to backtracking.
 (An example occurs in Figure~\ref{fig:assignment-graph-sat} in which the block $\{ s_3 \}$ has 3 authorised users; so its assignment can be delayed until after the other 2 blocks.)
 Variables within a block are all constrained to be equal, hence, eventually one of them will be picked as the branch variable; at this time the propagation will give values to all the others.
 The other members of a block hence will no longer be candidate branch variables, and they will not contribute to the size of the search tree.
 Hence in the backtracking portion of the branching, branching factor of the search will be limited by $|\pattern| - 1$, and the depth of the search by $|\pattern|$.
 Hence the search tree size is $O((|\pattern| - 1)^{|\pattern|})$, and the depth is polynomial.
%\qed
\end{proof}

This proposition is sufficient to show that a PB solver based on standard branch-and-propagate methods has the potential to solve the PBPB formulation in FPT time.
 However, Proposition~\ref{thm:PB-is-FPT} effectively shows that an unbalanced bipartite matching problem (with parts of size $O(k)$ and $O(n)$, respectively) can be solved by a PB solver in time polynomial in $n$ and exponential in $k$, whereas we know that the Hungarian method is polynomial in both $n$ and $k$.
 Although we have not observed difficult matching problems in our experiments with WSP algorithms, it is still natural to discuss the worst case, and consider what are the potential limitations of the PB approach.
 For this we will switch to a ``proof theory'' perspective and ask about the sizes of the proofs of unsatisfiability available within the PB representation (note that a proof of satisfiability of the matching problem is trivial in the sense that it is just the verification of a given witness).

\begin{proposition}
 \label{th:matching-polysize}
 When the PB formulation (\ref{eq:sum-x-one})--(\ref{eq:PHP-x-one}) is unsatisfiable, then there is a PB proof of that unsatisfiability, without introducing new variables, and that is polynomial in both $|\pattern|$ and $n$.
\end{proposition}
\begin{proof}
 Observe that we can take an arbitrary representative, a step, from within each block of the pattern $\pattern$ and use propagation through (\ref{eq:substituted-same-user}) to limit the users permitted for the representative. 
 Hence the problem (\ref{eq:sum-x-one})--(\ref{eq:PHP-x-one}) is precisely the matching of the selected representative of each block to a permitted user. 
 The Proposition~\ref{th:matching-polysize} follows from the Hall's marriage theorem \cite{Cameron1994:Combinatorics-book}; a matching problem on a bipartite graph $G = (L \cup R, E)$ has a complete matching of the vertices of the partition $L$, if and only if it is true that for all subsets $L'$ of $L$, there are at least $|L'|$ elements in $R$ that may match with some vertex in $L'$.
 In WSP language this basically means that the matching problem is unsatisfiable if and only if there is some subset $\mathcal{B}$ of blocks, for which the corresponding set of candidate users is smaller than $|\mathcal{B}|$; for example, see Figure~\ref{fig:assignment-graph-unsat}.
 Such a subset is a constrained form of the pigeon-hole problem (PHP), stating that $|\mathcal{B}|$ blocks cannot be assigned to fewer than $|\mathcal{B}|$ users, and restricting equations (\ref{eq:sum-x-one})--(\ref{eq:PHP-x-one}) to such a subset from the marriage theorem leads to a PB encoding of the PHP\@. 
 (There are also extra constraints to remove unauthorised assignments within the PHP, but these are not needed, as the counting already suffices.)
 Since the PHP is known to have a polynomial size PB proof without the use of new variables, see e.g.~\shortcite{CookEtal1987:complexity-cutting-plane-proofs,DixonEtal:2004:GenBoolSatI}, it follows there is also a PB proof of (\ref{eq:sum-x-one})--(\ref{eq:PHP-x-one}).
%\qed
\end{proof}

 Note that Propositions~\ref{thm:PB-is-FPT} and~\ref{th:matching-polysize} are quite different in that the first one is discussing the process of a solution, whereas the second one is about a witness to unsatisfiability but not the time to find it.
 Nevertheless we conclude that a PB solver might, at least, be expected to use tree search to solve the assignment problem in time exponential in $k$, but also have the potential to be able to solve it in time polynomial in $k$.
 
 We further note that our separate experiments confirmed that at large $n$ the matching problem occurring at the lower level of the algorithm are typically very easy as there is little conflict between users, and simple propagation is usually sufficient to solve the problem.

%%%%%%%%%%%%%%%%%%%%%%%%%%%%%%%%%%%%%%%%%%%%%%%%%%%%%%%%%%%%%%
%%%%%%%%%%%%%%%%%%% PART TWO %%%%%%%%%%%%%%%%%%%%%%%%%%%%%%%%%
%%%%%%%%%%%%%%%%%%%%%%%%%%%%%%%%%%%%%%%%%%%%%%%%%%%%%%%%%%%%%%
\section{Instance Generator and Phase Transitions}  
\label{sec:testbed}

 It is preferred to use real-world instances in computational studies. However, there are only very few real-world instances of WSP available publicly, and those are of size too small to be of interest in our computational experiments, see e.g.~\cite{Bertolissi2018}. 
 We have considered the instances appearing in practice from~\cite{Bertolissi2018,SantosRCP15,SantosRCP17}, namely, TRW (Trip Request Workflow), ITIL (IT Financial Reporting), and ISO (Budgeting for Quality Management) having 5, 7, 9 steps (tasks) and 5, 2, 3 not-equals constraints, respectively. 
 More details and description of user-step authorisations for 3 satisfiable and 3 unsatisfiable versions of these instances can be found in Section~5 of~\cite{SantosRCP17}. 
 Our new solver PBT correctly solved each of the six instances in less than $0.001$ sec. 
 If some businesses do not use more and more complex constraints because of computational complexity issues, PBT provides a good opportunity to consider larger and more involved workflow scenarios.

Therefore, due to the difficulty of acquiring real-world instances with an appropriate range of sizes, to support extensive studies of the scaling of the runtimes, similarly to other authors~\cite{JOCO2014,Roy2015,WaLi10}, we use the synthetic instance generator described in~\cite{JOCO2014} and available for downloading~\cite{SourceCodes}.
In this section, we first present the generator of the WSP instances.
We then empirically study the probability of satisfiability of the instances as we vary the generator parameters. 
We give evidence for PT, between the satisfiable and unsatisfiable regions.
The point is that the resulting instances from the PT region can be expected to be a good test of the effectiveness of solution algorithms.

\subsection{The Instance Generator} 
\label{sec:generator}
\label{sec:instance-generator}

 Three families of UI constraints are used: \emph{not-equals} (also called \emph{separation-of-duty}), \emph{at-most-$r$} and \emph{at-least-$r$} constraints.
 A not-equals constraint with scope $\set{s, t}$ is satisfied by a complete plan $\pi$ if and only if $\pi(s) \neq \pi(t)$.
 An at-most-$r$ constraint $c$ with scope $T_c$ is satisfied if and only if $|\pi(T_c)| \le r$.
 Similarly, an at-least-$r$ constraint $c$ with scope $T_c$ is satisfied if and only if $|\pi(T_c)| \ge r$.
 We do not explicitly consider the widely used binding-of-duty constraints, that require two steps to be assigned to one user, as those can be trivially eliminated during preprocessing.
 While the binding-of-duty and separation-of-duty constraints provide the basic modelling capabilities, the at-most-$r$ and at-least-$r$ constraints impose more general ``confidentiality'' and ``diversity'' requirements on the workflow, which can be important in some business environments.
 %Following~\cite{JOCO2014,KaGaGu,Roy2015}, 
 We decided to focus this study on at-least-3 and at-most-3 constraints with a scope of 5, $|T_c|=5$ for the following reason. 
 Cohen et al.~\cite{JOCO2014} performed several computational experiments with the at-least-$r$ and at-most-$r$ constraints for various values of $r$ and various sizes of scope $T$. 
 The constraints at-least-3 and at-most-3 with scope of size 5 were selected in~\cite{JOCO2014} as they appear to be in the practical range of the parameters $r$ and $|T|$ and the values $r=3$ and $|T|=5$ often lead to computationally challenging WSP instances. 
 For more details, see Section~5.2 of~\cite{JOCO2014}.

The specific stochastic WSP Instance Generator takes as input four parameters,
\begin{enumerate}
  \item $k$, the number of steps;
  \item $n$, the number of users;
  \item $e$, the number of not-equals constraints;
  \item $\gamma$, the number of at-most-3 and also the number of at-least-3 constraints (all with scope 5).
\end{enumerate}
The generator, which we denote as $\mathit{WIG}(k,n,e,\gamma)$, is stochastic, but as usual can be made deterministic by also specifying a value for the random generator seed.
 For each user $u \in U$, it generates $A(u)$ such that the size of $A(u)$ is first selected uniformly from $\{ 1, 2, \ldots, \lfloor 0.5 k \rfloor \}$ at random and then the set $A(S)$ itself is selected randomly and uniformly from $S$ with no repetitions.
 This results in each step having $n/4$ authorised (random) users on average. 
 The generator also produces $e$ distinct not-equals, $\gamma$ at-most-3 constraints, and $\gamma$ at-least-3 constraints, all uniformly at random.
 
 Note that in general, WSP instances could have some symmetries, e.g.\ different steps, or different users are interchangeable, potentially, making search processes more inefficient and needing the usage of symmetry breaking methods (e.g.\ see~\shortcite{CrawfordEtal1996:symmetry-breaking-predicates}).
 However, our benchmark instances have negligible amounts of such symmetry (both by theory and also by direct evaluation) and so we do not consider it here.

 The PBT algorithm, conversion routines, test instances with solutions and the test instance generator are available for downloading~\cite{SourceCodes}.

%%%%%%%%%%%%%%%%%%%%%%%%%%%
\subsection{Thresholds from the Instance Generator} 
\label{sec:setting-density}

 In this section, we focus on experiments to study the dependency of instance properties on the parameters of the instance generator, and show that the WSP instances we use exhibit the classic properties expected of PTs.

 Figure~\ref{fig:vary-params} shows an example, at $k=40$ and $n=400$, of the running time of PBT and the percentage of unsatisfiable (unsat) instances as they change with variation of parameters $e$ (Figure~\ref{fig:vary-d}) and $\gamma$ (Figure~\ref{fig:vary-alpha}).
 As one could expect, the number of unsat instances grows with the number of constraints.
 It can also be observed that the hardest instances are those near the 50\% level of the unsat curve.
Following standard arguments, under-constrained instances have a lot of valid plans which makes them relatively easy.
 Unsatisfiability of the oversubscribed instances can be proved relatively quickly due to heavy pruning of the branches.
 However, instances around the 50\% unsat level are likely to have none to few valid plans making it hard to find them or prove unsatisfiability.
Note that one can argue that real-world instances in many cases are likely to be in the region of 50\% unsatisfiable: Organisations are likely to be constraining their workflows up to the point when the workflows become unsatisfiable, or start with unsatisfiable workflows and gradually relax the constraints until obtaining satisfiable problems.
 
\begin{figure}[htb]
	\begin{subfigure}{0.48\textwidth}
		\begin{tikzpicture}
		\begin{axis}[
			compat=newest,
			width=0.9\textwidth,
			height=6cm,
			legend pos=north west,
			hide x axis,
			ylabel={Running time, sec},
			title={},
			xmin=0,
			xmax=200,
			ymin=0
			%legend cell align=left,
		]
        \addplot+[name path=A, blue!50!black, thin, solid, mark=none] table[
			x expr={\thisrow{density} * 40 * 39 / 2},
			y=runtime-25,
            each nth point=2
		] {vary-d.dat};
        \addplot+[name path=B, blue!50!black, thin, solid, mark=none] table[
			x expr={\thisrow{density} * 40 * 39 / 2},
			y=runtime-75,
            each nth point=2
		] {vary-d.dat};

        \addplot[blue!20] fill between[of=A and B];

        \addplot+[name path=C, blue!50!black, thin, solid, mark=none] table[
			x expr={\thisrow{density} * 40 * 39 / 2},
			y=runtime-35,
            each nth point=2
		] {vary-d.dat};
        \addplot+[name path=D, blue!50!black, thin, solid, mark=none] table[
			x expr={\thisrow{density} * 40 * 39 / 2},
			y=runtime-65,
            each nth point=2
		] {vary-d.dat};

        \addplot[blue, opacity=0.3] fill between[of=C and D];

        \addplot[blue, ultra thick] table[
			x expr={\thisrow{density} * 40 * 39 / 2},
			y=runtime-50,
            each nth point=2
		] {vary-d.dat};

		%\addplot[thick, blue] table[
		%	x expr={\thisrow{density} * 100},
		%	y=runtime,
        %    each nth point=2
		%] {vary-d.dat};
		\end{axis}
		
		\begin{axis}[
			compat=newest,
			width=0.9\textwidth,
			height=6cm,
			legend pos=north west,
			xlabel={Number of not-equals constraints $e$},
			ylabel={Unsat., \%},
			axis y line*=right,
			%legend cell align=left,
			grid=major,
			xmin=0,
			xmax=200,
			ymin=0,
			ymax=100
		]
		\addplot[ultra thick, red, dashed] table[
			x expr={\thisrow{density} * 40 * 39 / 2},
			y expr={\thisrow{unsat} * 100},
		] {vary-d.dat};
		\end{axis}
		\end{tikzpicture}
		\caption{The number $\gamma$ of at-most and at-least constraints is fixed at $\gamma = k$, and the number $e$ of not-equals constraints is varied.}
		\label{fig:vary-d}
	\end{subfigure}
	\hspace{0.04\textwidth}
	\begin{subfigure}{0.48\textwidth}
		\begin{tikzpicture}
		\begin{axis}[
			compat=newest,
			width=0.9\textwidth,
			height=6cm,
			legend pos=north west,
			hide x axis,
			ylabel={Running time, sec},
			xmin=0,
			xmax=79,
			ymin=0,
			%legend cell align=left,
		]
		\addplot+[name path=A, blue!50!black, thin, solid, mark=none] table[
			x=c,
			y=runtime-25,
		] {vary-alpha.dat};
        \addplot+[name path=B, blue!50!black, thin, solid, mark=none] table[
			x=c,
			y=runtime-75,
		] {vary-alpha.dat};

        \addplot[blue!20] fill between[of=A and B];
        
   		\addplot+[name path=C, blue!50!black, thin, solid, mark=none] table[
			x=c,
			y=runtime-35,
		] {vary-alpha.dat};
        \addplot+[name path=D, blue!50!black, thin, solid, mark=none] table[
			x=c,
			y=runtime-65,
		] {vary-alpha.dat};

        \addplot[blue, opacity=0.3] fill between[of=C and D];

        \addplot[blue, ultra thick] table[
			x=c,
			y=runtime-50,
		] {vary-alpha.dat};

		%\addplot[thick, blue] table[
		%	x=c,
		%	y=runtime-50,
		%] {vary-alpha.dat};
		\end{axis}
		
		\begin{axis}[
			compat=newest,
			width=0.9\textwidth,
			height=6cm,
			legend pos=north west,
			xlabel={Num.\ of at-most and at-least constraints $\gamma$},
			ylabel={Unsat., \%},
			axis y line*=right,
			%legend cell align=left,
			grid=major,
			xmin=0,
			xmax=79,
			ymin=0,
			ymax=100
		]
		\addplot[ultra thick, red, dashed] table[
			x=c,
			y expr={\thisrow{unsat} * 100},
		] {vary-alpha.dat};
		\end{axis}
		\end{tikzpicture} 
		\caption{The number $e$ of not-equals constraints is fixed at $e = 78$ (corresponding to 10\% of all available choices), and the value of $\gamma$ is varied.}
		\label{fig:vary-alpha}
	\end{subfigure}
 \caption{
	At $(k,n)=(40,400)$, the running times of PBT (blue) and percentage of unsatisfiable instances (red) for various values of the instance generator parameters.
    The blue shades show the [35\%--65\%] (deep blue) and [25\%--75\%] (lighter blue) percentiles of the runtimes.
 }
 \label{fig:vary-params}
\end{figure}
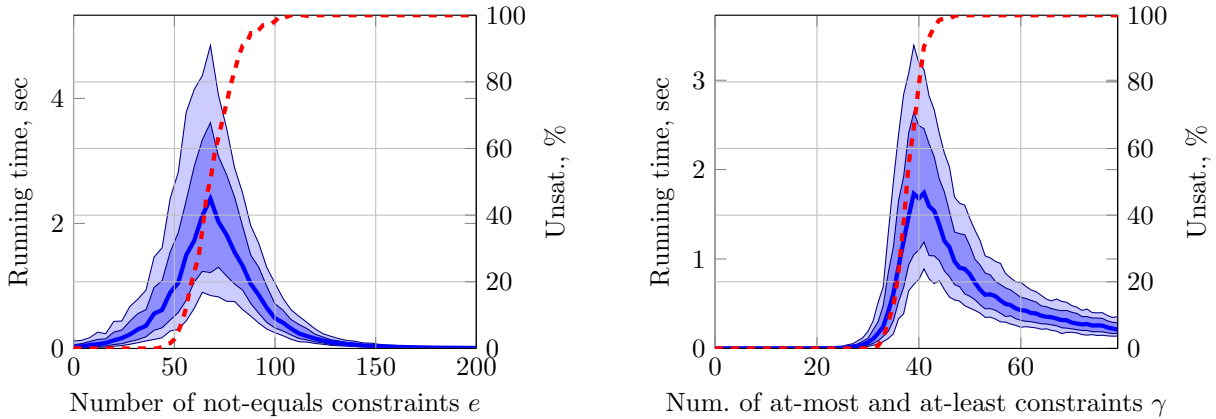

In the comparison of algorithms we use `critical' instances with $\gamma = k$ and $e = e_{50}$, where $e_{50} = e_{50}(k, n, \gamma)$ makes $\mathit{WIG}(k, n, e_{50}, \gamma)$ instances satisfiable with 50\% probability.
The values of $e_{50}(k, n, \gamma)$ are obtained empirically for each $k$, $n$ and $\gamma$ and are freely available to facilitate future work~\cite{SourceCodes}.

 Although we do not use it directly, it is also possible to estimate $e_{50}$ analytically based on the instance generator properties.
 In \appendixB, we give an (approximate) computation in the style of an `annealed estimate'~\cite{SelmanKirkpatrick1996:AIJ-critical-behavior} of the average number of solutions given $k$, $n$, $e$ and $\gamma$; with the intent to use it to obtain an approximate value of $e_{50}(k, n, \gamma)$.
 The annealed estimate seeks the average number of solutions over all instances, and so gives an upper bound on the PT -- because once the average drops below 0.5 then at least half the instances must be unsatisfiable.
 (A similar concept is used in the constrainedness parameter~\shortcite{Gent1996} which is expressed via the average number of feasible solutions in an ensemble of problem instances.)
 The main novelty of our analysis is that it accounts for the unevenness of distribution of solutions arising from the two-level nature of the problem.
 In particular, we observed that a straightforward strategy of direct estimation of the probability of a single plan being valid does not give a tight estimate in our case.
 Indeed, there might be millions of authorised plans per pattern but at the same time the expected number of eligible patterns can be well below one.
 In that case, most of the instances will have no valid plans at all but some very rare instances will have millions of valid plans.
 Since the straightforward estimation strategy gives the average number of valid plans per instance, its result is likely to be a significant over-estimate of the position of the PT.
 In order to estimate the critical point $e_{50}$ more accurately, we have to ask a different question: we need to know when the probability of an instance to have at least one valid plan is 50\%.
 For this, we have to estimate the number of eligible patterns and then the probability of a pattern being authorised.
 This will give us the expected number of valid \emph{patterns} which yields a more accurate estimate of the critical point $e_{50}$.

 Lastly, to further support that the observed phenomena have properties expected of a PT, we conducted a set of experiments at the critical points and around them.
 In particular, we show in \appendixC{} the emergence of \emph{forced variables} similar to~\cite{Culberson2001:frozen-graph-coloring}, i.e.\ the decision variables with values forced by the instance, in the critical region.
 We observe that the PT coincides with a rapid growth of the number of $M$ variables forced to be either 0 or 1, effectively corresponding to forced (not included explicitly) not-equals or ``equals'' constraints, respectively.

%%%%%%%%%%%%%%%%%%%%%%%%%%%%%%%%%%%%%%%%%%%%%%%%%%%%%%%%%%%%%%%%%%%%
\section{Computational Experiments}  
\label{sec:experiments}

 Specifically, we empirically study and compare the following WSP solvers:
\begin{description}
	\item[PBT] The algorithm proposed in this paper;
	\item[PUI] The FPT algorithm proposed and evaluated in \cite{JOCO2014};
	\item[UDPB (Res)] The old pseudo-Boolean SAT formulation of the problem (see Section~\ref{sec:old-formulation}) solved with SAT4J~\cite{BePa10} in the resolution proof system mode.
    We also attempted to use the cutting planes mode to solve UDPB but the performance was prohibitively poor for running the experiments.
	\item[PBPB (Res)] The new pseudo-Boolean SAT formulation, PBPB, of the problem (see Section~\ref{sec:new-formulation}) solved with SAT4J in the resolution proof system mode.
	\item[PBPB (CutP)] PBPB solved with SAT4J in the cutting planes proof system mode.
	
	\item[CSP (CP-SAT)] Our CSP formulation solved with CP-SAT, the latest constraint solver from OR-Tools.
\end{description}

 The PBT algorithm is implemented in C\#, and the PUI algorithm is implemented in C++.  
 Our test machine is based on two Intel Xeon CPU E5-2630 v2 (2.6~GHz) and has 32~GB RAM installed.  
 Hyper-threading is enabled, but we never run more than one experiment per physical CPU core concurrently, and concurrency is not exploited in any of the tested solution methods.

%%%%%%%%%%%%%%%%%%%%%%%%%%%%%%%%%%%%%%%%%%%%%%%%%%%%%%%%%%%%%%%%%%%%
\subsection{Slices in Studying Algorithm Scaling} 
\label{sec:slices}

 As discussed in Section~\ref{sec:testbed}, we focus on the PT WSP instances in our computational study.
 In a standard (non-FPT) study of PTs, this is generally straightforward -- at least conceptually, though potentially quite computationally challenging.
 For example, consider standard Random-3SAT; the size of the problem is indicated by the number, $n$, of propositional variables. 
 For each value of $n$, the number of clauses $c$ is selected so that the instances have a 50\%{} probability of being satisfiable, which we might write as ``set $c = c_{50}$''.
 Then, to study the algorithm's complexity, one tests it on these PT instances.
 
 However, a key aspect of FPT is that, in addition to a main problem size parameter $n$, it also has some other parameter $k$ which is closely involved in the problem complexity.
 One will generally wish to study the algorithm's scalability in terms of both $k$ and $n$.
 We call $n$ and $k$ \emph{size parameters} following the observation that they control the size of the space of WSP solutions.
 Consequently, we say that $(k, n)$ is the \emph{size space}.
 The remaining parameters $e$ and $\gamma$ are then \emph{constraint parameters} as they control the number of constraints, and they are chosen such that the instances have 50\% chance of being satisfiable as discussed in Section~\ref{sec:setting-density}.

 Since $(k, n)$ is two dimensional, studying the performance over the whole size space is computationally expensive and also difficult to analyse. 
Accordingly, in this paper, for simplicity and clarity we study the scaling along one-dimensional subspaces of $(k,n)$, which we will refer to as \emph{slices}.
 Since the size space is two-dimensional, we need to study the scaling in at least two independent (not necessarily orthogonal) directions; or along two independent one-dimensional ``slices''. 
 While studying the FPT properties, it is natural that such slices should also tend to focus on the regions in which $k$ is small compared to $n$.

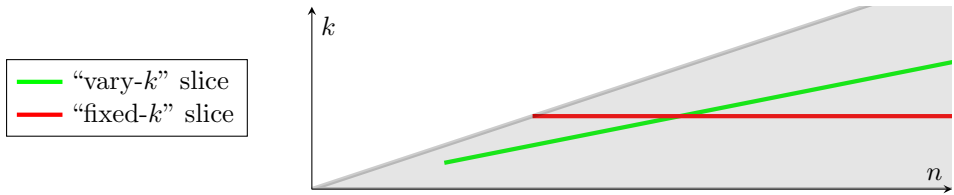
\begin{figure}[htb]
	\begin{center}
	\begin{tikzpicture}
		\begin{axis}[		
            compat=newest,
            width=10cm,
            height=4cm,
            legend pos=north west,
            xlabel={$n$},
            ylabel={$k$},
            title={},
            %legend cell align=left,
            grid=none,
            axis lines = middle,
            ticks=none,
            xmin=0,
            ymin=0,
            xmax=2.9,
            ymax=2.5,
            legend style={at={(-0.1,0.5)},anchor=east},
            legend cell align=left,
        ]
			\addplot[ultra thick, green, domain=0.6:3] {0.6 * x};
            \addlegendentry{``vary-$k$'' slice}        
        
       		\addplot[ultra thick, red, domain=1:3] {1};
			\addlegendentry{``fixed-$k$'' slice}          
		
			\addplot[opacity=0.2, fill=gray, ultra thick] coordinates {(0,0) (3,3) (3,0)} \closedcycle;
    	\end{axis}
	\end{tikzpicture} 
	\end{center}
	\caption{
    	Schematic view of the two ``slices'' used in our computational study in order to cover the two-dimensional size space $(k,n)$ used in order to investigate the empirical FPT properties. 
        Such FPT studies should naturally focus on the lower right (grey) region in which $k$ is small compared to $n$.
        }
	\label{fig:slices}
\end{figure}

 Many options of choosing the slices are possible, including non-linear slices, however in this paper we will just use two linear slices schematically illustrated in Figure~\ref{fig:slices}, as they give a good and useful insight into the behaviour: 
\begin{description}
	\item[``vary-$k$'']
    Vary the value of $k$, but the value of $n$ is given as a specified function of $k$.
    In this paper, following~\cite{JOCO2014,KaGaGu}, we use the choice $n = 10k$. 
    This gives a simple and clean way of keeping $k$ to be `small' compared to $n$. 
    We also report the experimental results for the $n = 100k$ slice in \appendixD, however the conclusions are somewhat similar.
    
	\item[``fixed-$k$'']
    Use a constant value of $k$ and vary $n$.
    This is a natural slice for a test of FPT performance; recall that the worst-case time complexity grows polynomially with $n$ at a fixed $k$, and one can expect the algorithms to demonstrate good scalability in this slice.

\end{description}

%%%%%%%%%%%%%%%%%%%%%%%%%%%%%%%%%%%%%%%%%%%%%%%%
\subsection{Performance Comparison: Slice ``vary-k''} 
\label{sec:vary-k-slice}

 To compare performance of various WSP algorithms, for each value of $k$ we generated 100 instances using $\mathit{WIG}(k, 10k, e_{50}, k)$. 
 %It is particularly important to select instances lying at the crossover point (50\% chance of being satisfiable), because otherwise the slice through the parameter space might move off of the phase transition and this may well be expected to distort the scaling properties.
 The empirical number of not-equals constraints $e_{50}$ for each $k$ needed to obtain PT instances is shown on Figure~\ref{fig:density-vs-k}.
One can observe the `zig-zag' shape of the curve in that the values corresponding to odd $k$'s are greater than the values corresponding to even $k$'s.
This minor artefact arises simply because the size of each authorisation list is randomly drawn by our instance generator from $[1, \lfloor 0.5k - 0.5 \rfloor]$.
As a result, the average number of authorisations in an instance with $k = 2i - 1$ is equal to that in an instance with $k = 2i$, $i \in \mathbb{N}$.
This makes the authorisations in `even' instances slightly more constrained compared to `odd' instances, which is then compensated by reduced number $e$ of not-equals constraints.

 We then solved each instance with each of the algorithms and in Figure~\ref{fig:performance} we report the median running time.
 We report separate times for the satisfiable and unsatisfiable instances.
 The immediate observations are that all the algorithms demonstrate roughly exponential growth of the running time, but that the performances of all the methods differ widely.
 
 \pgfplotsset{PBT/.append style={red}}
\pgfplotsset{PBPB/.append style={blue}}
\pgfplotsset{PUI/.append style={green}}
\pgfplotsset{UDPB/.append style={black}}
\pgfplotsset{PBPB CP/.append style={olive, no marks}}
\pgfplotsset{CSP/.append style={orange, no marks}}

\pgfplotsset{sat/.append style={very thick, dashed}}
\pgfplotsset{unsat/.append style={solid, ultra thick}}

\begin{figure}[tbhp]
%\pgfplotsset{yticklabel style={text width=3.2em,align=right}}
\begin{subfigure}{\textwidth}
\begin{flushright} 
\begin{tikzpicture}[trim axis right,trim axis left]
	\begin{semilogyaxis}[
		compat=newest,
		width=\textwidth,
		height=7.4cm,
		legend pos=south east,
		xlabel={Number of steps $k$},
		ylabel={Running time, sec},
		title={},
		legend cell align=left,
		grid=major,
        unbounded coords=jump,
        xmin=17,
        xmax=59
	]
	\addplot[PBT, sat, forget plot] table[
		x=k,
		y=pbt-sat-50,
	] {vary-k-10.dat};
	\addplot[PBT, unsat] table[
		x=k,
		y=pbt-unsat-50,
	] {vary-k-10.dat};
	\addlegendentry{PBT}

% 	\addplot[PUI, sat, forget plot] table[
% 		x=k,
% 		y=pbt-sat-50,
% 	] {vary-k-5.dat};
% 	\addplot[PUI, unsat] table[
% 		x=k,
% 		y=pbt-unsat-50,
% 	] {vary-k-5.dat};
% 	\addlegendentry{PBT n=5k}
    
%     \addplot[UDPB, sat, forget plot] table[
% 		x=k,
% 		y=pbt-sat-50,
% 	] {vary-k-100.dat};
% 	\addplot[UDPB, unsat] table[
% 		x=k,
% 		y=pbt-unsat-50,
% 	] {vary-k-100.dat};
% 	\addlegendentry{PBT n=100k}
	
	\addplot[PUI, sat, forget plot] table[
		x=k,
		y=ui-sat-50,
	] {vary-k-10.dat};
	\addplot[PUI, unsat] table[
		x=k,
		y=ui-unsat-50,
	] {vary-k-10.dat};
	\addlegendentry{PUI}

	\addplot[UDPB, sat, forget plot] table[
		x=k,
		y=sat4j-sat-50,
	] {vary-k-10.dat};
	\addplot[UDPB, unsat] table[
		x=k,
		y=sat4j-unsat-50,
	] {vary-k-10.dat};
	\addlegendentry{UDPB (Res)}

	\addplot[PBPB, sat, forget plot] table[
		x=k,
		y=mxsat-sat-50,
	] {vary-k-10.dat};
	\addplot[PBPB, unsat] table[
		x=k,
		y=mxsat-unsat-50,
	] {vary-k-10.dat};
	\addlegendentry{PBPB (Res)}

	%\addplot[gray, ultra thick, unsat, domain=18:58] 
    %{
    %	1e-5 * 2^(0.11*x * ln(x))
    %};
	%\addlegendentry{$2^{0.11k \cdot \log_2{k}}$}

	%\addplot[MxMIP, sat, forget plot] table[
	%	x=k,
	%	y=mip-sat-50,
	%] {vary-k-10.dat};
	%\addplot[MxMIP, unsat] table[
	%	x=k,
	%	y=mip-unsat-50,
	%] {vary-k-10.dat};
	%\addlegendentry{MxMIP (will remove)}
    
    \addplot [PBPB CP, sat, forget plot] table [
     	x=k, 
        y=mxpb-cp-sat-50
    ] {vary-k-10.dat};
    \addplot [PBPB CP, unsat, mark=*] table [
       	x=k, 
        y=mxpb-cp-unsat-50
    ] {vary-k-10.dat};
    \addlegendentry{PBPB (CutP)}    
	
    % \addplot [CSP, sat, forget plot] table [
    %  	x=k, 
    %     y=csp-sat-50
    % ] {csp-vary-k-10.dat};
    % \addplot [CSP, unsat, mark=*] table [
    %   	x=k, 
    %     y=csp-unsat-50
    % ] {csp-vary-k-10.dat};
    % \addlegendentry{CSP (CP-SAT)}       

% csp-vary-k-10-NEW.dat is for CSP data after forlond rebuild
    \addplot [CSP, sat, forget plot] table [
     	x=k, 
        y=csp-sat-50
    ] {csp-vary-k-10-NEW.dat};
    \addplot [CSP, unsat, mark=*] table [
       	x=k, 
        y=csp-unsat-50
    ] {csp-vary-k-10-NEW.dat};
    \addlegendentry{CSP (CP-SAT)}       
    
	%\addplot[sat, yellow, forget plot] table[
	%	x=k,
	%	y=oldmip-sat-50,
	%] {vary-k-10.dat};
	%\addplot[unsat, yellow] table[
	%	x=k,
	%	y=oldmip-unsat-50,
	%] {vary-k-10.dat};
	%\addlegendentry{UD-MIP (temp.)}
    
	%\addplot[sat, orange, forget plot] table[
	%	x=k,
	%	y=sat4j-cp-sat-50,
	%] {vary-k-10.dat};
	%\addplot[unsat, orange] table[
	%	x=k,
	%	y=sat4j-cp-unsat-50,
	%] {vary-k-10.dat};
	%\addlegendentry{UDPB CP (temp.)}    
\end{semilogyaxis}
\end{tikzpicture}
\end{flushright} 
\caption{The solid lines correspond to unsatisfiable instances, and dashed to satisfiable instances.
% \AJP{Note that for PBT, the n=5k, n=10k, and n=100k lines are very similar. Presumably, it is all about the UI, and they don't "see n".}
}
\end{subfigure}
\\[3ex]
\begin{subfigure}{\textwidth}
	\begin{flushright} 
	\begin{tikzpicture}[trim axis right,trim axis left]
        \begin{semilogyaxis}[		
            compat=newest,
            width=\textwidth,
            height=4.5cm,
            legend style={at={(0.6, 0.95)}},
            xlabel={Number of steps $k$},
            ylabel={Time over $2^{0.5k}$},
            title={},
            legend cell align=left,
            grid=major,
            xmin=17,
            xmax=59]       
		\addplot+[name path=A, red!50!black, thin, solid, mark=none, forget plot] table[
            x=k,
            y expr={\thisrow{pbt-unsat-35} / (2^(0.5  * x ))},
        ] {vary-k-10.dat};
        \addplot+[name path=B, red!50!black, thin, solid, mark=none, forget plot] table[
            x=k,
            y expr={\thisrow{pbt-unsat-65} / (2^(0.5  * x ))},
        ] {vary-k-10.dat};

        \addplot[red, opacity=0.2, forget plot] fill between[of=A and B];
        
%        \addplot[red, ultra thick] table[
        \addplot[red, mark=x, mark size=4pt] table[
            x=k,
            y expr={\thisrow{pbt-unsat-50} / (2^(0.5  * x ))},
        ] {vary-k-10.dat};
		\addlegendentry{PBT unsat}
        
        %\addplot[blue, ultra thick, unsat, opacity=0.5, domain=18:58] 
   		%{
    	%	3.5e-6 * 2^(0.05  * x)
	    %};
		%\addlegendentry{$3.5 \cdot 10^{-6} \cdot 2^{0.5k}$}
        
        \addplot[blue, ultra thick, unsat, opacity=0.5, domain=18:58] 
   		{
    		(4e-5 * 2^(0.076  * x * log2(x) )) / (2^(0.5  * x ))
	    };
		\addlegendentry{$\text{const} \cdot 2^{k \cdot \log_2{k} / 13.2}$}        
    \end{semilogyaxis}
    \end{tikzpicture}
	\end{flushright} 
	\caption{
        PBT unsat running time with the [35--65] percentile range (shaded) as an indication of a confidence interval on the medians.
        The vertical axis is rescaled to better show the fit.
        Any exponential function would be a straight line in this plot, however empirical running time appears to be curved.
        The blue line appears to be an accurate approximation of the running times, demonstrating that scaling of PBT running time closely follows $B_k$, with a speed-up power factor 13.2.
        %\AJP{COMMENT: yes i prefer this version as is easier to explain and makes it much visually clearer the actual behaviour is not a plain exponential. Any chance to give Rsquared values? Note the value is probably very good for the manifestly  bad fit as well; but better for the proper one. out of curiosity, but not for the paper: Is there a statistical test that would reject the plain exponential fit? not sure how to do this except maybe to fit on a subrange of small k, and then test on large k. After all the point is to predict the value at larger k.  }
        %\DK{To AJP: What is Rsquared?  Sum of squared residuals?}
        %\AG{Yes, this new figure looks much better: the blue line follows the shaded region. Description in the text should make it more clear what's going on.}
        %\DK{Andrei, the explanation is there already -- beginning of next page.}
%\AJP{maybe we can internally also do a similar plot for the PBPB(res) and CSP unsat data, and just report on the speedup power factor. 
%Their speedup factor seems to be very similar, but maybe just a bit smaller? }
}
    \label{fig:pbt-unsat-fit-scaled}
\end{subfigure}
\\[3ex]
\begin{subfigure}{\textwidth}
	\begin{flushright} 
     \begin{tikzpicture}[trim axis right,trim axis left]
        \begin{axis}[		
            compat=newest,
            width=\textwidth,
            height=3.5cm,
            legend pos=north west,
            xlabel={Number of steps $k$},
            ylabel={$e_{50}$},
            title={},
            %legend cell align=left,
            grid=major,
            xmin=17,
            xmax=59]       
        ]
        \addplot[ultra thick] table[
            x=k,
            y=e,
    ] {vary-k-10.dat};
        \end{axis}
     \end{tikzpicture}
	\end{flushright} 
    \caption{
		The number $e_{50}$ of not-equals constraints at phase transition.
     }
     \label{fig:density-vs-k}
     \label{fig:e-vs-k}
\end{subfigure}
\caption{
	Evaluation of algorithms' performance along the ``vary-$k$'' slice, i.e.\ $\mathit{WIG}(k, 10k, e_{50}, k)$.}
\label{fig:performance}
\end{figure}
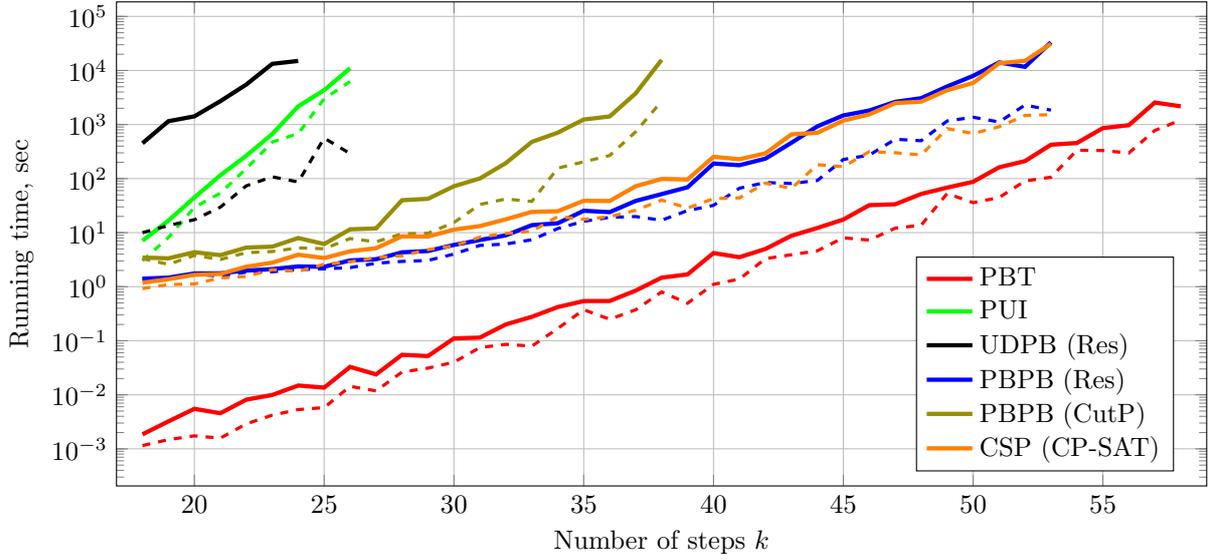
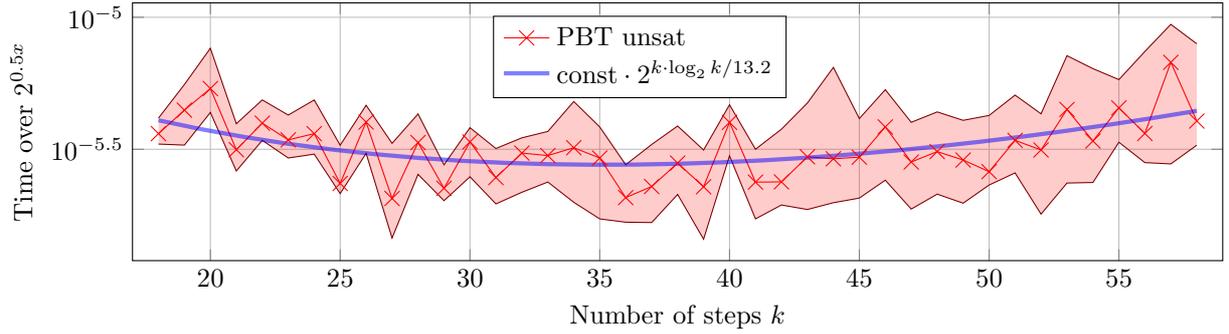
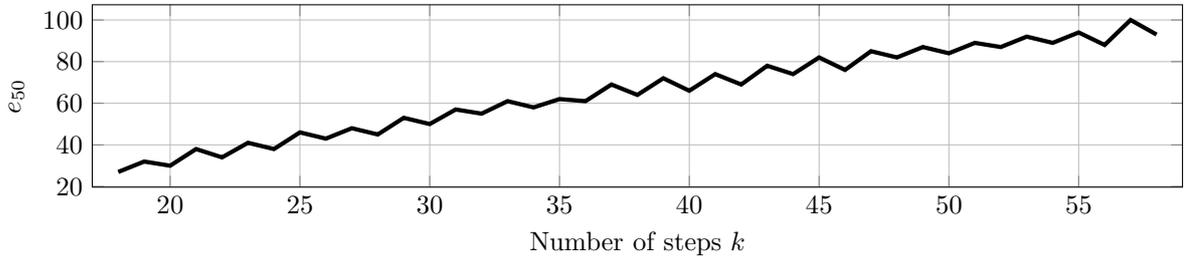

 Crucially, the new PBPB (Res) and PBT (Res) both drastically outperform the previous UDPB (Res) and PUI, showing lower growth rate and also being significantly faster even on small instances.
 We first look at the scaling of the best performing PBT, on the unsatisfiable instances as the discussion of scaling of satisfiable instances can be obscured by finding solutions early in the search tree.
 Although it is not immediately obvious, the scaling of the PBT on unsatisfiable instances in Figure~\ref{fig:performance} is slightly super-exponential; the empirical curve bends slight upwards and so $2^{ak}$, with a constant $a$, does not give a convincing fit. 
Deep analysis of the average case effects of heuristic improvements in such tree-based search is not yet possible, but generally the expectation, based on experience, is that heuristics will improve the coefficients in the exponents but retain the form.
 Recalling from Section~\ref{sec:worst-case-analysis}, that the number of patterns scales as $2^{\Theta(k\log k)}$, it is reasonable to compare the empirical scaling to $2^{(k \log_2 k) / b}$ for some empirically determined constant $b$. 
 In Figure~\ref{fig:pbt-unsat-fit-scaled} we show that a good fit is a function $2^{k \cdot \log_2 k/13.2}$, confirming our assumption and indicating the effectiveness of the branching heuristics and pruning.

\begin{figure}[htb]
\pgfplotsset{yticklabel style={text width=2.5em,align=right}}
\begin{tikzpicture}
	\begin{semilogyaxis}[
		compat=newest,
		width=0.97\textwidth,
		height=5cm,
		legend pos=north east,
		xlabel={Number of steps $k$},
		ylabel={PBPB / PBT},
		title={},
		legend cell align=left,
		grid=major,
        ymin=10,
        xmin=17,
        xmax=59
	]
	\addplot[sat] table[
		x=k,
		y expr=\thisrow{mxsat-sat-50} / \thisrow{pbt-sat-50},
	] {vary-k-10.dat};
	\addlegendentry{Satisfiable}

	\addplot[unsat] table[
		x=k,
		y expr=\thisrow{mxsat-unsat-50} / \thisrow{pbt-unsat-50},
	] {vary-k-10.dat};
	\addlegendentry{Unsatisfiable}
	\end{semilogyaxis}
\end{tikzpicture}
\caption{
	Comparison of PBPB and PBT performance.
	The vertical coordinate is the ratio between the median running time of PBPB (Res) and median running time of PBT.
    }
\label{fig:ratio}
\end{figure}
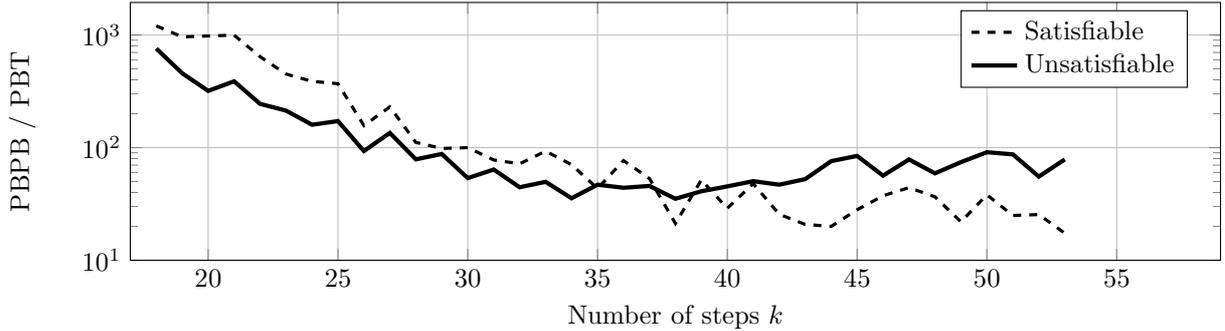

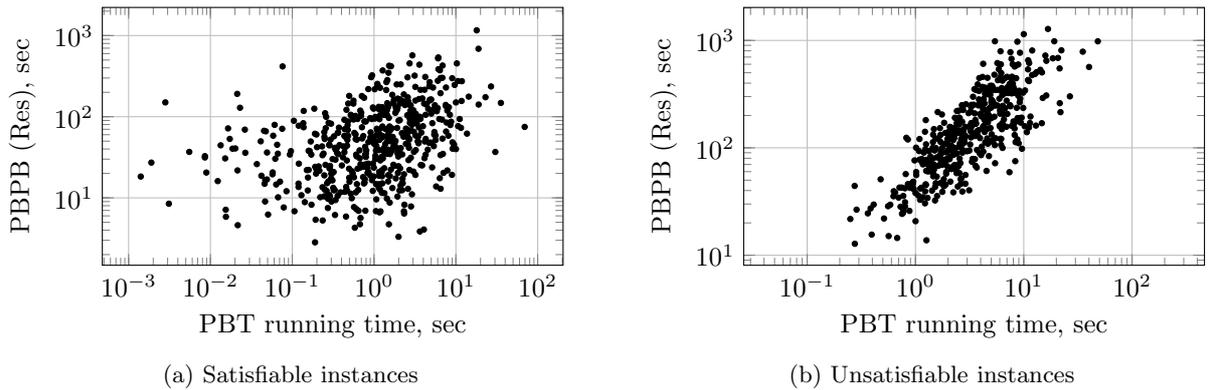
\begin{figure}[htb]
	\begin{subfigure}{0.48\textwidth}
		\skipfigure{
		\begin{tikzpicture}
		\begin{loglogaxis}[
	        axis equal,
			width=\textwidth,
			xlabel={Running time, sec},
			height=5cm,
			grid=major,
			xlabel={PBT running time, sec},
			ylabel={PBPB (Res), sec},
			tick label style={/pgf/number format/fixed}]
		\addplot [only marks, black, mark size=1pt]
			table [x=pbt-sat, y=mxpb-sat] {correlation-sat.dat};

		\end{loglogaxis}
		\end{tikzpicture}
		}
		\caption{Satisfiable instances}
	\end{subfigure}
	\hspace{0.04\textwidth}
	\begin{subfigure}{0.48\textwidth}
		\skipfigure{
		\begin{tikzpicture}
		\begin{loglogaxis}[
	        axis equal,
			width=\textwidth,
			xlabel={Running time, sec},
			height=5cm,
			grid=major,
			xlabel={PBT running time, sec},
			ylabel={PBPB (Res), sec},
			tick label style={/pgf/number format/fixed}]
		\addplot [only marks, black, mark size=1pt]
			table [x=pbt-unsat, y=mxpb-unsat] {correlation-unsat.dat};

		\end{loglogaxis}
		\end{tikzpicture}
		}
		\caption{Unsatisfiable instances}		
	\end{subfigure}
	\caption{
		Correlation between PBT and PBPB (Res) running times on instances of size $k = 40$.
		1000 instances used in this experiment (about 500 satisfiable and 500 unsatisfiable).
	}
	\label{fig:correlation}
\end{figure}

 The next observation is that PBT is faster than PBPB (Res) by one to two orders of magnitude, but that the scaling behaviours are similar.
 A more accurate comparison can be done by inspecting Figure~\ref{fig:ratio}.
 While the limited range of $k$ and noise do not allow us to convincingly make any definite conclusions, we hypothesise that the ratio of the runtimes between PBT and PBPB (Res) is close to a constant for large $k$.
 Also, PBPB (Res) is relatively slow on small instances, presumably because of an expensive initialisation/preprocessing, normal for an off-the-shelf solver, but this does not affect the method's scalability.
 The similarity of the scaling on larger instances ($k > 35$) supports our hypothesis from Section~\ref{sec:branching-strategies} that the search processes of PBPB (Res) and PBT could well be similar.

It is also clear that PBPB (Res) and CSP (CP-SAT) show (remarkably) similar performance, despite CSP (CP-SAT) not explicitly using patterns.
In this paper, we treat the PBPB and CSP solvers in an ``off-the-shelf blackbox'' fashion and do not study their internal workings. 
However, CP-SAT is built on top of a SAT solver, and so we speculate that the search in the SAT solver effectively mimics a pattern-based reasoning. This could be directly through the the $v_k$ variables. Alternatively, conversion to SAT could introduce a new variable for each term or predicate that occurs in the constraints, including the terms $x_i = x_j$ and $x_i \neq x_j$ in (\ref{eq:csp-not-equals}), (\ref{eq:csp-at-least}) and (\ref{eq:csp-at-most}) -- such Boolean variables would play a similar role to the pattern variables $M_{ij}$.

 We also observe that the ratio between the solution time of unsatisfiable and satisfiable instances steadily grows for PBPB (Res), while it stays roughly constant for PBT\@.
 This may be explained by the fact that the PB solver is likely to employ some heuristics for ordering search branches; these heuristics generally improve the running time of the solver on satisfiable instances while leaving its performance on unsatisfiable instances intact.
 PBT does not currently have any such heuristic; in our attempt to implement one, the gain was comparable to the overheads, and, thus, we dropped the branch ordering heuristic in our final implementation of PBT\@.

 We also directly investigated whether the running times of PBPB (Res) correlate with the running times of PBT, see Figure~\ref{fig:correlation}.
 On satisfiable instances the correlation is relatively weak which is natural as the running time depends on the branching decisions which differ in the two algorithms.
 On unsatisfiable instances the correlation is much stronger which again shows that, although the individual branching decisions of the two algorithms may be different, the effectiveness of the heuristics is comparable.
 
The gap between the running times on satisfiable and unsatisfiable instances is relatively small (within one order of magnitude) for all the solvers except for UDPB (Res).
We hypothesise that this difference is due to the inherent symmetry of the UDPB formulation, meaning that there are many valid plans and there is a high probability of finding one well before exhausting the entire search tree.
In contrast, the number of valid patterns in a PT instance is likely to be small, and one is likely to be found only after searching a significant part of the search tree.
 Note that PBPB (Res) is still superior to UDPB (Res) on satisfiable instances, as the plans search tree is much larger than the patterns search tree when $n \gg k$.

\begin{figure}[htb]
\pgfplotsset{yticklabel style={text width=2.5em,align=right}}
\begin{tikzpicture}
	\begin{semilogyaxis}[
		compat=newest,
		width=\textwidth,
		height=7cm,
		legend pos=north west,
		xlabel={Number of steps $k$},
		ylabel={Running time, sec},
		title={},
		legend cell align=left,
		grid=major,
        %unbounded coords=jump,
        xmin=17,
        xmax=59
	]
	\addplot[blue, ultra thick, dashed] table[
		x=k,
		y=beta0.5,
	] {vary-k-betas.dat};
	\addlegendentry{$\beta = 0.5$}

	\addplot[green, ultra thick, dashed] table[
		x=k,
		y=beta0.75,
	] {vary-k-betas.dat};
	\addlegendentry{$\beta = 0.75$}

	\addplot[red, ultra thick] table[
		x=k,
		y=beta1,
	] {vary-k-betas.dat};
	\addlegendentry{$\beta = 1$ (PT)}

	\addplot[gray, ultra thick] table[
		x=k,
		y=beta1.25,
	] {vary-k-betas.dat};
	\addlegendentry{$\beta = 1.25$}

	\addplot[black, ultra thick] table[
		x=k,
		y=beta1.5,
	] {vary-k-betas.dat};
	\addlegendentry{$\beta = 1.5$}
\end{semilogyaxis}
\end{tikzpicture}
\caption{
	Performance of PBT on instances outside the PT region (obtained by changing $\beta$).
}
\label{fig:different-beta}
\end{figure}
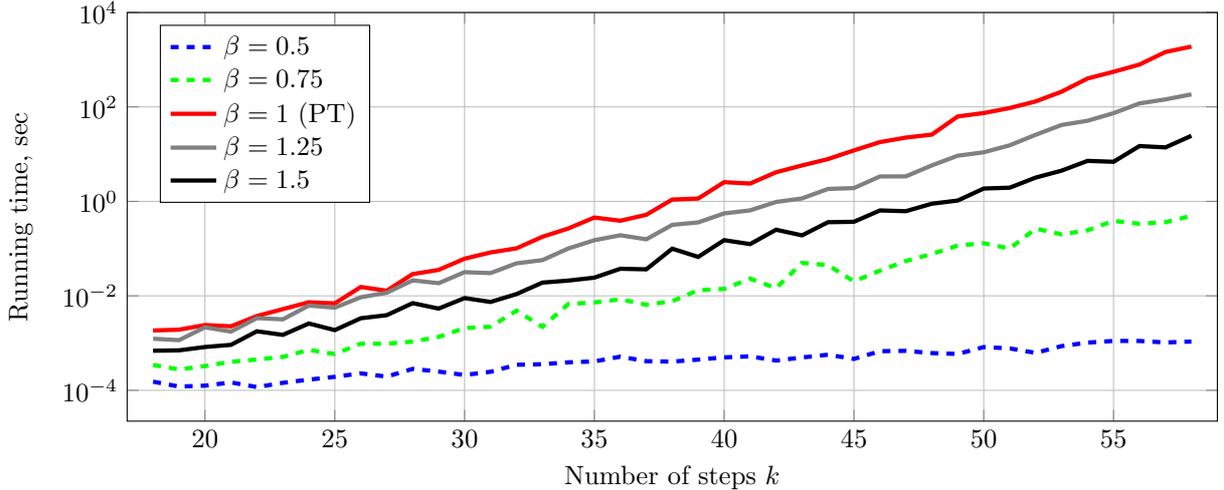
 
 Although we argued that good experiments should seek instances in the PT region of parameter space, it is still interesting to verify the performance of PBT on under- and over-constrained instances.
 To build such instances in a consistent way, we define a new parameter $\beta$, and study instances $\mathit{WIG}(k, 10k, \beta e_{50}(k, 10k, k), \beta k)$, i.e.\ the PT instances with the number of constraints scaled by $\beta$.
 Figure~\ref{fig:different-beta} shows how the scaling changes as we move away from the PT, $\beta = 1.0$. 
 It shows the classic so-called ``easy-hard-easy'' behaviour. 
 Below the PT ($\beta < 1$) the runtimes and scaling are much better than at the PT; that is most of the instances have many solutions, and so solving will terminate early.
  Above the PT ($\beta > 1$) most of the instances are unsatisfiable, but the pruning will increase, and this is reflected in the improved scaling.

%%%%%%%%%%%%%%%%%%%%%%%%%%%%%%%%%%%%%%%%%%%%%%%%
\subsection{Performance Comparison: Slice ``fixed-k"} \label{sec:fixed-steps}

 All the FPT algorithms discussed in this paper are designed to solve large WSP with relatively small number $k$ of steps.
 This reflects the fact that in large organisations there might be thousands of users with only tens of steps in an instance.
 Thus it is of both theoretical and practical importance to evaluate scalability of the approaches with regards to the number of users $n$.
 
\begin{figure}[bth]
%\pgfplotsset{yticklabel style={text width=3em,align=right}}
\begin{subfigure}{\textwidth}
	\begin{flushright}
    \begin{tikzpicture}[trim axis right,trim axis left]
		\begin{loglogaxis}[
			width=0.97\textwidth,
			xlabel={Number of users $n$},
			ylabel={Running time, sec},
			height=10cm,
			grid=major,
            legend cell align=left,
            legend pos=north east,
            xmin=30,
            xmax=1000000000,
            ymax=50000,
            ymin=0.001
        ]
		
        \addplot [PBT, sat, forget plot] table [x=n, y=pbt-sat] {vary-n-NEW.dat};
        \addplot [PBT, unsat] table [x=n, y=pbt-unsat] {vary-n-NEW.dat};
        \addlegendentry{PBT}

%% AJP:
%% I think vary-n.dat if from DK laptop?
%% vary-n-NEW.data is from the rebuilt forlond
%% they seem essentially the same so can use either?
%% or will rerun the PBT when have possible.

        %\addplot [PUI, sat, forget plot] table [x=n, y=pbt-sat] {vary-n-NEW.dat};
        %\addplot [PUI, unsat] table [x=n, y=pbt-unsat] {vary-n-NEW.dat};
        %\addlegendentry{PBT NEW}

        \addplot [PUI, sat, forget plot] table [x=n, y=ui-sat] {vary-n.dat};
        \addplot [PUI, unsat] table [x=n, y=ui-unsat] {vary-n.dat};
        \addlegendentry{PUI}

        \addplot [UDPB, sat, forget plot] table [x=n, y=pb-sat] {vary-n.dat};
        \addplot [UDPB, unsat] table [x=n, y=pb-unsat] {vary-n.dat};
        \addlegendentry{UDPB (Res)}
        
        \addplot [PBPB, sat, forget plot] table [x=n, y=mx-sat] {vary-n.dat};
        \addplot [PBPB, unsat] table [x=n, y=mx-unsat] {vary-n.dat};
        \addlegendentry{PBPB (Res)}
		
        \addplot [PBPB CP, sat, forget plot] table [x=n, y=mxpbcp-sat] {vary-n.dat};
        \addplot [PBPB CP, unsat, mark=*] table [x=n, y=mxpbcp-unsat] {vary-n.dat};
        \addlegendentry{PBPB (CutP)}

        \addplot [CSP, sat, forget plot] table [x=n, y=csp-sat] {vary-n-NEW.dat};
        \addplot [CSP, unsat, mark=*] table [x=n, y=csp-unsat] {vary-n-NEW.dat};
        \addlegendentry{CSP (CP-SAT)}

%%%% AJP: From separate figure
        %\addplot [MxMIP, sat, forget plot] table [x=n, y=mip-sat] {vary-n.dat};
        %\addplot [MxMIP, unsat] table [x=n, y=mip-unsat] {vary-n.dat};
        %\addlegendentry{MxMIP}
		
%%%% above From separate figure

%        \addplot [gray, ultra thick, dotted, domain=50:1000] {0.3e-6 * x^(4)};

%        \addplot [gray, ultra thick, dotted, domain=500:20000] {30e-3 * x^(0.5*log10(x))};

        \addplot [green, line width=3mm, opacity=0.2, domain=2000:150000] {0.003 * x^(1.25)};
        \addlegendentry{$\text{const} \cdot n^{1.25}$}
        
        % \addplot [blue, line width=3mm, opacity=0.2, domain=10000:150000] {0.00000012 * x^(2.0)};
        % \addlegendentry{$\text{const} \cdot n^{2.0}$}

        \addplot [olive, line width=3mm, opacity=0.2, domain=2000:300000] 
        	{0.002 * x^(0.9)};
        \addlegendentry{$\text{const} \cdot n^{0.9}$}

%         \addplot [green, line width=3mm, opacity=0.2, domain=8000:3000000] 
%         	{0.001 * x^(0.90)};
%         \addlegendentry{$\text{const} \cdot n^{0.90}$}

        %\addplot [orange, line width=3mm, opacity=0.2, domain=8000:3000000] 
        %	{0.000275 * x};
        %\addlegendentry{$\text{const} \cdot n$}

        \addplot [orange, line width=3mm, opacity=0.2, domain=50000:3000000] 
        	{0.00115 * x^(0.9)};
        \addlegendentry{$\text{const} \cdot n^{0.9}$}

        \addplot [orange, thin, decorate, decoration={coil,aspect=0}] coordinates {(2000, 0.001) (2000, 50000)};
        \addlegendentry{$n = n^*$}

        %\addplot [black, thick, solid, domain=36:300] {0.0025 * (x)^(0.46 * ln(x))};
        %\addlegendentry{$y = \text{const} \cdot x^{\text{const} \cdot \log{x}}$}
		\end{loglogaxis}
	\end{tikzpicture}
	\end{flushright}
    \caption{
    	Median runtimes for the different solvers, as a function of $n$.
    	Solid lines show running times for unsat instances and dashed lines for sat instances.}
	\label{fig:vary-n-loglog}
\end{subfigure}
\\[3ex]
\begin{subfigure}{\textwidth}
	\begin{flushright}
	\begin{tikzpicture}[trim axis right,trim axis left]
		\begin{semilogxaxis}[
			width=0.97\textwidth,
			xlabel={Number of users $n$},
			ylabel={$e_{50}$},
			height=4cm,
			grid=major,
            xmin=30,
            xmax=1000000000,
            ymin=0,
            ymax=45
        ]
		
        \addplot [black, ultra thick, mark=*] table [x=n, y=e] {vary-n-NEW.dat};
        %\addlegendentry{$e$}
        
        \addplot [orange, thin, decorate, decoration={coil,aspect=0}] coordinates {(2000, 0) (2000, 45)};
        
		\end{semilogxaxis}
	\end{tikzpicture}
	\end{flushright}
	\caption{
    	Number $e$ of not-equals constraints as a function of $n$ selected so as to be 50\% satisfiable.
        }
	\label{fig:vary-n-num-e}
\end{subfigure}
	\caption{
    	The ``fixed-$k$'' slice for $k = 18$, i.e.\ $\textit{WIG}(18, n, e_{50}, 18)$.
        %(a) Running time as a function of $n$ at the phase transition point.
        %(b) The number of edges $e$ required for the instances to be at the phase transition point. 
    	%  The running time of xPB is roughly proportional to $n^{\text{const} \log{n}}$.
}
\label{fig:vary-n}
\end{figure}
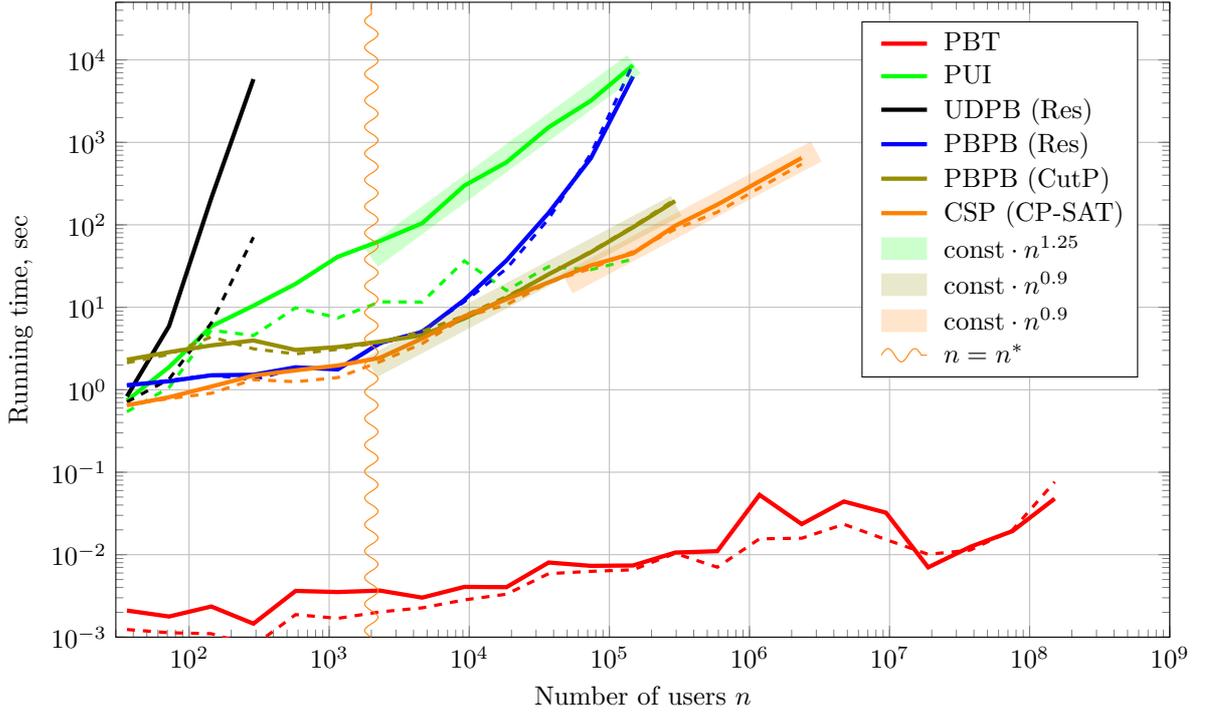
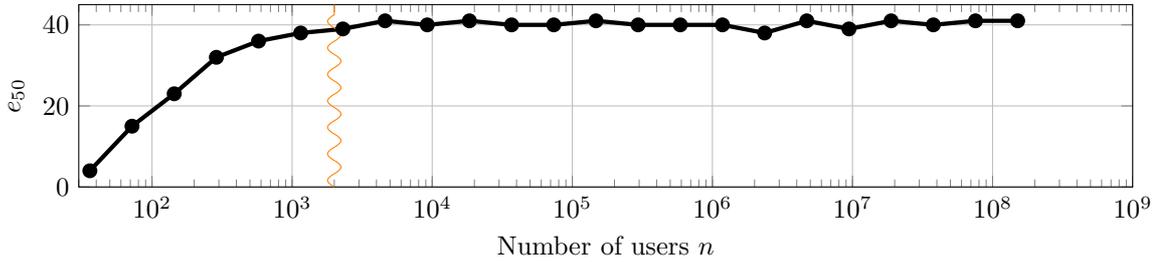

 Figure~\ref{fig:vary-n} reports on a set of experiments with $\mathit{WIG}(18, n, e_{50}, 18)$ instances.
 Note that a relatively small value of $k = 18$ was chosen to allow even the slowest algorithms to terminate in reasonable time.
 The (impractically large) maximum value of $n$ was selected to investigate behaviour of the algorithms.

 Figure~\ref{fig:vary-n-num-e} shows how the value of $e = e_{50}(18, n, 18)$ varies with $n$ so as to remain at the PT region.
 The corresponding results for the runtimes of the different algorithms are given in Figure~\ref{fig:vary-n-loglog}.
 Although all the instances at all values of $n$ are selected from a PT region, it seems that there are two regions above and below a dividing value of $n^* \approx 2000$. 
 For $n < n^*$ the number of not-equals constraints in the graph is increasing with $n$, but for $n > n^*$ the number of the not-equals constraints in the graph is roughly constant, and the properties of the instances are practically independent of $n$.

 The simplest behaviour to interpret is that of PUI on unsatisfiable instances, which exhibits a scaling that is approximately proportional to $n^{1.25}$. 
 This slightly superlinear scaling is natural as the algorithm works through the users one at a time, and if the instance is highly constrained by the not-equals constraints then the work per user may well become roughly constant, with only a mild accumulation of new patterns and, hence, increase of the patterns pool and associated runtime costs.
 On satisfiable instances PUI has a potential to solve the problem after $O(k)$ iterations, i.e.\ with a perfect user ordering heuristic its running time could theoretically be independent of $n$.
 However, the real user ordering heuristic does not always pick the ``right'' users and the running time mildly increases with $n$, matching our expectations.
 
 The running time of PBT shows little dependency on $n$.
 Observe that the upper level of the search algorithm in PBT does not depend on $n$.
 Due to the heuristic described in Section~\ref{sec:authorisations-pruning}, the size of the assignment graph (the lower level) is bounded by $k^2$, i.e.\ does not generally grow with $n$.
 Hence, only the generation of the assignment graph depends on $n$ in PBT\@.
 However, the larger the value of $n$, the less likely that full scans of the list of users are needed.
 As a result, PBT demonstrates sub-linear scaling in this experiment, solving instances with more than $10^8$ users in under a second.
 Even if not directly practical in case of WSP, this result shows that a careful design and implementation of an FPT algorithm has a potential to routinely address very large problems.

 In contrast to the experiments along the ``vary-$k$'' slice, the PBPB (Res) scaling on the ``fixed-$k$'' slice is not similar to that of PBT, and is possibly worse than polynomial (as the slope increases on the log-log plot).
 On the other hand, the PBPB (CutP) and CSP (CP-SAT), both show very good performance, demonstrating sub-linear scaling (roughly $n^{0.9}$), and outperforming PBPB (Res) on large $n$'s.
 For PBPB (CutP), it is natural to hypothesise that this is because at large $n$ the matching problem becomes even more important, making the cutting planes proof system more suitable.
 For CSP (CP-SAT), we can hypothesise that, unlike the SAT4J solver, the OR~Tools SAT solver efficiently discovers strategies similar to the one proposed in Proposition~\ref{th:matching-polysize}, and thus solves the problem in FPT time.
 This behaviour is interesting, and good news for SAT solvers, but requires future investigation as there might be potential for significant improvement to enhance the solver's performance using the fact that the problem is FPT.

 We note here that the off-the-shelf solvers (PBPB~(Res), PBPB~(CutP) and CSP~(CP-SAT) all compete with each other in the `vary-$k$' and `fixed-$k$' slices (see also \appendixD{} for the $n = 100k$ slice), however none of them is a universal winner.
 It is clear though that, as expected, the resolution proof system performs better than the cutting planes resolution system on the `vary-$k$' slice.

%%%%%%%%%%%%%%%%%%%%%%%%%%%%%%%%%%%%%%%%%%%%%%%%%%%%%%%%%%%%%
\section{Conclusion}  
\label{sec:conclusion}
 
 In this paper, we studied the Workflow Satisfiability Problem (WSP), with User Independent (UI) constraints, which admits FPT algorithms.
 Employing a new view of the FPT nature of the problem and the AI methods, we designed a new FPT algorithm, PBT, that significantly outperforms the previous state-of-the-art algorithm, extending the size of instances that can be reasonably tackled from $k \approx 20$ to $k \approx 50$.
 This is partly due to the fact that PBT is the first FPT algorithm for WSP with UI constraints that has polynomial space complexity.
 We believe that an important lesson is that having found an FPT algorithm for a problem should be just a starting point for designing a practical algorithm as there are still likely to be many opportunities for significant improvement via the repertoire of intelligent heuristic search mechanisms.

 The direct study of algorithms for the WSP was also complemented with a study of PTs arising from a generator of WSP instances.
 We found strong evidence of PT phenomena in the same fashion as previous extensive studies within graph theory and AI\@.

 While typical PT studies deal with a single size parameter, FPT implies that there are two `size parameters', both playing important role in the complexity of the problem.
 This paper gives a novel combination of studying FPT and PT and the use of multiple `slices' of the size space for a thorough empirical study of scaling of algorithms.

A common approach to solving decision problems in practical computing is to formulate them using a general-purpose declarative language and then using off-the-shelf solvers.
 Naturally, one may be interested in representations, with which off-the-shelf general solvers result in similar scaling to the direct implementations of FPT algorithms.  
 We did observe such good behaviour of our new PBPB formulation solved by SAT4J (in resolution proof system mode) on the ``vary-$k$'' slice, in which the scaling was roughly as good as the PBT solver, and a lot better than that of the previous solvers.
 A very similar performance was observed from a CSP solver even though our CSP formulation did not explicitly exploit our understanding of the FPT nature of the problem.
 We expect that the internal mechanism of handling our CSP formulation by CP-SAT in fact leads to an internal search process similar to pattern-based methods of PBPB and the overall winner PBT.

 Although the running time of the general-purpose solvers was by one or two orders of magnitude worse than that of PBT, as might be expected, it indicated that the solvers were able to determine a good search strategy. 
 Furthermore, one may expect some additional efficiencies of the general purpose solvers when applied to real instances, as an off-the-shelf solver is more likely to be able to exploit associated structures.
 
The behaviour of these off-the-shelf solvers was not as good in the ``fixed-$k$'' slice experiments compared to PBT\@.
 The PBPB (Res) solver demonstrated, apparently, exponential scaling despite us showing in Proposition~\ref{thm:PB-is-FPT} that a simple tree-based solver is capable of exhibiting FPT time, i.e.\ polynomial scaling in ``fixed-$k$'' slice.
 Also, the memory usage by the general-purpose solver was a bottleneck; e.g.\ CSP~(CP-SAT) consumed around 1.5-2~GB of RAM at $k = 18$ and $n \approx 2.4 \cdot 10^6$.
 In comparison, for such a problem, PBT takes around 1~MB of RAM on top of the size of instance data itself (about 20~MB).
 This indicates that there is a good potential for improvement of current off-the-shelf solvers to enable them to take advantage of FPT properties.

%%%%%%%%%%%%%%%%%%%%%%%%%%%%%%%%%%%%%%%%%%
\subsection{Future Directions}  
\label{sec:future}

 A natural direction for future research is to further improve the performance of the PBT algorithm.
 One can investigate improving the pruning from the authorisations by adding extra lookahead; further improving the branching heuristics (possibly by exploiting machine learning for adaptive search) and also finding branch ordering selections that give a net gain on the satisfiable instances.
 A particularly important direction arises from noting that PB solvers on the PBPB encoding are likely to be benefiting from learning and storing of no-goods (entailed constraints). 
 It would be interesting to consider how PBT could be enhanced with such no-good learning, and in such a way that is compatible with FPT -- enhancing the FPT-driven two-layer nature of the solution, rather than breaking it.
 The PT properties of the set of generated instances should also be mapped in more detail, along with consideration of a wider range of UI constraints. 
 Such an enhanced study of the PT properties should be also exploited for further evaluation of proposed algorithms.
 
 Although the combination of PBPB and SAT4J in the resolution proof system mode worked well, there was evidence from the ``fixed-$k$'' that in some regions of the space of instances it was performing less well.
 In particular, scaling with $n$ seemed to be non-polynomial, and this deserves further investigation.
 Possibly, general PB solution methods need different branching heuristics, or could be extended to better exploit the matching problem (equivalently, list colouring of a clique) that arises as a vital sub-problem when using the PBPB formulation.
We emphasise that here we do not attempt to systematically cover `all encodings', but to take some representative and natural ones and study their behaviour from an FPT perspective. However, we believe the methodology and initial results give a foundation for a future systematic study of combinations of encodings and solvers.

 We propose this study as a contribution towards ensuring that general purpose solvers are appropriately effective on FPT problems; and give an example of developing a formulation that enables solvers to exploit the inherent FPT properties of the problem.  
 A future challenge in AI may be to study how a general-purpose solver can automatically discover such formulations. 

 An important outcome is that the combination of decomposition and FPT ideas leads to new highly-effective algorithm, and then combining FPT with PT ideas give a powerful framework for empirical study.
  Our PT study of WSP also revealed interesting challenges in empirical average time complexity studies for FPT problems.
 We proposed using multiple slices through the size space while adjusting other parameters of the instance generator to stay in the PT region.
 However, it is still an open question how to best select these slices, or indeed how to do a more integrated study of the effects of the multiple parameters.
 Considering the general interest of the AI community in analysis of a widening range of complexity classes (e.g.~\shortcite{BaileyEtAl2007:PT-in-PP}) (including studies of FPT problems, see e.g.~\shortcite{DeHaan:2015:IJCAI:FPT-planning,HaanEtal2013:IJCAI-FPT-planning}) and understanding of practical implications of these complexity classes, we believe that further development of the study of interactions between FPT and PT offers the potential for deeper insight into computational challenges arising in AI.

%%%%%%%%%%%%%%%%%%%%%%%%%%%%%%%%%%%%%%%%%%%%%%%%%%%%%%%%%%%%%
\bigskip
\paragraph{Acknowledgements} We would like to thank the reviewers for the valuable comments that helped us to improve the paper.
 This research was partially supported by EPSRC grants EP/H000968/1 (for DK and AP), EP/K005162/1 (for AG and GG), and a Leverhulme Trust grant (GG).
 The source codes of the instance generator and PBT, as well as the test instances, solutions, translation routines and the experimental data, are publicly available~\cite{SourceCodes}.

\ifdefined\excludeappendix
\else
\clearpage

%%%%%%%%%%%%%%%%%%%%%%%%%%%%%%
\appendix
\section{Encoding Constraints} 
\label{ap:encoding-constraints}

 As noted in Section~\ref{sec:new-formulation}, any UI constraint can be formulated in terms of $M$ variables only.
 However, the straightforward approach adopting a list of all obeying or disobeying patterns may result into encoding of exponential size in the size of the constraint scope.
 Here we show some more compact encodings of several standard UI constraints.
 Let $c$ be the constraint to be encoded and $T_c = \set{s_1, s_2, \ldots, s_q}$.

\subsection{Easy Cases}

 For completeness, below we give a list of constraints that are easy to compactly encode with the $M$ variables.

\begin{itemize}
	\item 
    Not-equals, also called separation of duty: $M_{s_1, s_2} = 0$.
    
    \item
    Equals, also called binding of duty: $M_{s_1, s_2} = 1$.
    
    \item
    All-different, or at-least-$q$-out-of-$q$: $M_{s_i, s_j} = 0$ for every $1 \le i < j \le q$.
    
    \item
    Not-all-different, or at-most-$(q - 1)$-out-of-$q$: $\sum_{i = 1}^{q - 1} \sum_{j = i + 1}^{q} M_{s_i, s_j} \ge 1$.
\end{itemize}

%%%%%%%%%%%%%%%%%%%%%%%%%%%%%%%%%%%%%%%%%%%%%
\subsection{At-least and At-most Constraints with Scope Size up to Five}
\label{ap:counting-constraints-up-to-five}
 Let $G = (T_c, E)$ be a graph with vertex set $T_c$ and edges $E = \{ (s_i, s_j) :\; M_{s_i, s_j} = 1 \}$.
 Observe that $G$ uniquely represents a pattern on step set $T_c$ as defined by appropriate variables $M$'s; it consists of cliques only, with each clique corresponding to a block of the pattern.
 Recall that at-least and at-most constraints are restricting the number of distinct users to be assigned to the scope; thus they are step-symmetric, i.e.\ satisfiability of a single at-least/at-most constraint does not depend on the permutation of steps in its scope.
 In particular, the number of users assigned to the scope $T_c$ is exactly the number of cliques in $G$, and this number of cliques can often be determined by simple counting of edges in $G$.
 Figure~\ref{fig:counting-edges} shows all patterns (subject to step permutations) on scope of size 5 and gives the possible number of edges.

\newcommand{\scopegraph}[1][]{%
  	\begin{tikzpicture}[
		%>=stealth', 
		%shorten >=1pt, 
		thick,
		vertex/.style={circle, draw}]
		
		\foreach \i in {0,1,...,4}{
			\pgfmathtruncatemacro{\colour}{\colour[\i]}
			\node[vertex] (\i) at (\i*360/5: 1cm) {$u_{\colour}$};
		}
		
		\foreach \i in {0,...,3}{
			\pgfmathtruncatemacro{\icolour}{\colour[\i]};
			%\pgfmathtruncatemacro{\i1}{\i + 1};
			\foreach \j in {\i,...,4}{
				\pgfmathtruncatemacro{\jcolour}{\colour[\j]}
				\ifthenelse{\equal{\icolour}{\jcolour} \AND \j > \i}{\path[ultra thick, black] (\i) edge (\j);}{};
			}
		}

        %		edge [bend right] node[left] {0.3} (2)
		%		edge [loop above] node {0.1} (1);
	\end{tikzpicture}
}

\begin{figure}[htb]
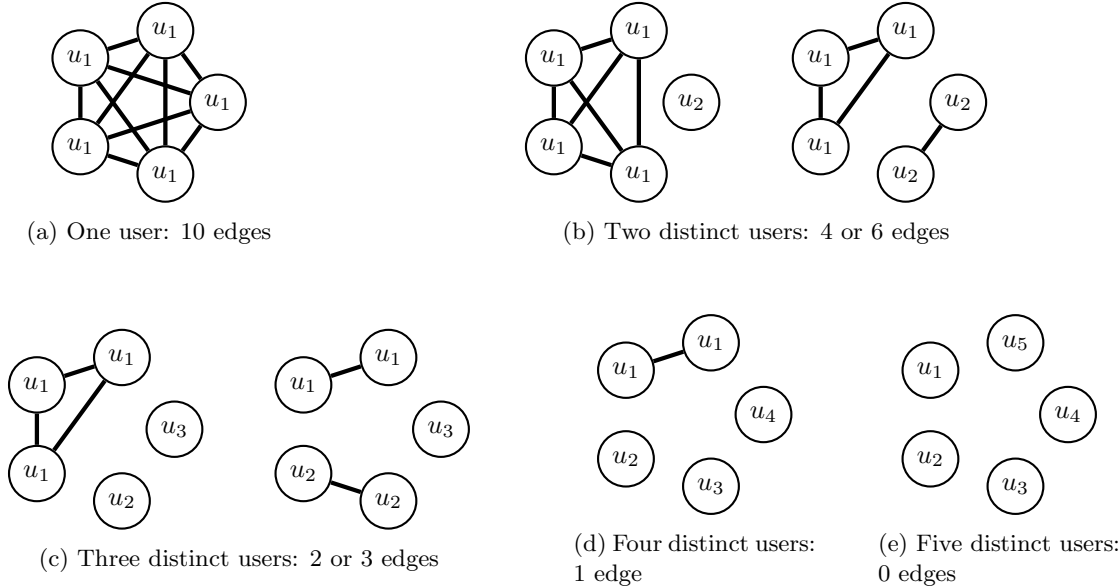

\centering
\begin{subfigure}{0.3\textwidth}
\centering
\def\colour{{1,1,1,1,1}} \scopegraph

\caption{One user: 10 edges}
\end{subfigure}
\qquad
\begin{subfigure}{0.6\textwidth}
\centering
\def\colour{{2,1,1,1,1}} \scopegraph
\qquad
\def\colour{{2,1,1,1,2}} \scopegraph

\caption{Two distinct users: 4 or 6 edges}
\end{subfigure}

\bigskip
\bigskip

\begin{subfigure}{0.45\textwidth}
\centering
\def\colour{{3,1,1,1,2}} \scopegraph
\qquad
\def\colour{{3,1,1,2,2}} \scopegraph

\caption{Three distinct users: 2 or 3 edges}
\end{subfigure}
\qquad
\begin{subfigure}{0.2\textwidth}
\centering
\def\colour{{4,1,1,2,3}} \scopegraph

\caption{Four distinct users: 1 edge}
\end{subfigure}
\qquad
\begin{subfigure}{0.2\textwidth}
\centering
\def\colour{{4,5,1,2,3}} \scopegraph

\caption{Five distinct users: 0 edges}
\end{subfigure}

\caption{
	Graphs $G$ illustrating of the user assignments within a scope of five steps.
    There are $B_5 = 52$ patterns with scope of size 5, however those differing only by a permutation of the steps are not given as they do not change the number of edges.
	}
\label{fig:counting-edges}
\end{figure}

\newcolumntype{L}[1]{>{\raggedright\let\newline\\\arraybackslash\hspace{0pt}}m{#1}}
\newcolumntype{C}[1]{>{\centering\let\newline\\\arraybackslash\hspace{0pt}}m{#1}}
\newcolumntype{R}[1]{>{\raggedleft\let\newline\\\arraybackslash\hspace{0pt}}m{#1}}
 
\begin{table}[htb]
 \centering
 \begin{tabular}{@{} c *{6}{C{1cm}} @{}}
 	\toprule
 	& \multicolumn{6}{c}{\# distinct users} \\
 	\cmidrule(l){2-7}
 	$|T_c| = q$	& 1 	& 2		& 3		& 4		& 5		& 6 \\
 	\midrule
 	2		& 1 	& 0		& --	& --	& --	& -- \\
 	3		& 3 	& 1		& 0		& --	& --	& -- \\
 	4		& 6 	& 2--3	& 1		& 0		& --	& -- \\
 	5		& 10	& 4--6	& 2--3	& 1		& 0		& -- \\
 	6		& 15	& \phantom{0}6--10	& 3--6	& 2--3	& 1		& 0 \\
 	\bottomrule
 \end{tabular}
 \caption{
 	This table gives the bounds $\underline{\sigma_{q, r}} \le |E| \le \overline{\sigma_{q, r}}$ for the number $|E|$ of edges in a graph $G$ for each scope size $q$ and number of distinct users.
    One can observe that counting the number of edges is sufficient to define at-least and at-most constraints with scope size up to 5.
 }
 \label{tab:sum-ranges}
\end{table}
 
 Table~\ref{tab:sum-ranges} gives the possible number of edges in $G$ for various scope sizes $q$ and numbers of distinct assigned users.
 One can see that for any $q \le 5$ the ranges do not overlap, indicating that it is possible to determine the number of distinct users assigned to $T_c$ by counting edges in $G$.
 However, at $q = 6$, some of the ranges overlap, and then not every at-least and at-most constraint with $q \ge 6$ can be encoded in this way.
 
 For $q \le 5$ and $1 \le r \le q$, let $\underline{\sigma}_{q, r}$ and $\overline{\sigma}_{q, r}$ be the lower and upper bounds, respectively, on the number of edges in $G$ with $q$ nodes and exactly $r$ cliques.
 Then we can formulate an at-most-$r$ constraint as
\begin{equation}
    \sum_{i = 1}^{q - 1} \sum_{j = i + 1}^{q} M_{s_i, s_j} \ge \underline{\sigma}_{q, r},
\end{equation}
and an at-least-$r$ constraint as
\begin{equation}
    \sum_{i = 1}^{|T| - 1} \sum_{j = i + 1}^{|T|} M_{s_i, s_j} \le \overline{\sigma}_{q, r}.
\end{equation}

 Encoding of at-least and at-most constraints with scope sizes above five is considered in the next section.

\subsection{Other Constraints}

 Here we give some other encodings that, among others, allow us to formulate at-least and at-most constraints with scope size above 5 and some other standard UI constraints.
 
 Observe that existence of at least $t$ cliques in $G$ (see \ref{ap:counting-constraints-up-to-five}) means that there is at least one independent set $T' \subset T_c$ of size $t$.
 Thus, to encode an at-most-$r$ constraint, we can request that there is no independent set of size $r + 1$:
\begin{equation}
	\sum_{s' < s'' \in T'} M_{s', s''} \ge 1 \text{ for every } T' \subset T_c,\ |T'| = r + 1.
\end{equation}
 This encoding requires $\binom{q}{r + 1}$ constraints and no new variables.
 
 Now consider a permutation $\sigma$ of $T_c$, and let $\sigma(q + 1) = \sigma(1)$.
 Observe that 
$$
	|\set{ i :\; i \in \set{1, 2, \ldots, q} \text{ and } (\sigma(i), \sigma(i + 1)) \notin E}| \ge t,
$$ 
where $t$ is the number of cliques in $G$.
 Hence, the following encodes an at-least-$r$ constraint:
 \begin{equation}
	\sum_{i = 1}^{q} (1 - M_{s_{\pi(i)}, s_{\pi(i + 1)}}) \ge r \text{ for every permutation $\pi$ of $T_c$.}
 \end{equation}
 Some of the constraints will be identical due to the step symmetry; as a result, we will need $(q - 1)! / 2$ constraints, and no new variables.

 \bigskip

 We also propose more compact encodings that involve creation of new variables.
 Let $G_i$ be a subgraph of $G$ induced by a node set $\set{s_1, s_2, \ldots, s_i}$; observe that $G_i$ also consists of cliques only.
 Let $t_i$ be the difference in the number of cliques between $G_i$ and $G_{i - 1}$ for $i = 2, 3, \ldots, q$, and $t_1 = 1$.
 Then the number of cliques in $G = G_q$ can be computed as $\sum_{i = 1}^q t_i$.
 
 Using variables $t$, we can encode an at-least-$r$ constraint with an arbitrary scope as follows:
\begin{align}
	& t_1 = 1,\\
	& t_i \le 1 - M_{s_j, s_i} \ \forall j < i,\ i = 2, 3, \ldots, q \text{ and}\\
    & \sum_{i = 1}^q t_i \ge r.
\end{align}
 This encoding takes $O(q^2)$ constraints and $O(q)$ new variables.

 Similarly we can encode an at-most-$r$ constraint with an arbitrary scope:
\begin{align}
	& t_1 = 1,\\
	& t_i \ge \sum_{j = 1}^{i - 1} (1 - M_{s_j, s_i}) - (i - 2), \ i = 2, 3, \ldots, q \text{ and}\\
    & \sum_{i = 1}^q t_i \le r.
\end{align}
 This encoding takes $O(q)$ constraints and $O(q)$ new variables.

\bigskip

 Another standard UI constraint is the \emph{generalised threshold constraint} $(t_l, t_r, T_c)$ which restricts the number of steps assigned to a user involved in execution of steps $T_c$~\cite{CrGuYe13}.
 More formally, each user assigned to at least one step $s \in T_c$ is required to execute between $t_l$ and $t_r$ steps in $T_c$.
 Observe that this constraint can be enforced by restricting the size of cliques in $G$ between $t_l$ and $t_r$:
 \begin{equation}
    t_l \le \sum_{j = 1}^{q} M_{s_i, s_j} \le t_r \text{ for } s_i \in T_c \,.
 \end{equation}

%%%%%%%%%%%%%%%%%%%%%%%%%%%%%%%%%%%%%%%%%%%%%%%%%%%%%%%%%%%%%%%%%%%%%%%%%%%%%%%%%%%%%%%%%%%%%
\section{Estimation of the Critical Point} 
\label{ap:estimated-PT}

 Here we give an (approximate) computation in the style of an `annealed estimate' of the average number of solutions given $k$, $n$, $\gamma$ and $e$; with the intent to use it to obtain indications of the location of the PT points.
 Since many of the probabilities depend on the number $b$ of blocks in the solution, we first compute the estimate for a given $b$ and then aggregate the results.
 
 The Stirling number $c(k,b)$ of the second kind is the number of ways to partition a set of $k$ labelled objects into $b$ nonempty unlabelled subsets. Equivalently, $c(k,b)$ is the number of different equivalence relations with precisely $b$ equivalence classes that can be defined on a $k$-element set.
 
 The number of patterns with $b$ blocks is exactly the Stirling number $c(k, b)$ of the second kind, and there are exactly $P(n, b) = \frac{n!}{(n - b)!}$ plans implementing a pattern with $b$ blocks.
 Consider a scope $T$ of size $|T| = q$.
 Let $p(q, r)$ be the probability of that $T$ intersect with exactly $r$ distinct pattern blocks.
 There are $b^q$ ways to assign blocks within $T$, and $c(q, r) \cdot P(b, r)$ ways to assign exactly $r$ distinct blocks.
 Hence, $p(q, r) = \frac{c(q, r) \cdot P(b, r)}{b^q}$.
 Then the probability that a random pattern (or plan, which is the same in this context) satisfies an at-most-$r$ constraint is $\sum_{r' = 1}^r p(q, r')$.
 The probability that a random pattern satisfies an at-least-$r$ constraint is $1 - \sum_{r' = 1}^{r - 1} p(q, r')$.
 Not-equals is a special case of at-least-$r$ with $q = 2$ and $r = 2$; hence, the probability that a not-equals hits a random pattern is $1 - p(2, 1) = 1 - \frac{1}{b}$ as one could predict.
 We conclude that the number $N^\text{elig}_\text{pat}(b)$ of eligible patterns with $b$ blocks is on average 
 \begin{equation}
	\label{eq:eligible-patterns}
	N^\text{elig}_\text{pat}(b) = c(k, b) \cdot \frac{1}{b^e} \cdot \left( \left( \sum_{r' = 1}^r p(q, r') \right) \cdot \left( 1 - \sum_{r' = 1}^{r - 1} p(q, r') \right) \right)^\gamma \,.
 \end{equation}

 \begin{equation}
	\label{eq:eligible-patterns-all}
	N^\text{elig}_\text{pat} = \sum_b N^\text{elig}_\text{pat}(b) 
 \end{equation}

 When \ref{eq:eligible-patterns-all} gives $N^\text{elig}_\text{pat} < 0.5$ then at least half the instances have no eligible pattern and so must be unsatisfiable, and so this gives an upper bound on the location of the PT. 
 However, it is relatively weak, and so we also need, unsurprisingly, to also take account of the authorisations.

 The probability of a random plan being authorised is $p^\text{auth}(b) = (1/4)^k$ as the probability of a single step being authorised is $1/4$.
 Then the number of valid plans, on average, is 
 \begin{equation}
	\label{eq:num-valid-plans}
   N^\text{valid}_\text{plans} = \sum_{b = 3}^k N^\text{elig}_\text{pat}(b) \cdot P(n, b) \cdot p^\text{auth}(b) \,.
 \end{equation}
 (Note that there are no eligible patterns with $b < 3$ due to at-least-3 constraints.) 
 We report the number $N^\text{elig}_\text{pat}(b)$ of eligible patterns, number $P(n, b)  \cdot p^\text{auth}(b)$ of authorised plans per pattern and the average number $N^\text{valid}_\text{plans}$ of valid plans in Table~\ref{tab:annealed-estimate}.
 Observe that $N^\text{valid}_\text{plans}$ is often well above 1, especially for large $b$, but this is mainly because of the huge number of authorised plans per pattern; the expected number of eligible patterns for large $b$ is negligible.
 This suggests that the distribution of $N^\text{valid}_\text{plans}$ is highly multimodal; averaged over the instances of the given parameters, the number of valid plans will be large but the majority of instances will have no valid plans at all.
 Nonetheless, \ref{eq:num-valid-plans} can be used to give an upper bound on the location of the PT because when the average number drops to 0.5, then, independently of the distribution, at least one-half of the instances must be unsatisfiable.
 
 Thus, the average number of valid plans might not be a practical indicator of the probability of instance satisfiability; to establish the PT parameters, we suggest a different technique.
 In particular, we exploit the two-layer nature of the problem.
 We estimate the probability $p^\text{auth}_\text{pat}(b)$ that a pattern with $b$ blocks is authorised, i.e.\ there exists at least one plan realising it, and then use it to compute the probability $p^\text{sat}(b)$ that there exists at least one valid plan.
 
 To obtain a rough estimate of $p^\text{auth}_\text{pat}(b)$ (which is then not an upper bound), we assume that all the blocks in the pattern are of the same size $k / b$.
 This might be justified as reasonable because
 \begin{enumerate}
     \item
     There are more patterns with roughly equal size blocks -- by standard counting;
 
 	 \item
     Authorisations are harder to satisfy for larger blocks and the increase is likely be super-linear in the block size -- and this will tend to promote solutions that have a fair distribution of block sizes.
 \end{enumerate}
 
 The probability of a single block being authorised by a random user is $(1/4)^{k / b}$, and the probability that there is at least one user authorised to a given block is $1 - \left(1 - \left(\frac{1}{4}\right)^{k / b}\right)^n$.
 To simplify calculations, we also relax the requirement that distinct blocks need to be assigned to distinct users, and conclude that
 \begin{equation}
   \label{eq:prob-authorised}
   p^\text{auth}_\text{pat}(b) = \left( 1 - \left[1 - \left(\frac{1}{4}\right)^{k / b}\right]^n \right)^b \,.
 \end{equation}

 Using (\ref{eq:prob-authorised}) we compute $p^\text{sat}(b)$ as follows:
 \begin{equation}
 \label{eq:prob-sat-b}
 p^\text{sat}(b) = \begin{cases}
   		N^\text{elig}_\text{pat}(b) \cdot p^\text{auth}(b) & \text{if } N^\text{elig}_\text{pat}(b) < 1, \\
   		1 - \left(1 - p^\text{auth}(b)\right)^{N^\text{elig}_\text{pat}(b)} & \text{if } N^\text{elig}_\text{pat}(b) \ge 1. \\     
   \end{cases}
 \end{equation}
 Then the probability $p^\text{sat}$ of existence of at least one valid pattern is $1 - \prod_{b = 3}^k 1 - p^\text{sat}(b)$, and the PT region can be established by finding parameters, using \ref{eq:prob-sat-b}, leading to $p^\text{sat} = 0.5$.

 Observe that $p^\text{sat}(b)$ is very low at high $b$ (see Table~\ref{tab:annealed-estimate}) reflecting the fact that there is a very low probability of existence of a valid plan with many blocks.
 In fact, Table~\ref{tab:annealed-estimate} shows that the number of blocks in a valid pattern is tightly bounded -- hence we expect the number of users in a valid plan being usually forced by the instance.

 Moreover, we can see that the total number of eligible patterns is relatively low at PT and, hence, the complexity of the problem is driven by constraints and not authorisations.
 This is reflected in our overall strategy focusing on constraints (by means of patterns) and considering authorisations as a secondary component.

 Another interesting observation is that the at-least-3 constraints are relatively weak while at-most-3 constraints significantly affect the probability of a pattern being eligible.
 This fact is exploited by our branching heuristic, see Section~\ref{sec:branching-heuristic}.
  
 \begin{table}
 \small
\begin{tabular}{@{} r *{4}{p{4.5em}} p{5em} *{3}{p{4.5em}} @{}}
\toprule
 & \multicolumn{3}{c}{Prob.\ per constraint} & \multicolumn{3}{c}{Number of} & \multicolumn{2}{c}{Probability of} \\
\cmidrule(lr){2-4}
\cmidrule(lr){5-7}
\cmidrule(l){8-9}
 & $\neq$ & $\le$ & $\ge$ & eligible patterns & auth.\ plans per pattern & valid plans & pattern is auth. & valid pat.\ exists \\
\midrule
$b$ & & & & $N^\text{elig}_\text{pat}(b)$ & $P(b, r)$ & $N^\text{valid}_\text{plans}$ & $p^\text{auth}_\text{pat}(b)$ & $p^\text{sat}(b)$ \\
\midrule
3 & $6.7 \cdot 10^{-1}$ & $1.0 \cdot 10^{0}$ & $6.2 \cdot 10^{-1}$ & $2.8 \cdot 10^{-2}$ & $2.3 \cdot 10^{-11}$ & $6.5 \cdot 10^{-13}$ & $2.3 \cdot 10^{-11}$ & $6.5 \cdot 10^{-13}$ \\
4 & $7.5 \cdot 10^{-1}$ & $7.7 \cdot 10^{-1}$ & $8.2 \cdot 10^{-1}$ & $2.4 \cdot 10^{4}$ & $6.9 \cdot 10^{-9}$ & $1.6 \cdot 10^{-4}$ & $6.9 \cdot 10^{-9}$ & $1.6 \cdot 10^{-4}$ \\
5 & $8.0 \cdot 10^{-1}$ & $5.8 \cdot 10^{-1}$ & $9.0 \cdot 10^{-1}$ & $3.6 \cdot 10^{5}$ & $2.0 \cdot 10^{-6}$ & $7.3 \cdot 10^{-1}$ & $1.8 \cdot 10^{-6}$ & $4.7 \cdot 10^{-1}$ \\
6 & $8.3 \cdot 10^{-1}$ & $4.4 \cdot 10^{-1}$ & $9.4 \cdot 10^{-1}$ & $1.5 \cdot 10^{5}$ & $6.0 \cdot 10^{-4}$ & $8.8 \cdot 10^{1}$ & $2.7 \cdot 10^{-4}$ & $1.0 \cdot 10^{0}$ \\
7 & $8.6 \cdot 10^{-1}$ & $3.5 \cdot 10^{-1}$ & $9.6 \cdot 10^{-1}$ & $1.3 \cdot 10^{4}$ & $1.8 \cdot 10^{-1}$ & $2.2 \cdot 10^{3}$ & $1.4 \cdot 10^{-2}$ & $1.0 \cdot 10^{0}$ \\
8 & $8.8 \cdot 10^{-1}$ & $2.8 \cdot 10^{-1}$ & $9.7 \cdot 10^{-1}$ & $5.0 \cdot 10^{2}$ & $5.2 \cdot 10^{1}$ & $2.6 \cdot 10^{4}$ & $1.9 \cdot 10^{-1}$ & $1.0 \cdot 10^{0}$ \\
9 & $8.9 \cdot 10^{-1}$ & $2.3 \cdot 10^{-1}$ & $9.8 \cdot 10^{-1}$ & $1.3 \cdot 10^{1}$ & $1.5 \cdot 10^{4}$ & $1.9 \cdot 10^{5}$ & $6.2 \cdot 10^{-1}$ & $1.0 \cdot 10^{0}$ \\
10 & $9.0 \cdot 10^{-1}$ & $1.9 \cdot 10^{-1}$ & $9.9 \cdot 10^{-1}$ & $2.4 \cdot 10^{-1}$ & $4.4 \cdot 10^{6}$ & $1.1 \cdot 10^{6}$ & $9.1 \cdot 10^{-1}$ & $2.2 \cdot 10^{-1}$ \\
11 & $9.1 \cdot 10^{-1}$ & $1.6 \cdot 10^{-1}$ & $9.9 \cdot 10^{-1}$ & $3.7 \cdot 10^{-3}$ & $1.3 \cdot 10^{9}$ & $4.8 \cdot 10^{6}$ & $9.9 \cdot 10^{-1}$ & $3.7 \cdot 10^{-3}$ \\
12 & $9.2 \cdot 10^{-1}$ & $1.4 \cdot 10^{-1}$ & $9.9 \cdot 10^{-1}$ & $5.0 \cdot 10^{-5}$ & $3.7 \cdot 10^{11}$ & $1.8 \cdot 10^{7}$ & $1.0 \cdot 10^{0}$ & $5.0 \cdot 10^{-5}$ \\
\multicolumn{9}{c}{$\cdots$}\\
30 & $9.7 \cdot 10^{-1}$ & $2.6 \cdot 10^{-2}$ & $1.0 \cdot 10^{0}$ & $3.2 \cdot 10^{-49}$ & $4.0 \cdot 10^{55}$ & $1.3 \cdot 10^{7}$ & $1.0 \cdot 10^{0}$ & $3.2 \cdot 10^{-49}$ \\
\bottomrule
\end{tabular}
 \caption{
 	Computational analysis of random WSP instances; $k = 30$, $\gamma = k$, $e = e_\text{PT} = 50$, $n = 10k$.
    $b$ is the number of blocks in the solution.
    Columns 2 to 4 show probabilities of a random plan with the given number of blocks satisfying corresponding constraint.
    Using the estimates of the number of eligible patterns and authorised plans per pattern, we compute the average number of valid plans.
    We also estimate the probability of existence of at least one valid pattern.
 }
 \label{tab:annealed-estimate}
 \end{table}

 Our experiments show that the estimate of $p^\text{sat}(b)$ is relatively accurate; for example, for $k = 30$ it predicts PT at $\beta = 1.17$ (for the definition of $\beta$ see Section~\ref{sec:vary-k-slice}), and for $k = 50$ at $\beta = 1.02$, see Figure~\ref{fig:annealed-estimate}.
 Hence, our formulas can be used to quickly predict if certain parameters are likely to result in sat or in unsat instances.
 
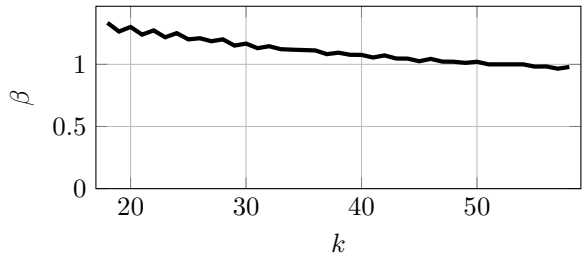
\begin{figure}[htb]
	\centering
     \begin{tikzpicture}
        \begin{axis}[		
            compat=newest,
            width=0.5\textwidth,
            height=4cm,
            legend pos=north west,
            xlabel={$k$},
            ylabel={$\beta$},
            title={},
            %legend cell align=left,
            grid=major,
            ymin=0,
            xmin=17,
            xmax=59
            %clip bounding box=upper bound
        ]       
        \addplot[ultra thick] table[
            x=k,
            y=beta
        ] {annealed-estimate.dat};
	 	\end{axis}
     \end{tikzpicture}
     \caption{
     	Predicted values of $\beta$ by our improved annealed estimate method.
        Recall that, by definition, the correct value of $\beta$ is 1 for any $k$.}
     \label{fig:annealed-estimate}
\end{figure}

%%%%%%%%%%%%%%%%%%%%%%%%%%%%%%%%%%%%%%%%%%%%%%%%
\section{Emergence of Forced Variables} 
\label{ap:PT-forced}

 A particularly interesting aspect of PTs is the emergence of forced variables in the critical region. 
 That is, variables that must have some particular value in all solutions -- and so are entailed by the system.
 This has been extensively studied in the context of Random 3SAT and Graph Colouring Problem (e.g.\ \shortcite{Culberson2001:frozen-graph-coloring,KilbyEtal2005:backbones,parkes1997clustering,mezard2005clustering}).
 In the context of standard graph colouring, the permutation symmetry means no vertex can be forced to have a particular colour. 
 Hence, the forcing is instead considered in terms of whether different vertices are forced to have the same colour, or else forced to have different colours. 
 In other words, an instance might imply separation-of-duty and/or binding-of-duty constraints that are not explicitly listed in its description.
 This directly corresponds to whether $M$ variables become forced. 
 
 We have performed experiments to determine this empirically in the WSP instances in the region of the PT\@.
 Determining whether a particular value of $M$ is forced can be done directly by adding the negation to the instance and then testing for unsatisfiability.

\newcommand{\histogramforced}[3]{
     \begin{tikzpicture}
        \begin{axis}[		
            compat=newest,
            width=0.33\textwidth,
            height=3cm,
            legend pos=north west,
            title={},
            %legend cell align=left,
            grid=none,
        	ymin=0,
            ymax=0.25,
            xmin=0,
            xmax=#3
            %clip bounding box=upper bound
        ]       
        \addplot[
        	fill=blue!50,
        	hist={
            	density,
	        	bins=10,
	    	    %data min=0.5,
    	    	%data max=4
	    	}
        ] table[
            y={Forced-#2},
        ] {forced-constraints-20-#1.dat};
	 	\end{axis}
     \end{tikzpicture}
}

\newcommand{\equalshistright}{50}
\newcommand{\notequalshistright}{150}

\begin{table}[htb]
	\centering
	\begin{tabular}{
            @{}
            >{\raggedleft\arraybackslash} m{0.04\textwidth} 
            >{\raggedleft\arraybackslash} m{0.08\textwidth} 
            >{\raggedleft\arraybackslash} m{0.08\textwidth} 
            >{\centering\arraybackslash} m{0.32\textwidth} 
            >{\centering\arraybackslash} m{0.32\textwidth}
            @{}
        }
    	\toprule    	
    	$\beta$
        	& Avg.\ $=$
        	& Avg.\ $\neq$
        	& $=$
        	& $\neq$ \\
        \midrule
		%0.5 
        %	& 0 (0)
        %    & 0 (0)
        %    & \histogramforced{0.5}{equals}{\equalshistright}
        %    & \histogramforced{0.5}{not-equals}{\notequalshistright} \\
		0.8 
        	& 0% & 0
            & 12% & 4
            & \histogramforced{0.8}{equals}{\equalshistright}
            & \histogramforced{0.8}{not-equals}{\notequalshistright} \\
		1.0 
        	& 16% & 11
            & 91% & 96
            & \histogramforced{1}{equals}{\equalshistright}
            & \histogramforced{1}{not-equals}{\notequalshistright} \\
		1.2 
        	& 31% & 33
            & 112% & 115
        	& \histogramforced{1.2}{equals}{\equalshistright}
            & \histogramforced{1.2}{not-equals}{\notequalshistright} \\
% 		1.5 \AJP{DROP THIS?} \DXK{Drop what?}
%         	& 36 & 37
%             & 106 & 106
%             & \histogramforced{1.5}{equals}{\equalshistright}
%             & \histogramforced{1.5}{not-equals}{\notequalshistright} \\
        \bottomrule
    \end{tabular}
    \caption{Forced constraints. 
    	In this experiment, $k = 20$ and the total number of $M$-variables is 190 (recall that $M_{ij} = M_{ji}$; we count such a pair as one variable).
        Constraints explicitly defined by the instance are not counted as forced.
    }
    \label{tab:forced-constraints}
\end{table}

 The results for the WSP are summarised in Table~\ref{tab:forced-constraints}.
 Note that there is a difference between a forced binding-of-duty (equals) constraint, denoted `$=$', and a forced separation-of-duty (not-equals) constraint, denoted `$\neq$'.
 Although a forced not-equals is more likely as there are usually more zeros than ones in the $M$ matrix, and as $M_{s_1, s_2} = 0$ is a weaker decision than $M_{s_1, s_2} = 1$, we still observe quite a lot of forced equals constraints. 
 We also see that as we move through the PT region from slightly under-constrained to slightly over-constrained, the number of forced $M$ variables increases rapidly. 
 Similar behaviour has been empirically seen in Random 3SAT (see e.g.~\cite{parkes1997clustering}).
 This is important in that it again gives evidence that the WSP instances are behaving like a PT would be expected to behave, and so can be expected to be a good and effective test of algorithms for the WSP.
 
 Note that we have only studied the freezing of the $M$ variables, but due to the user authorisations (or list-colouring) it is quite possible that other notions of freezing also emerge.
 We also believe it further supports that the PT is worthy of study in its own right.
  
 Another observation is that the information on forced variables (or constraints) could actually be useful to the users of the WSP decision support system.
 Indeed, knowing which of the constraints are forced might help the user to understand the instance and change it when necessary.
 A fast effective solver such as PBT can be used to produce such information, as well as specific solutions.

%\section{Anotehr vary-k $n = 100k$ slice} %% latex will not do maths in a section header?
\section{A supplementary ``vary-k'' slice}
\label{ap:100k}

 Figure~\ref{fig:100k} shows how the solvers perform on the $n = 100k$ slice (for the $n = 10k$ slice, see Figure~\ref{fig:slices}).
 
\begin{figure}[htb]
\begin{tikzpicture}[]
	\begin{semilogyaxis}[
		compat=newest,
		width=\textwidth,
		height=8cm,
		legend pos=south east,
		xlabel={Number of steps $k$},
		ylabel={Running time, sec},
		title={},
		legend cell align=left,
		grid=major,
        unbounded coords=jump,
        xmin=17,
        xmax=59
	]
	%\addplot[UDPB, sat, forget plot] table[
	%	x=k,
	%	y=sat4j-sat-50,
	%] {vary-k-100.dat};
	%\addplot[UDPB, unsat] table[
	%	x=k,
	%	y=sat4j-unsat-50,
	%] {vary-k-100.dat};
	%\addlegendentry{UDPB}

	\addplot[PBT, sat, forget plot] table[
		x=k,
		y=pbt-sat-50,
	] {vary-k-100.dat};
	\addplot[PBT, unsat] table[
		x=k,
		y=pbt-unsat-50,
	] {vary-k-100.dat};
	\addlegendentry{PBT}

	\addplot[PUI, sat, forget plot] table[
		x=k,
		y=ui-sat-50,
	] {vary-k-100.dat};
	\addplot[PUI, unsat] table[
		x=k,
		y=ui-unsat-50,
	] {vary-k-100.dat};
	\addlegendentry{PUI}

	\addplot[PBPB, sat, forget plot] table[
		x=k,
		y=mxsat-sat-50,
	] {vary-k-100.dat};
	\addplot[PBPB, unsat] table[
		x=k,
		y=mxsat-unsat-50,
	] {vary-k-100.dat};
	\addlegendentry{PBPB (Res)}

    %\addplot [UDPB, sat, forget plot] table [
    % 	x=k, 
    %    y=sat4j-sat-50
    %] {vary-k-100.dat};
    %\addplot [UDPB, unsat, mark=*] table [
    %   	x=k, 
    %    y=sat4j-unsat-50
    %] {vary-k-100.dat};
    %\addlegendentry{UDPB}

    \addplot [PBPB CP, sat, forget plot] table [
     	x=k, 
        y=mxpb-cp-sat-50
    ] {vary-k-100.dat};
    \addplot [PBPB CP, unsat, mark=*] table [
       	x=k, 
        y=mxpb-cp-unsat-50
    ] {vary-k-100.dat};
    \addlegendentry{PBPB (CutP)}

    \addplot [CSP, sat, forget plot] table [
     	x=k, 
        y=csp-sat-50
    ] {csp-vary-k-100.dat};
    \addplot [CSP, unsat, mark=*] table [
       	x=k, 
        y=csp-unsat-50
    ] {csp-vary-k-100.dat};
    \addlegendentry{CSP (CP-SAT)}

% 	\addplot[PUI, sat, forget plot] table[
% 		x=k,
% 		y=pbt-sat-50,
% 	] {vary-k-5.dat};
% 	\addplot[PUI, unsat] table[
% 		x=k,
% 		y=pbt-unsat-50,
% 	] {vary-k-5.dat};
% 	\addlegendentry{PBT n=5k}

%     \addplot [PUI, sat, forget plot] table [
%      	x=k, 
%         y=csp-sat-50
%     ] {vary-k-5.dat};
%     \addplot [PUI, unsat, mark=*] table [
%       	x=k, 
%         y=csp-unsat-50
%     ] {vary-k-5.dat};
%     \addlegendentry{CSP n=5k}
  
	%\addplot[PBPB, sat, forget plot] table[
	%	x=k,
	%	y=pbt-sat-50,
	%] {vary-k-2.dat};
	%\addplot[PBPB, unsat] table[
	%	x=k,
	%	y=pbt-unsat-50,
	%] {vary-k-2.dat};
	%\addlegendentry{PBT n=2k}
        
    %\addplot [PBPB, sat, forget plot] table [
    % 	x=k, 
    %    y=csp-sat-50
    %] {vary-k-2.dat};
    %\addplot [PBPB, unsat, mark=*] table [
    %   	x=k, 
    %    y=csp-unsat-50
    %] {vary-k-2.dat};
    %\addlegendentry{CSP n=2k}
    
	%\addplot[red, line width=3mm, opacity=0.2, domain=30:55] 
    %{
    %	1.5e-5 * 2^( x * log2(x)/12.2)
    %};
	%\addlegendentry{const $\cdot 2^{ k \cdot \log_2{k}/12.2}$}

\end{semilogyaxis}
\end{tikzpicture}
\caption{Performance of the solvers on the $n = 100k$ slice.}
\label{fig:100k}
\end{figure}
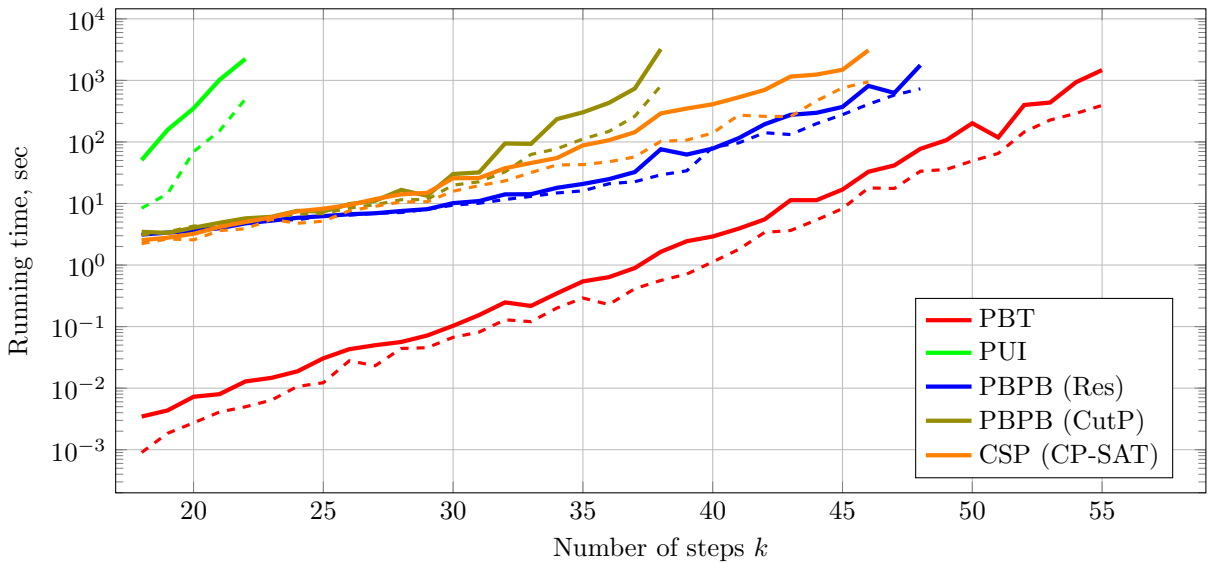

  The poor performance of UDPB~(Res) made it impractical to include it in this experiment; this is due to the non-FPT scaling of UDPB\@.
  The performance of most of the solvers in the $n = 100k$ slice closely matches that in the $n = 10k$ slice.
 On changing from $n=10k$ to $n=100k$, the performance of PBT barely changes, which is consistent with its low dependence on $n$ in the fixed-$k$ slice. 
 The two PBPB solvers became a little faster, and the CSP became a bit slower. \thepage
 PUI is slower by a constant factor of about 10 on the unsat instances, which is the expected result as it iterates over all the users.

\fi

%\section*{References}

%\bibliographystyle{apalike}
% \bibliographystyle{splncs03}
%\bibliographystyle{plain}

%\bibliographystyle{abbrv}

\vskip 0.2in
\bibliography{refs}

\begin{thebibliography}{}

\bibitem[\protect\BCAY{{American National Standards Institute}}{{American
  National Standards Institute}}{2004}]{ansi-rbac04}
{American National Standards Institute} \BBOP2004\BBCP.
\newblock \BBOQ American national standard for information technology -- {R}ole
  {B}ased {A}ccess {C}ontrol ({ANSI INCITS} 359-2004)\BBCQ.

\bibitem[\protect\BCAY{Bailey, Dalmau,\ \BBA\ Kolaitis}{Bailey
  et~al.}{2007}]{BaileyEtAl2007:PT-in-PP}
Bailey, D.~D., Dalmau, V., \BBA\ Kolaitis, P.~G. \BBOP2007\BBCP.
\newblock \BBOQ Phase transitions of {PP}-complete satisfiability
  problems\BBCQ\
\newblock {\Bem Discrete Applied Mathematics}, {\Bem 155\/}(12), 1627 -- 1639.

\bibitem[\protect\BCAY{Basin, Burri,\ \BBA\ Karjoth}{Basin
  et~al.}{2012}]{Basin2012}
Basin, D., Burri, S.~J., \BBA\ Karjoth, G. \BBOP2012\BBCP.
\newblock \BBOQ Optimal workflow-aware authorizations\BBCQ\
\newblock In {\Bem Proceedings of the 17th ACM Symposium on Access Control
  Models and Technologies}, SACMAT'12, \BPGS\ 93--102, New York, NY, USA. ACM.

\bibitem[\protect\BCAY{Basin, Burri,\ \BBA\ Karjoth}{Basin
  et~al.}{2014}]{BaBuKa14}
Basin, D., Burri, S.~J., \BBA\ Karjoth, G. \BBOP2014\BBCP.
\newblock \BBOQ Obstruction-free authorization enforcement: Aligning security
  and business objectives\BBCQ\
\newblock {\Bem Journal of Computer Security}, {\Bem 22\/}(5), 661--698.

\bibitem[\protect\BCAY{Bertino, Ferrari,\ \BBA\ Atluri}{Bertino
  et~al.}{1999}]{BeFeAt99}
Bertino, E., Ferrari, E., \BBA\ Atluri, V. \BBOP1999\BBCP.
\newblock \BBOQ The specification and enforcement of authorization constraints
  in workflow management systems\BBCQ\
\newblock {\Bem ACM Transactions on Information and System Security}, {\Bem
  2\/}(1), 65--104.

\bibitem[\protect\BCAY{Bertolissi, dos Santos,\ \BBA\ Ranise}{Bertolissi
  et~al.}{2018}]{Bertolissi2018}
Bertolissi, C., dos Santos, D.~R., \BBA\ Ranise, S. \BBOP2018\BBCP.
\newblock \BBOQ Solving multi-objective workflow satisfiability problems with
  optimization modulo theories techniques\BBCQ\
\newblock In {\Bem Proceedings of the 23d ACM Symposium on Access Control
  Models and Technologies}, SACMAT '18, \BPGS\ 117--128, New York, NY, USA.
  ACM.

\bibitem[\protect\BCAY{Bollobas}{Bollobas}{1985}]{Bollobas:book}
Bollobas, B. \BBOP1985\BBCP.
\newblock {\Bem Random Graphs}.
\newblock Academic Press, London, England.

\bibitem[\protect\BCAY{Boros\ \BBA\ Hammer}{Boros\ \BBA\
  Hammer}{2002}]{BorosHammer2002:PBO}
Boros, E.\BBACOMMA\  \BBA\ Hammer, P.~L. \BBOP2002\BBCP.
\newblock \BBOQ Pseudo-boolean optimization\BBCQ\
\newblock {\Bem Discrete Applied Mathematics}, {\Bem 123\/}(1-3), 155--225.

\bibitem[\protect\BCAY{Cameron}{Cameron}{1994}]{Cameron1994:Combinatorics-book}
Cameron, P.~J. \BBOP1994\BBCP.
\newblock {\Bem Combinatorics: Topics, Techniques, Algorithms}.
\newblock Cambridge University Press.

\bibitem[\protect\BCAY{Cheeseman, Kanefsky,\ \BBA\ Taylor}{Cheeseman
  et~al.}{1991}]{cheeseman91where}
Cheeseman, P., Kanefsky, B., \BBA\ Taylor, W.~M. \BBOP1991\BBCP.
\newblock \BBOQ {Where the Really Hard Problems Are}\BBCQ\
\newblock In {\Bem Proceedings of the 12th International Joint Conference on
  Artificial Intelligence}, {IJCAI}'91, \BPGS\ 331--337.

\bibitem[\protect\BCAY{Cohen, Crampton, Gagarin, Gutin,\ \BBA\ Jones}{Cohen
  et~al.}{2014}]{CoCrGaGuJo14}
Cohen, D., Crampton, J., Gagarin, A., Gutin, G., \BBA\ Jones, M.
  \BBOP2014\BBCP.
\newblock \BBOQ Iterative plan construction for the workflow satisfiability
  problem\BBCQ\
\newblock {\Bem Journal of Artificial Intelligence Research}, {\Bem 51},
  555--577.

\bibitem[\protect\BCAY{Cohen, Crampton, Gagarin, Gutin,\ \BBA\ Jones}{Cohen
  et~al.}{2016}]{JOCO2014}
Cohen, D., Crampton, J., Gagarin, A., Gutin, G., \BBA\ Jones, M.
  \BBOP2016\BBCP.
\newblock \BBOQ Algorithms for the workflow satisfiability problem engineered
  for counting constraints\BBCQ\
\newblock {\Bem Journal of Combinatorial Optimization}, {\Bem 32\/}(1), 3--24.

\bibitem[\protect\BCAY{Cook, Coullard,\ \BBA\ Tur{\'a}n}{Cook
  et~al.}{1987}]{CookEtal1987:complexity-cutting-plane-proofs}
Cook, W., Coullard, C.~R., \BBA\ Tur{\'a}n, G. \BBOP1987\BBCP.
\newblock \BBOQ On the complexity of cutting-plane proofs\BBCQ\
\newblock {\Bem Discrete Applied Mathematics}, {\Bem 18\/}(1), 25--38.

\bibitem[\protect\BCAY{Crampton}{Crampton}{2005}]{Cr05}
Crampton, J. \BBOP2005\BBCP.
\newblock \BBOQ A reference monitor for workflow systems with constrained task
  execution\BBCQ\
\newblock In Ferrari, E.\BBACOMMA\  \BBA\ Ahn, G.-J.\BEDS, {\Bem Proceedings of
  the 10th ACM Symposium on Access Control Models and Technologies},
  {SACMAT}'05, \BPGS\ 38--47. ACM.

\bibitem[\protect\BCAY{Crampton, Gutin,\ \BBA\ Karapetyan}{Crampton
  et~al.}{2015}]{ValuedWSP-SACMAT}
Crampton, J., Gutin, G., \BBA\ Karapetyan, D. \BBOP2015\BBCP.
\newblock \BBOQ Valued workflow satisfiability problem\BBCQ\
\newblock In {\Bem Proceedings of the 20th ACM Symposium on Access Control
  Models and Technologies}, SACMAT'15, \BPGS\ 3--13.

\bibitem[\protect\BCAY{Crampton, Gutin,\ \BBA\ Yeo}{Crampton
  et~al.}{2013}]{CrGuYe13}
Crampton, J., Gutin, G., \BBA\ Yeo, A. \BBOP2013\BBCP.
\newblock \BBOQ On the parameterized complexity and kernelization of the
  workflow satisfiability problem\BBCQ\
\newblock {\Bem ACM Transactions on Information and System Security}, {\Bem
  16\/}(1), 4:1--4:31.

\bibitem[\protect\BCAY{Crampton, Gutin, Karapetyan,\ \BBA\ Watrigant}{Crampton
  et~al.}{2017a}]{CramptonGKW17}
Crampton, J., Gutin, G.~Z., Karapetyan, D., \BBA\ Watrigant, R.
  \BBOP2017a\BBCP.
\newblock \BBOQ The bi-objective workflow satisfiability problem and workflow
  resiliency\BBCQ\
\newblock {\Bem Journal of Computer Security}, {\Bem 25\/}(1), 83--115.

\bibitem[\protect\BCAY{Crampton, Gutin,\ \BBA\ Watrigant}{Crampton
  et~al.}{2017b}]{CramptonGW17}
Crampton, J., Gutin, G.~Z., \BBA\ Watrigant, R. \BBOP2017b\BBCP.
\newblock \BBOQ On the satisfiability of workflows with release points\BBCQ\
\newblock In Bertino, E., Sandhu, R., \BBA\ Weippl, E.~R.\BEDS, {\Bem
  Proceedings of the 22nd {ACM} Symposium on Access Control Models and
  Technologies, \emph{SACMAT'17}}, \BPGS\ 207--217. {ACM}.

\bibitem[\protect\BCAY{Crawford, Ginsberg, Luks,\ \BBA\ Roy}{Crawford
  et~al.}{1996}]{CrawfordEtal1996:symmetry-breaking-predicates}
Crawford, J.~M., Ginsberg, M.~L., Luks, E.~M., \BBA\ Roy, A. \BBOP1996\BBCP.
\newblock \BBOQ Symmetry-breaking predicates for search problems\BBCQ\
\newblock In {\Bem Proceedings of the Fifth International Conference on
  Principles of Knowledge Representation and Reasoning}, KR'96, \BPGS\
  148--159, San Francisco, CA, USA. Morgan Kaufmann Publishers Inc.

\bibitem[\protect\BCAY{Culberson\ \BBA\ Gent}{Culberson\ \BBA\
  Gent}{2001}]{Culberson2001:frozen-graph-coloring}
Culberson, J.\BBACOMMA\  \BBA\ Gent, I. \BBOP2001\BBCP.
\newblock \BBOQ Frozen development in graph coloring\BBCQ\
\newblock {\Bem Theoretical Computer Science}, {\Bem 265\/}(1-2), 227--264.

\bibitem[\protect\BCAY{Davis, Logemann,\ \BBA\ Loveland}{Davis
  et~al.}{1962}]{DavisLL1962}
Davis, M., Logemann, G., \BBA\ Loveland, D. \BBOP1962\BBCP.
\newblock \BBOQ A machine program for theorem-proving\BBCQ\
\newblock {\Bem Communications of the ACM}, {\Bem 5\/}(7), 394--397.

\bibitem[\protect\BCAY{De~Haan, Kronegger,\ \BBA\ Pfandler}{De~Haan
  et~al.}{2015}]{DeHaan:2015:IJCAI:FPT-planning}
De~Haan, R., Kronegger, M., \BBA\ Pfandler, A. \BBOP2015\BBCP.
\newblock \BBOQ Fixed-parameter tractable reductions to {SAT} for
  planning\BBCQ\
\newblock In {\Bem Proceedings of the 24th International Conference on
  Artificial Intelligence}, IJCAI'15, \BPGS\ 2897--2903. AAAI Press.

\bibitem[\protect\BCAY{Dixon, Ginsberg,\ \BBA\ Parkes}{Dixon
  et~al.}{2004}]{DixonEtal:2004:GenBoolSatI}
Dixon, H.~E., Ginsberg, M.~L., \BBA\ Parkes, A.~J. \BBOP2004\BBCP.
\newblock \BBOQ Generalizing boolean satisfiability {I}: Background and survey
  of existing work\BBCQ\
\newblock {\Bem Journal of Artificial Intelligence Research}, {\Bem 21},
  193--243.

\bibitem[\protect\BCAY{dos Santos, Ranise, Compagna,\ \BBA\ Ponta}{dos Santos
  et~al.}{2017}]{SantosRCP17}
dos Santos, D.~R., Ranise, S., Compagna, L., \BBA\ Ponta, S.~E. \BBOP2017\BBCP.
\newblock \BBOQ Automatically finding execution scenarios to deploy
  security-sensitive workflows\BBCQ\
\newblock {\Bem Journal of Computer Security}, {\Bem 25\/}(3), 255--282.

\bibitem[\protect\BCAY{dos Santos, Ranise, Compagna,\ \BBA\ Ponta}{dos Santos
  et~al.}{2015}]{SantosRCP15}
dos Santos, D.~R., Ranise, S., Compagna, L., \BBA\ Ponta, S.~E. \BBOP2015\BBCP.
\newblock \BBOQ Assisting the deployment of security-sensitive workflows by
  finding execution scenarios\BBCQ\
\newblock In Samarati, P.\BED, {\Bem Data and Applications Security and Privacy
  {XXIX}, Proceedings of DBSec 2015}, \lowercase{\BVOL}\ 9149 of {\Bem Lecture
  Notes in Computer Science}, \BPGS\ 85--100. Springer.

\bibitem[\protect\BCAY{Downey\ \BBA\ Fellows}{Downey\ \BBA\
  Fellows}{2013}]{DoFe13}
Downey, R.\BBACOMMA\  \BBA\ Fellows, M.~R. \BBOP2013\BBCP.
\newblock {\Bem Fundamentals of Parameterized Complexity}.
\newblock Texts in Computer Science. Springer-Verlag, London.

\bibitem[\protect\BCAY{Dukanovic\ \BBA\ Rendl}{Dukanovic\ \BBA\
  Rendl}{2008}]{DukanovicRendl2008:SDP-heuristic-GCP}
Dukanovic, I.\BBACOMMA\  \BBA\ Rendl, F. \BBOP2008\BBCP.
\newblock \BBOQ A semidefinite programming-based heuristic for graph
  coloring\BBCQ\
\newblock {\Bem Discrete Applied Mathematics}, {\Bem 156\/}(2), 180 -- 189.

\bibitem[\protect\BCAY{Dutton\ \BBA\ Brigham}{Dutton\ \BBA\
  Brigham}{1981}]{Dutton1981}
Dutton, R.~D.\BBACOMMA\  \BBA\ Brigham, R.~C. \BBOP1981\BBCP.
\newblock \BBOQ A new graph colouring algorithm\BBCQ\
\newblock {\Bem The Computer Journal}, {\Bem 24\/}(1), 85--86.

\bibitem[\protect\BCAY{Fellows, Friedrich, Hermelin, Narodytska,\ \BBA\
  Rosamond}{Fellows et~al.}{2011}]{Fellows2011}
Fellows, M.~R., Friedrich, T., Hermelin, D., Narodytska, N., \BBA\ Rosamond,
  F.~A. \BBOP2011\BBCP.
\newblock \BBOQ Constraint satisfaction problems: Convexity makes
  {AllDifferent} constraints tractable\BBCQ\
\newblock In {\Bem Proceedings of the 22nd International Joint Conference on
  Artificial Intelligence}, IJCAI'11, \BPGS\ 522--527.

\bibitem[\protect\BCAY{Gent, MacIntyre, Prosser,\ \BBA\ Walsh}{Gent
  et~al.}{1996}]{Gent1996}
Gent, I.~P., MacIntyre, E., Prosser, P., \BBA\ Walsh, T. \BBOP1996\BBCP.
\newblock \BBOQ The constrainedness of search\BBCQ\
\newblock In {\Bem Proceedings of the 13th National Conference on Artificial
  Intelligence and 8th Innovative Applications of Artificial Intelligence
  Conference, \emph{Vol.\,1}}, AAAI/IAAI'96, \BPGS\ 246--252.

\bibitem[\protect\BCAY{Gutin, Kratsch,\ \BBA\ Wahlstr\"om}{Gutin
  et~al.}{2015}]{Gutin2015}
Gutin, G., Kratsch, S., \BBA\ Wahlstr\"om, M. \BBOP2015\BBCP.
\newblock \BBOQ Polynomial kernels and user reductions for the workflow
  satisfiability problem\BBCQ\
\newblock {\Bem Algorithmica}, {\Bem 75\/}(2), 383--402.

\bibitem[\protect\BCAY{Gutin\ \BBA\ Wahlstr{\"{o}}m}{Gutin\ \BBA\
  Wahlstr{\"{o}}m}{2016}]{GW2016}
Gutin, G.\BBACOMMA\  \BBA\ Wahlstr{\"{o}}m, M. \BBOP2016\BBCP.
\newblock \BBOQ Tight lower bounds for the workflow satisfiability problem
  based on the strong exponential time hypothesis\BBCQ\
\newblock {\Bem Information Processing Letters}, {\Bem 116\/}(3), 223--226.

\bibitem[\protect\BCAY{Haken}{Haken}{1985}]{HAKEN1985:intractability}
Haken, A. \BBOP1985\BBCP.
\newblock \BBOQ The intractability of resolution\BBCQ\
\newblock {\Bem Theoretical Computer Science}, {\Bem 39}, 297 -- 308.

\bibitem[\protect\BCAY{Huberman\ \BBA\ Hogg}{Huberman\ \BBA\
  Hogg}{1987}]{Huberman87:phase}
Huberman, B.~A.\BBACOMMA\  \BBA\ Hogg, T. \BBOP1987\BBCP.
\newblock \BBOQ Phase transitions in artificial intelligence systems\BBCQ\
\newblock {\Bem Artificial Intelligence}, {\Bem 33\/}(2), 155--171.

\bibitem[\protect\BCAY{Impagliazzo\ \BBA\ Paturi}{Impagliazzo\ \BBA\
  Paturi}{2001}]{ImPa01}
Impagliazzo, R.\BBACOMMA\  \BBA\ Paturi, R. \BBOP2001\BBCP.
\newblock \BBOQ On the complexity of k-{SAT}\BBCQ\
\newblock {\Bem Journal of Computer and System Sciences}, {\Bem 62\/}(2),
  367--375.

\bibitem[\protect\BCAY{Karapetyan}{Karapetyan}{2019}]{SourceCodes}
Karapetyan, D. \BBOP2019\BBCP.
\newblock \BBOQ Source codes of the {Pattern Backtracking} algorithm, the
  instance generator, the instances used in the paper, corresponding solutions
  and the converter of instances into {PBPB}, {UDPB} and {CSP}
  formulations\BBCQ.
\newblock \texttt{https://dx.doi.org/10.5526/ERDR-00000114}, retrieved 22 July
  2019.

\bibitem[\protect\BCAY{Karapetyan, Gagarin,\ \BBA\ Gutin}{Karapetyan
  et~al.}{2015}]{KaGaGu}
Karapetyan, D., Gagarin, A., \BBA\ Gutin, G. \BBOP2015\BBCP.
\newblock \BBOQ Pattern backtracking algorithm for the workflow satisfiability
  problem with user-independent constraints\BBCQ\
\newblock In Wang, J.\BBACOMMA\  \BBA\ Yap, C.\BEDS, {\Bem Frontiers in
  Algorithmics, \emph{FAW 2015}}, \lowercase{\BVOL}\ 9130 of {\Bem Lecture
  Notes in Computer Science}, \BPGS\ 138--149. Springer.

\bibitem[\protect\BCAY{Kilby, Slaney, Thi{\'e}baux,\ \BBA\ Walsh}{Kilby
  et~al.}{2005}]{KilbyEtal2005:backbones}
Kilby, P., Slaney, J., Thi{\'e}baux, S., \BBA\ Walsh, T. \BBOP2005\BBCP.
\newblock \BBOQ Backbones and backdoors in satisfiability\BBCQ\
\newblock In {\Bem Proceedings of the 20th National Conference on Artificial
  Intelligence, \emph{Vol.\,3}}, AAAI'05, \BPGS\ 1368--1373. AAAI Press.

\bibitem[\protect\BCAY{Kronegger, Pfandler,\ \BBA\ Pichler}{Kronegger
  et~al.}{2013}]{HaanEtal2013:IJCAI-FPT-planning}
Kronegger, M., Pfandler, A., \BBA\ Pichler, R. \BBOP2013\BBCP.
\newblock \BBOQ Parameterized complexity of optimal planning: A detailed
  map\BBCQ\
\newblock In {\Bem Proceedings of the 23d International Joint Conference on
  Artificial Intelligence}, IJCAI'13, \BPGS\ 954--961.

\bibitem[\protect\BCAY{Le~Berre\ \BBA\ Parrain}{Le~Berre\ \BBA\
  Parrain}{2010}]{BePa10}
Le~Berre, D.\BBACOMMA\  \BBA\ Parrain, A. \BBOP2010\BBCP.
\newblock \BBOQ The {SAT4J} library, release 2.2, system description\BBCQ\
\newblock {\Bem Journal on Satisfiability, Boolean Modeling and Computation},
  {\Bem 7}, 59--64.

\bibitem[\protect\BCAY{Lov\'asz}{Lov\'asz}{1979}]{Lovasz1979:shannon}
Lov\'asz, L. \BBOP1979\BBCP.
\newblock \BBOQ On the {S}hannon capacity of a graph\BBCQ\
\newblock {\Bem IEEE Transactions on Information Theory}, {\Bem 25\/}(1), 1--7.

\bibitem[\protect\BCAY{M{\'e}zard, Mora,\ \BBA\ Zecchina}{M{\'e}zard
  et~al.}{2005}]{mezard2005clustering}
M{\'e}zard, M., Mora, T., \BBA\ Zecchina, R. \BBOP2005\BBCP.
\newblock \BBOQ Clustering of solutions in the random satisfiability
  problem\BBCQ\
\newblock {\Bem Physical Review Letters}, {\Bem 94\/}(19), 197205:1--197205:4.

\bibitem[\protect\BCAY{Mitchell, Selman,\ \BBA\ Levesque}{Mitchell
  et~al.}{1992}]{mitchell92hard}
Mitchell, D., Selman, B., \BBA\ Levesque, H. \BBOP1992\BBCP.
\newblock \BBOQ Hard and easy distributions of {SAT} problems\BBCQ\
\newblock In Rosenbloom, P.\BBACOMMA\  \BBA\ Szolovits, P.\BEDS, {\Bem
  Proceedings of the 10th National Conference on Artificial Intelligence},
  AAAI'92, \BPGS\ 459--465, Menlo Park, California. AAAI Press.

\bibitem[\protect\BCAY{Ordyniak\ \BBA\ Szeider}{Ordyniak\ \BBA\
  Szeider}{2013}]{OrdyniakSzeider2013:JAIR:FPT-Bayesian}
Ordyniak, S.\BBACOMMA\  \BBA\ Szeider, S. \BBOP2013\BBCP.
\newblock \BBOQ Parameterized complexity results for exact bayesian network
  structure learning\BBCQ\
\newblock {\Bem Journal of Artificial Intelligence Research}, {\Bem 46},
  263--302.

\bibitem[\protect\BCAY{Parkes}{Parkes}{1997}]{parkes1997clustering}
Parkes, A.~J. \BBOP1997\BBCP.
\newblock \BBOQ Clustering at the phase transition\BBCQ\
\newblock In {\Bem Proceedings of the 14th National Conference on Artificial
  Intelligence}, AAAI'97, \BPGS\ 340--345. AAAI Press.

\bibitem[\protect\BCAY{Razborov}{Razborov}{2002}]{Razborov2002:Complexity-PHP}
Razborov, A.~A. \BBOP2002\BBCP.
\newblock \BBOQ Proof complexity of pigeonhole principles\BBCQ\
\newblock In Kuich, W., Rozenberg, G., \BBA\ Salomaa, A.\BEDS, {\Bem
  Developments in Language Theory, \emph{DLT 2001}}, \lowercase{\BVOL}\ 2295 of
  {\Bem Lecture Notes in Computer Science}, \BPGS\ 100--116, Berlin,
  Heidelberg. Springer.

\bibitem[\protect\BCAY{Roy, Sural, Majumdar, Vaidya,\ \BBA\ Atluri}{Roy
  et~al.}{2015}]{Roy2015}
Roy, A., Sural, S., Majumdar, A.~K., Vaidya, J., \BBA\ Atluri, V.
  \BBOP2015\BBCP.
\newblock \BBOQ Minimizing organizational user requirement while meeting
  security constraints\BBCQ\
\newblock {\Bem ACM Transactions on Management Information Systems}, {\Bem
  6\/}(3), 12:1--12:25.

\bibitem[\protect\BCAY{Selman\ \BBA\ Kirkpatrick}{Selman\ \BBA\
  Kirkpatrick}{1996}]{SelmanKirkpatrick1996:AIJ-critical-behavior}
Selman, B.\BBACOMMA\  \BBA\ Kirkpatrick, S. \BBOP1996\BBCP.
\newblock \BBOQ Critical behavior in the computational cost of satisfiability
  testing\BBCQ\
\newblock {\Bem Artificial Intelligence}, {\Bem 81\/}(1-2), 273--295.

\bibitem[\protect\BCAY{Stuckey}{Stuckey}{2018}]{MiniZinc}
Stuckey, P.~J. \BBOP2018\BBCP.
\newblock \BBOQ Mini{Z}inc {C}hallenge 2018\BBCQ.
\newblock
  \url{https://www.minizinc.org/challenge2018/challenge_results2018.pdf},
  retrieved 20 July 2019.

\bibitem[\protect\BCAY{Wang\ \BBA\ Li}{Wang\ \BBA\ Li}{2010}]{WaLi10}
Wang, Q.\BBACOMMA\  \BBA\ Li, N. \BBOP2010\BBCP.
\newblock \BBOQ Satisfiability and resiliency in workflow authorization
  systems\BBCQ\
\newblock {\Bem ACM Transactions on Information and System Security}, {\Bem
  13\/}(4), 40:1--40:35.

\bibitem[\protect\BCAY{West}{West}{2001}]{West}
West, D.~B. \BBOP2001\BBCP.
\newblock {\Bem Introduction to Graph Theory\/} (2nd \BEd).
\newblock Prentice-Hall.

\end{thebibliography}
\bibliographystyle{theapa}

\end{document}